\documentclass[preprint,5p,twocolumn,times,10pt]{elsarticle}

\usepackage{xcolor}
\usepackage[colorlinks=True]{hyperref}
\usepackage{url}
\usepackage{booktabs}
\usepackage{amsfonts,amsmath,amsthm,amssymb}
\newtheorem{theorem}{Theorem}[section]
\newtheorem{proposition}{Proposition}[section]
\newtheorem{assumption}{Assumption}[section]
\newtheorem{definition}{Definition}[section]
\newtheorem{corollary}{Corollary}[section]
\newtheorem{lemma}{Lemma}[section]

\usepackage[nameinlink]{cleveref}
\crefname{theorem}{theorem}{theorems}
\Crefname{theorem}{Theorem}{Theorems}
\crefname{proposition}{proposition}{propositions}
\Crefname{proposition}{Proposition}{Propositions}
\crefname{assumption}{assumption}{assumptions}
\Crefname{assumption}{Assumption}{Assumptions}
\crefname{definition}{definition}{definitions}
\Crefname{definition}{Definition}{Definitions}
\crefname{corollary}{corollary}{corollaries}
\Crefname{corollary}{Corollary}{Corollaries}
\crefname{lemma}{lemma}{lemmas}
\Crefname{lemma}{Lemma}{Lemmas}
\crefname{remark}{remark}{remarks}
\Crefname{remark}{Remark}{Remarks}

\usepackage{thmtools, thm-restate}
\usepackage{bm}
\usepackage{mathtools}
\usepackage[normalem]{ulem}
\usepackage{algorithm}
\usepackage{algpseudocode}

\newcommand{\rev}[1]{#1}

\journal{Neural Networks}

\begin{document}

\begin{frontmatter}



\title{Toward Theoretical Insights into Diffusion Trajectory Distillation via Operator Merging}

\author[1,2]{Weiguo Gao}
\ead{wggao@fudan.edu.cn}
\author[1]{Ming Li\corref{cor1}}
\cortext[cor1]{Corresponding author.}
\ead{mingli23@m.fudan.edu.cn}

\affiliation[1]{organization={School of Mathematical Sciences, Fudan University},
city={Shanghai},
postcode={200433},
country={China}}

\affiliation[2]{organization={Shanghai Key Laboratory of Contemporary Applied Mathematics},
city={Shanghai},
postcode={200433},
country={China}}

\begin{abstract}
Diffusion trajectory distillation accelerates sampling by training a student model to approximate the multi-step denoising trajectories of a pretrained teacher model using far fewer steps. Despite strong empirical results, the trade-off between distillation strategy and generative quality remains poorly understood. We provide a theoretical characterization by reinterpreting trajectory distillation as an operator merging problem, differentiating our analysis between two distinct regimes. In the linear Gaussian regime, where approximation error is zero, we isolate optimization error, specifically signal shrinkage driven by finite training time, as the primary bottleneck. This characterization allows us to derive the theoretically optimal merging strategy, which exhibits a variance-driven phase transition and is computable via a Pareto dynamic programming algorithm. In the nonlinear Gaussian mixture regime, we prove that distilling composite steps incurs unavoidable approximation error due to the exponential growth of mixture components, and we quantify how these errors amplify across merges. Together, these results clarify the distinct theoretical mechanisms governing each regime and provide principled guidance for method selection.
\end{abstract}

\end{frontmatter}

%%%%%%%%%%%%%%%%%%%%%%%%%%%%%%%%%%%%%%%%%%%%%%%%%%%%%%%%%%%%%%%%%%%%%%%%%%%%%%%%
\section{Introduction}
\label{sec:introduction}

Diffusion models have become a cornerstone of modern generative modeling, delivering state-of-the-art results across a wide range of domains~\citep{ho2020denoising,karras2022elucidating,nichol2021improved,ren2025zeroth,song2021denoising,song2019generative,song2021score,yu2025tree}. Their success stems from a multi-step denoising process that progressively refines a Gaussian noise into a structured data sample. While this iterative refinement enables high sample fidelity, it also imposes a critical bottleneck: generating a single output often requires hundreds or thousands of steps, rendering diffusion models impractical for real-time applications~\citep{ho2020denoising,song2019generative}.

To overcome this limitation, recent work has turned to \emph{diffusion distillation}~\citep{berthelot2023tract,gu2023boot,li2024reward,lin2024sdxl,luhman2021knowledge,luo2023diff,meng2023distillation,salimans2022progressive,sauer2024adversarial,song2023consistency,xu2025one,yin2024improved,yin2024one,zhou2024score,zhou2025adversarial}, which trains a student model to mimic the behavior of a full diffusion trajectory (i.e., the teacher model) using only a small number of steps. These methods, including trajectory distillation, distribution matching distillation, and adversarial distillation, have demonstrated striking empirical success, with some achieving high-quality synthesis in just one step~\citep{luhman2021knowledge,sauer2024adversarial,yin2024one}. Despite their growing adoption, the mechanisms behind their effectiveness remain poorly understood. In particular, it is unclear how design choices such as merge order, trajectory length, or student model initialization influence generation quality and efficiency. This lack of theoretical understanding makes it difficult to anticipate the impact of these factors and often forces practitioners to rely on trial-and-error or heuristics that may not generalize across datasets or noise schedules. More importantly, without clear guiding principles, it becomes challenging to systematically improve distillation strategies.

In this work, we conduct a detailed analysis of \emph{diffusion trajectory distillation}~\citep{berthelot2023tract,gu2023boot,li2024reward,lin2024sdxl,luhman2021knowledge,meng2023distillation,salimans2022progressive,song2023consistency}, a class of methods that directly compress the sampling trajectory of a teacher model into fewer student steps. This approach avoids the complications associated with auxiliary score networks (as in distribution matching distillation~\citep{luo2023diff,xu2025one,yin2024improved,yin2024one,zhou2024score}) or discriminators (as in adversarial distillation~\citep{lin2024sdxl,sauer2024adversarial,yin2024improved,zhou2025adversarial}), while capturing the core idea that underlies many practical distillation techniques.

To better understand trajectory distillation, we study a simplified setting that permits clean theoretical analysis. As demonstrated by~\citet{li2024understanding}, diffusion models exhibit an inductive bias toward learning \emph{Gaussian structures} under limited capacity. We therefore begin with the unimodal Gaussian case, where both the teacher and the student models can be represented as linear operators applied to noisy observations. In this setting, trajectory distillation reduces to an \emph{operator merging} problem, where a multi-step teacher trajectory is compressed into fewer student steps. We show that each merge causes \emph{signal shrinkage} through a convex-combination effect driven by discretization and the student model's finite optimization time, making explicit how the merge order affects signal preservation.

\begin{table*}[htb]
\centering
\caption{Comparative analysis of theoretical frameworks for diffusion trajectory distillation in linear (Gaussian) versus nonlinear (Gaussian mixtures) regimes. Note that the symbol ``--'' in the optimization error row indicates analytical intractability, and both this analysis and the derivation of optimal strategies in the nonlinear regime are left as future work.}
\smallskip
\label{tab:theoretical_comparison}
\footnotesize
\renewcommand{\arraystretch}{1.15}
\begin{tabular}{@{}p{2.35cm}p{6.7cm}p{8.4cm}@{}}
\toprule
Aspect & Linear regime (Gaussian) & Nonlinear regime (Gaussian mixture) \\
\midrule
Distribution & 
Gaussian \(\mathcal{N}(\bm 0, \bm \Lambda)\) with diagonal \(\bm \Lambda\) (\Cref{as:real_data_distribution}) & 
Gaussian mixtures \(\sum_{k=1}^K \pi_k \mathcal{N}\bigl(\bm{x}_0; \bm{\mu}_k, \bm{\Lambda}_k\bigr)\) (\cref{eq:gmm_assumption}) \\
Optimal denoiser & 
Linear operator \(\alpha_t\bm{\Lambda}(\alpha_t^2 \bm{\Lambda} + \sigma_t^2 \bm{I})^{-1}\bm{z}_t\) (\Cref{prop:optimal_denoising_estimator}) & \(K\)-component affine mixture of experts (MoE) operator
\(\sum_{k=1}^K
\gamma_{k,t}(\bm{z}_t)
\bigl(
\bm{\mu}_k
+
\alpha_t \bm{\Lambda}_k
\bigl(\alpha_t^2 \bm{\Lambda}_k + \sigma_t^2 \bm{I}\bigr)^{-1}
(\bm{z}_t - \alpha_t \bm{\mu}_k)
\bigr)\) (\Cref{prop:optimal_mixture_denoising_estimator})\\
Model param. & 
Time-dependent linear operator \(\bm A_t^{\mathrm{st}}\bm z_t\) (\cref{eq:student_model,eq:teacher_model}) & 
\(K\)-component affine MoE operator \(\sum_{k=1}^K \gamma_{k,t}^{\mathrm{st}}(\bm{z}_t)\bigl(\bm{A}_{k,t}^{\mathrm{st}} \bm{z}_t + \bm{b}_{k,t}^{\mathrm{st}}\bigr)\) (\cref{eq:gmm_teacher_moe,eq:gmm_student_moe}) \\
\midrule
Approximation err. & 
Zero (linear closure holds) & 
Compressing \(K^k\) teacher components into \(K\) student experts incurs nonzero error (\Cref{subsec:gmm_approximation_error} and \Cref{thm:two_step_merging_error}) \\
Optimization err. & 
Finite optimization time forces interpolation between target and initialization (\Cref{subsec:optimization_dynamics_of_the_student} and \Cref{prop:student_gradient_flow}) & 
-- \\
Error propagation & 
Compounding errors accumulate differently based on variance (\Cref{prop:student_gradient_flow}) & 
Early errors are amplified by the teacher model's Lipschitz constant (\Cref{subsec:error_propagation_across_merges} and \Cref{thm:error_propagation_two_stage_merge}) \\
\midrule
Resulting strategy & 
Optimal merge plans exhibit a variance-driven phase transition, while mixed-variance conflicts are resolved via Pareto dynamic programming (\Cref{sec:variance_driven_optimal_merge_plans_and_pareto_dp_in_the_gaussian_regime}) & 
-- \\

\bottomrule
\end{tabular}
\end{table*}

Building on this foundation, we study which merge plans are optimal when the goal is to match the teacher model's output distribution. Using the \(2\)-Wasserstein distance for Gaussian outputs, we prove a sharp \emph{phase transition} in the optimal strategy as a function of data variance, identifying regimes in which standard baselines such as vanilla distillation~\citep{luhman2021knowledge} and sequential BOOT~\citep{gu2023boot} are optimal. Guided by these theoretical characterizations, we present a Pareto dynamic programming algorithm that computes globally optimal merge plans while efficiently pruning inferior intermediate candidates, which is crucial when different dimensions in the data have different variances.

To address multimodality, we extend the analysis to Gaussian mixture data. In this regime, each teacher operator can be written as a \(K\)-component affine mixture-of-experts map, and composing \(k\) steps expands to \(K^k\) effective components. We characterize the unavoidable approximation error when a \(K\)-component student operator distills such a composite teacher, and we show how these errors accumulate across merges, explaining why long-horizon merging becomes substantially harder for multimodal data. We provide a comparative summary of our theoretical frameworks for the linear and nonlinear regimes in~\Cref{tab:theoretical_comparison}.

The main contributions of this paper are as follows:
\begin{itemize}
\item \textbf{Operator-merging foundations and signal shrinkage in the unimodal Gaussian case.}
In the linear Gaussian regime, we derive closed-form teacher operators and formulate trajectory distillation as an operator-merging problem (\Cref{subsec:gaussian_assuption_and_the_optimal_denoising_estimator}). This view isolates a concrete degradation mechanism by characterizing how merging induces signal shrinkage under finite student optimization time (\Cref{subsec:optimization_dynamics_of_the_student}).
\item \textbf{A variance-driven phase transition in optimal merge plans.}
Using the \(2\)-Wasserstein distance between Gaussian output distributions as the objective (\Cref{subsec:w2_metric_and_surrogate_target}), we prove a sharp phase transition in the optimal merging strategy as a function of data variance \(\lambda\) (\Cref{subsec:phase_transitions_in_the_optimal_strategy}). The result recovers sequential BOOT when \(\lambda \leq 1\) and vanilla trajectory distillation when \(\lambda \gg 1\), providing a principled explanation for when each baseline is expected to work well.
\item \textbf{A Pareto dynamic programming algorithm for optimal merge planning.}
Motivated by the optimality characterization, we propose a Pareto dynamic programming algorithm that computes optimal merge plans in high-dimensional settings with mixed variances, while efficiently pruning inferior intermediate candidates (\Cref{subsec:pareto_dynamic_programming_for_optimal_merging}).
\item \textbf{Gaussian-mixture extension with approximation and error propagation bounds.}
We extend the analysis to Gaussian mixtures by modeling each teacher step as a \(K\)-component affine mixture-of-experts operator and showing that \(k\)-step composition expands to \(K^k\) components (\Cref{subsec:optimal_denoising_for_gmm}). In this setting, we quantify the intrinsic approximation error of a \(K\)-component student when distilling a composite teacher (\Cref{subsec:gmm_approximation_error}), and we characterize how this error accumulates across merges (\Cref{subsec:error_propagation_across_merges}).
\end{itemize}

%%%%%%%%%%%%%%%%%%%%%%%%%%%%%%%%%%%%%%%%%%%%%%%%%%%%%%%%%%%%%%%%%%%%%%%%%%%%%%%%
\section{Related work}
\label{sec:related_work}

While diffusion models~\citep{ho2020denoising, karras2022elucidating, song2021score} offer superior generative quality, their reliance on long iterative chains necessitates acceleration. Early efforts focused on fast samplers, reinterpreting diffusion as continuous-time differential equations~\citep{song2021score}. Advanced ODE/SDE solvers~\citep{karras2022elucidating,liu2022pseudo,lu2022dpm,lu2022dpmpp,zhao2023unipc} can reduce inference to \(20\)--\(50\) steps. However, to achieve extreme acceleration (e.g., single-step generation), recent research has pivoted toward diffusion distillation, broadly categorized into trajectory, distribution matching, and adversarial approaches. While a separate line of work pursues one-step generation via formulations that do not fit the distillation paradigm (e.g., \rev{Shortcut models~\citep{frans2025one}}, MeanFlow~\citep{geng2025mean}, AlphaFlow~\citep{zhang2025alphaflow}, SoFlow~\citep{luo2025soflow}), we focus on distillation because it allows the student to inherit the generative power of a high-quality, multi-step teacher. By adopting this teacher--student framework, our analysis can formally characterize which trajectory-merging strategies are optimal for capturing the teacher model's pretrained output distribution.

\paragraph{Diffusion trajectory distillation}
These methods accelerate diffusion sampling by compressing the sampling trajectory of a teacher model into fewer student steps. This includes vanilla distillation~\citep{luhman2021knowledge}, which learns a single-step approximation of the full trajectory; progressive distillation~\citep{lin2024sdxl,meng2023distillation,salimans2022progressive}, which halves the number of steps iteratively; consistency models~\citep{li2024reward,song2023consistency}, which aim to map noisy samples at arbitrary timesteps to a shared clean sample via a consistency objective; and BOOT~\citep{gu2023boot}, which maps the noisiest sample step-by-step to lower-noise states. Under the teacher model's supervision, these methods can reduce inference to under \(10\) steps without significant loss in generative quality.

Despite empirical success, the theoretical understanding of trajectory distillation remains underexplored. While~\citet{xu2024local} demonstrated that partitioning the diffusion process into smaller intervals can reduce the distribution gap between sub-models, the comparative advantage of different trajectory merging strategies remains uncharacterized. This lack of theoretical clarity motivates our analysis. Our work investigates both unimodal and multimodal Gaussian settings, where we recast trajectory distillation as an operator merging problem and prove a sharp phase transition in the optimal merging strategy as a function of data variance.

\paragraph{Distribution matching and adversarial distillation}
Alternative strategies focus on matching probability distributions rather than specific trajectories. Distribution matching distillation minimizes divergences (e.g., Kullback--Leibler divergence~\citep{yin2024improved,yin2024one}, Fisher divergence~\citep{zhou2024score}, \(f\)-divergence~\citep{xu2025one}, or integral probability metrics like the integral KL~\citep{luo2023diff}) between student and teacher distributions. This typically requires training an auxiliary student score network to approximate the gradient of the log-density, complicating the optimization landscape. Adversarial distillation~\citep{lin2024sdxl,sauer2024adversarial,yin2024improved,zhou2025adversarial} employs discriminators to distinguish real from generated samples. While powerful, adversarial training is often unstable at high noise levels where both real and generated samples are heavily corrupted, and is often hybridized with trajectory distillation to ensure convergence~\citep{yin2024one}.

%%%%%%%%%%%%%%%%%%%%%%%%%%%%%%%%%%%%%%%%%%%%%%%%%%%%%%%%%%%%%%%%%%%%%%%%%%%%%%%%
\section{Background}
\label{sec:background}

We begin with a brief overview of diffusion models in~\Cref{subsec:diffusion_models} and trajectory distillation in~\Cref{subsec:diffusion_trajectory_distillation}, unifying several distillation methods under a shared notation. Then, we review the \(2\)-Wasserstein distance for Gaussian distributions in~\Cref{subsec:wasserstein_distance_for_gaussian_outputs}.

%%%%%%%%%%%%%%%%%%%%%%%%%%%%%%%%%%%%%%%%%%%%%%%%%%%%%%%%%%%%
\subsection{Diffusion models}
\label{subsec:diffusion_models}

Diffusion models are a class of generative models that generate data samples by gradually transforming Gaussian noise into structured data. In the forward process, a clean data point \(\bm{x}_0\sim p_0\) is progressively corrupted by adding Gaussian noise according to a prescribed schedule. More precisely, for discrete time steps \(t \in \{0, 1, \dotsc, T\}\), the forward process is defined by
\begin{equation}
\label{eq:forward_process}
\bm{z}_t = \alpha_t\bm{x}_0 + \sigma_t\bm{\varepsilon}, \quad \bm{\varepsilon} \sim \mathcal{N}(\bm{0}, \bm{I}),
\end{equation}
where the noise schedule \(\{\alpha_t\}\), \(\{\sigma_t\}\) balances signal preservation and noise injection. We assume that the noise schedule satisfies~\Cref{as:noise_schedule}.

\begin{assumption}[Noise schedule]
\label{as:noise_schedule}
The noise schedule \(\{\alpha_t\}\) and \(\{\sigma_t\}\) satisfies \(\alpha_t^2 + \sigma_t^2 = 1\) for all \(t\), with \(\alpha_t\) decreasing. The boundary conditions are \(\alpha_0 = 1\), \(\sigma_0 = 0\), and \(\alpha_T = 0\), \(\sigma_T = 1\).
\end{assumption}

In the forward process, a neural network parameterized by \(\bm{\eta}\) is trained to estimate the clean sample \(\bm{x}_0\) from a noisy observation \(\bm{z}_t\). It learns an estimator \(\hat{\bm{x}}_{\bm{\eta}}(\bm{z}_t, t)\) by minimizing the expected denoising loss
\begin{equation}
\label{eq:denoising_loss_function}
\mathcal{L}_{\text{denoise}}(\bm{\eta}) = \mathbb{E}_{t \sim \mathcal{U}(0,T)} \bigl[ w_t \cdot \mathbb{E}_{\bm{x}_0 \sim p_0, \bm{\varepsilon} \sim \mathcal{N}(\bm{0}, \bm{I})} \lVert \hat{\bm{x}}_{\bm{\eta}}(\bm{z}_t, t) - \bm{x}_0 \rVert_2^2 \bigr],
\end{equation}
where \(\bm{z}_t = \alpha_t \bm{x}_0 + \sigma_t \bm{\varepsilon}\), and \(w_t\) controls the importance of each timestep \(t\) during training. During sampling, the model reconstructs \(\bm{x}_0\) by iteratively denoising \(\bm{z}_T\). The Denoising Diffusion Implicit Model (DDIM)~\citep{song2021denoising} defines this reverse process deterministically as
\begin{equation}
\label{eq:denosing_process}
\bm{z}_{t-1} = \alpha_{t-1}\hat{\bm{x}}_{\bm{\eta}}(\bm{z}_t, t) + \sigma_{t-1} \cdot \biggl( \frac{\bm{z}_t - \alpha_t \hat{\bm{x}}_{\bm{\eta}}(\bm{z}_t, t)}{\sigma_t} \biggr) \coloneqq \bm{f}_{\bm{\eta}}(\bm{z}_t, t)
\end{equation}
with \(t=T\dotsc,1\). Here, \(\bm{f}_{\bm{\eta}}(\bm{z}_t, t)\) denotes the deterministic update that maps \(\bm{z}_t\) to a less noisy sample \(\bm{z}_{t-1}\). We will adopt this function as the primary parameterization throughout this work.

%%%%%%%%%%%%%%%%%%%%%%%%%%%%%%%%%%%%%%%%%%%%%%%%%%%%%%%%%%%%
\subsection{Diffusion trajectory distillation}
\label{subsec:diffusion_trajectory_distillation}

Diffusion trajectory distillation~\citep{berthelot2023tract,gu2023boot,li2024reward,lin2024sdxl,luhman2021knowledge,meng2023distillation,salimans2022progressive,song2023consistency} accelerates sampling by training a student model to approximate multiple denoising steps of a teacher model in \emph{a single step}, aiming to minimize degradation in sample quality. This leads to the definition of a composite operator over \(k\) denoising steps in~\Cref{def:composite_teacher_operator}.

\begin{definition}[Composite operator of multiple steps]
\label{def:composite_teacher_operator}
Given a noisy sample \(\bm{z}_t\), the composite operator \(\mathcal{T}_k(\bm{z}_t)\) denotes the sequential application of \(k\) teacher denoising steps, i.e.,
\begin{equation}
\mathcal{T}_k(\bm{z}_t) \coloneqq \bm{f}_{\bm{\eta}}\bigl(\bm{f}_{\bm{\eta}}\bigl(\dotsb \bm{f}_{\bm{\eta}}(\bm{z}_t, t) \dotsb, t-k+2\bigr), t-k+1\bigr),
\end{equation}
which maps \(\bm{z}_t\) to a cleaner sample \(\bm{z}_{t-k}\). Here, \(\bm{f}_{\bm{\eta}}\) is a single-step operator defined in~\eqref{eq:denosing_process}.
\end{definition}

Several trajectory distillation strategies can be unified using this notation, including \emph{vanilla distillation}~\citep{luhman2021knowledge}, \emph{progressive distillation}~\citep{lin2024sdxl,meng2023distillation,salimans2022progressive}, as well as two variants we propose: \emph{sequential consistency}, based on~\citep{li2024reward,song2023consistency}, and \emph{sequential BOOT}, adapted from~\citep{gu2023boot}. We summarize them in~\Cref{tab:distillation_methods}, and provide detailed descriptions and illustrations in~\ref{app:diffusion_trajectory_distillation_methods}.

% \begin{table*}[htb]
% \centering
% \caption{Overview of the four representative diffusion trajectory distillation strategies. Each method trains a student \( \bm{A}^{\mathrm{st}} \) to match a multi-step teacher operator \( \mathcal{T}_k \).}
% \label{tab:distillation_methods}
% \smallskip
% \small
% \begin{tabular}{@{}lll@{}}
% \toprule
% \textbf{Method} & \textbf{Target} & \textbf{Brief description} \\
% \midrule
% \textbf{Vanilla distillation} & \( \bm{A}^{\mathrm{st}}(\bm{z}_T)\approx\mathcal{T}_T(\bm{z}_T) \) & Approximates the entire trajectory in one step \\
% \textbf{Progressive distillation} & \( \bm{A}^{\mathrm{st}}(\bm{z}_t)\approx\mathcal{T}_2(\bm{z}_t) \) & Merges two steps at a time and reduces step iteratively \\
% \textbf{Sequential consistency} & \( \bm{A}^{\mathrm{st}}(\bm{z}_t)\approx\mathcal{T}_t(\bm{z}_t) \) & Sequentially maps any timestep \(\bm{z}_t\) directly to \(\bm{z}_0\) \\
% \textbf{Sequential BOOT} & \( \bm{A}^{\mathrm{st}}(\bm{z}_T)\approx\mathcal{T}_k(\bm{z}_T) \) & Fixes the input \(\bm{z}_T\) and sequentially extends the target to \(\bm{z}_0\) \\
% \bottomrule
% \end{tabular}
% \end{table*}

\begin{table*}[htb]
\centering
\caption{Overview of the four representative diffusion trajectory distillation strategies. Each method trains a student \( \bm{A}^{\mathrm{st}} \) to match a multi-step teacher operator \( \mathcal{T}_k \).}
\label{tab:distillation_methods}
\smallskip
\small
\begin{tabular}{@{}lll@{}}
\toprule
Method & Target & Brief description \\
\midrule
Vanilla distillation & \( \bm{A}^{\mathrm{st}}(\bm{z}_T)\approx\mathcal{T}_T(\bm{z}_T) \) & Approximates the entire trajectory in one step \\
Progressive distillation & \( \bm{A}^{\mathrm{st}}(\bm{z}_t)\approx\mathcal{T}_2(\bm{z}_t) \) & Merges two steps at a time and reduces step iteratively \\
Sequential consistency & \( \bm{A}^{\mathrm{st}}(\bm{z}_t)\approx\mathcal{T}_t(\bm{z}_t) \) & Sequentially maps any timestep \(\bm{z}_t\) directly to \(\bm{z}_0\) \\
Sequential BOOT & \( \bm{A}^{\mathrm{st}}(\bm{z}_T)\approx\mathcal{T}_k(\bm{z}_T) \) & Fixes the input \(\bm{z}_T\) and sequentially extends the target to \(\bm{z}_0\) \\
\bottomrule
\end{tabular}
\end{table*}

%%%%%%%%%%%%%%%%%%%%%%%%%%%%%%%%%%%%%%%%%%%%%%%%%%%%%%%%%%%%
\subsection{Wasserstein distance for Gaussian outputs}
\label{subsec:wasserstein_distance_for_gaussian_outputs}

We quantify the discrepancy between student and teacher output distributions using the squared \(2\)-Wasserstein distance, see, e.g.,~\citep{santambrogio2015optimal}. For probability measures \(\mu\) and \(\nu\) on \(\mathbb{R}^d\), this is defined as
\begin{equation}
\label{eq:w2_definition}
W_2^2(\mu,\nu)
\coloneqq
\inf_{\pi \in \Pi(\mu,\nu)}
\mathbb{E}_{(\bm{x},\bm{y})\sim \pi}\bigl[\lVert\bm{x}-\bm{y}\rVert_2^2\bigr],
\end{equation}
where \(\Pi(\mu,\nu)\) is the set of couplings with marginals \(\mu\) and \(\nu\).

Our analysis focuses on centered Gaussian distributions, \(\mathcal{N}(\bm{0},\bm{\Sigma})\) and \(\mathcal{N}(\bm{0},\bm{\Sigma}')\), for which the distance admits a closed-form solution
\begin{equation}
\label{eq:w2_gaussian_closed_form}
W_2^2\bigl(\mathcal{N}(\bm{0},\bm{\Sigma}),\mathcal{N}(\bm{0},\bm{\Sigma}')\bigr)
=
\mathrm{tr}\bigl(\bm{\Sigma} + \bm{\Sigma}' - 2\bigl(\bm{\Sigma}^{1/2}\bm{\Sigma}'\bm{\Sigma}^{1/2}\bigr)^{1/2}\bigr).
\end{equation}
In the specific case where the covariance matrices are diagonal, i.e., \(\bm{\Sigma}=\mathrm{Diag}(\sigma_1^2,\dotsc,\sigma_d^2)\) and \(\bm{\Sigma}'=\mathrm{Diag}((\sigma_1')^2,\dotsc,(\sigma_d')^2)\), this simplifies to the Euclidean distance between the vectors of standard deviations, i.e.,
\begin{equation}
\label{eq:w2_diag_linear_maps}
W_2^2\bigl(\mathcal{N}(\bm{0},\bm{\Sigma}),\mathcal{N}(\bm{0},\bm{\Sigma}')\bigr)
=
\sum_{i=1}^d (\sigma_i-\sigma_i')^2.
\end{equation}

%%%%%%%%%%%%%%%%%%%%%%%%%%%%%%%%%%%%%%%%%%%%%%%%%%%%%%%%%%%%%%%%%%%%%%%%%%%%%%%%
\section{Operator-merging foundations in the Gaussian regime}
\label{sec:operator_merging_foundations_in_the_gaussian_regime}

This section lays the theoretical foundations for trajectory distillation in a linear regime. In~\Cref{subsec:gaussian_assuption_and_the_optimal_denoising_estimator}, we first characterize the optimal denoising estimator under diagonal Gaussian data, and show that DDIM updates correspond to coordinate-wise linear operators. We then analyze the training dynamics of a student model that learns to approximate the multi-step teacher operator via gradient flow in~\Cref{subsec:optimization_dynamics_of_the_student}. Finally, in~\Cref{subsec:w2_metric_and_surrogate_target}, we introduce a variance-corrected Wasserstein metric to evaluate the distilled student model.

%%%%%%%%%%%%%%%%%%%%%%%%%%%%%%%%%%%%%%%%%%%%%%%%%%%%%%%%%%%%
\subsection{Gaussian assumption and the optimal denoising estimator}
\label{subsec:gaussian_assuption_and_the_optimal_denoising_estimator}

As diffusion models have an implicit bias toward learning Gaussian structures in data under limited capacity~\citep{li2024understanding}, we assume in~\Cref{as:real_data_distribution} that the real data distribution \(p_0\) is a centered Gaussian with diagonal covariance. This simplification, commonly used in prior work~\citep{li2024understanding,wang2023diffusion,wang2023hidden}, permits analytical tractability. While we restrict our focus to the single Gaussian case in this section to establish core principles, we will extend our analysis to Gaussian mixture in~\Cref{sec:approximation_error_and_error_propagation_in_the_gaussian_mixture_regime}.

\begin{assumption}[Gaussian distribution]
\label{as:real_data_distribution}
We assume that the real data distribution \(p_0\) is a \(d\)-dimensional centered Gaussian distribution with diagonal covariance matrix, i.e., \(p_0 = \mathcal{N}(\bm{0}, \bm{\Lambda})\), where \(\bm{\Lambda} = \mathrm{Diag}(\lambda_1, \dotsc, \lambda_d)\), and the eigenvalues satisfy \(\lambda_1\geq \lambda_2\geq\dotsb\geq\lambda_d\geq 0\).
\end{assumption}

The diagonality of the covariance captures key second-order structure while simplifying the analysis. It is justified when working in a transformed basis (e.g., via PCA~\citep{mackiewicz1993principal} or learned encoders), where dominant correlations are disentangled and the remaining components are approximately uncorrelated.

Under~\Cref{as:real_data_distribution}, we can derive a closed-form expression for the optimal denoiser \(\hat{\bm{x}}_0^\star(\bm{z}_t, t)=\mathbb{E}[\bm{x}_0|\bm{z}_t]\) which minimizes the denoising loss~\eqref{eq:denoising_loss_function}. \Cref{prop:optimal_denoising_estimator} shows that it takes the form of a \emph{linear} function in \(\bm{z}_t\).

\begin{proposition}[Optimal denoising estimator for Gaussian]
\label{prop:optimal_denoising_estimator}
Assume the real data distribution \(p_0\) is given by~\Cref{as:real_data_distribution}, the forward process follows~\eqref{eq:forward_process}, and the denoising estimator minimizes the denoising loss~\eqref{eq:denoising_loss_function}. Then the optimal denoising estimator for \(\bm{z}_t\) is
\begin{equation}
\hat{\bm{x}}_0^\star(\bm{z}_t, t) = \mathbb{E}[\bm{x}_0|\bm{z}_t] = \alpha_t\bm{\Lambda}(\alpha_t^2 \bm{\Lambda} + \sigma_t^2 \bm{I})^{-1}\bm{z}_t.
\end{equation}
\end{proposition}

\begin{proof}[Proof of~\Cref{prop:optimal_denoising_estimator}]
We prove a more general result. Suppose \(\bm{x}_0 \sim \mathcal{N}(\bm{\mu}, \bm{\Sigma})\), and the forward process is defined by~\eqref{eq:forward_process}, i.e.,
\begin{equation}
\bm{z}_t = \alpha_t \bm{x}_0 + \sigma_t \bm{\epsilon}, \quad \bm{\epsilon} \sim \mathcal{N}(\bm{0}, \bm{I}).
\end{equation}
Then the conditional distribution \(p(\bm{x}_0 | \bm{z}_t)\) is Gaussian, and the optimal estimator minimizing~\eqref{eq:denoising_loss_function} is its conditional expectation:
\begin{equation}
\mathbb{E}[\bm{x}_0 | \bm{z}_t] = \bm{A}_t \bm{z}_t + \bm{b}_t,
\end{equation}
where
\begin{equation}
\bm{A}_t = \alpha_t \bm{\Sigma}( \alpha_t^2 \bm{\Sigma} + \sigma_t^2 \bm{I})^{-1},
\quad
\bm{b}_t = ( \bm{I} - \alpha_t\bm{A}_t) \bm{\mu}.
\end{equation}
To derive this, we note that \(\bm{z}_t | \bm{x}_0 \sim \mathcal{N}(\alpha_t \bm{x}_0, \sigma_t^2 \bm{I})\), and hence the joint distribution is
\begin{equation}
\begin{bmatrix}
\bm{x}_0 \\
\bm{z}_t
\end{bmatrix}
\sim \mathcal{N}\Biggl(
\begin{bmatrix}
\bm{\mu} \\
\alpha_t \bm{\mu}
\end{bmatrix},
\begin{bmatrix}
\bm{\Sigma} & \alpha_t \bm{\Sigma} \\
\alpha_t \bm{\Sigma} & \alpha_t^2 \bm{\Sigma} + \sigma_t^2 \bm{I}
\end{bmatrix}
\Biggr).
\end{equation}
The conditional mean is given by the standard Gaussian conditioning formula
\begin{equation}
\mathbb{E}[\bm{x}_0 | \bm{z}_t] = \bm{\mu} + \alpha_t \bm{\Sigma} ( \alpha_t^2 \bm{\Sigma} + \sigma_t^2 \bm{I} )^{-1} (\bm{z}_t - \alpha_t \bm{\mu}).
\end{equation}
Rewriting this as an affine function of \(\bm{z}_t\), we obtain
\begin{equation}
\mathbb{E}[\bm{x}_0 | \bm{z}_t]
= \underbrace{\alpha_t \bm{\Sigma} ( \alpha_t^2 \bm{\Sigma} + \sigma_t^2 \bm{I})^{-1}}_{\bm{A}_t} \bm{z}_t
+ \underbrace{(\bm{I} - \alpha_t\bm{A}_t) \bm{\mu}}_{\bm{b}_t}.
\end{equation}
Now consider the case where \(\bm{\mu} = \bm{0}\) and \(\bm{\Sigma}\) is diagonal. More generally, suppose \(\bm{\Sigma} = \bm{U} \bm{\Lambda} \bm{U}^\top\) is the eigendecomposition of the covariance. Then
\begin{equation}
\bm{A}_t = \alpha_t \bm{U} \bm{\Lambda} ( \alpha_t^2 \bm{\Lambda} + \sigma_t^2 \bm{I} )^{-1} \bm{U}^\top,
\end{equation}
so the estimator becomes
\begin{equation}
\mathbb{E}[\bm{x}_0 | \bm{z}_t] = \bm{U} \bigl( \alpha_t \bm{\Lambda}( \alpha_t^2 \bm{\Lambda} + \sigma_t^2 \bm{I})^{-1} \bigr) \bm{U}^\top \bm{z}_t.
\end{equation}
In particular, if \(\bm{\Sigma} = \bm{\Lambda}\) is diagonal, then \(\bm{U} = \bm{I}\), and we recover the expression in the proposition:
\begin{equation}
\mathbb{E}[\bm{x}_0 | \bm{z}_t] = \alpha_t \bm{\Lambda} ( \alpha_t^2 \bm{\Lambda} + \sigma_t^2 \bm{I} )^{-1} \bm{z}_t.
\end{equation}
This completes the proof.
\end{proof}

Although~\Cref{prop:optimal_denoising_estimator} assumes a diagonal covariance matrix, the result generalizes. As shown in the proof, for \(p_0 = \mathcal{N}(\bm{\mu}, \bm{\Sigma})\), the optimal denoising estimator is \emph{affine} in \(\bm{z}_t\). This reduces to the linear form in~\Cref{prop:optimal_denoising_estimator} via an affine transformation that centers the mean and diagonalizes the covariance. As a corollary, the relationship between two consecutive steps in DDIM sampling is also captured by a linear operator, as shown in~\Cref{cor:optimal_single_operator}. The proof follows straightforwardly from substituting the optimal denoising estimator into the DDIM update~\eqref{eq:denosing_process}, so we omit it.

\begin{restatable}[Relationship between consecutive steps is linear]{corollary}{OptimalSingleOperator}
\label{cor:optimal_single_operator}
Following the assumptions in~\Cref{prop:optimal_denoising_estimator} and the DDIM sampling process~\eqref{eq:denosing_process}, the relationship between two consecutive steps can be captured by a linear operator:
\begin{equation}
\begin{aligned}
\bm{z}_{t-1} &= \alpha_{t-1}\hat{\bm{x}}_{0}^\star(\bm{z}_t, t) + \sigma_{t-1} \cdot \biggl( \frac{\bm{z}_t - \alpha_t \hat{\bm{x}}_0^\star(\bm{z}_t, t)}{\sigma_t} \biggr)\\
&= (\alpha_{t-1}\alpha_t\bm{\Lambda}+\sigma_{t-1}\sigma_t\bm{I})(\alpha_t^2\bm{\Lambda}+\sigma_t^2\bm{I})^{-1}\bm{z}_t.
\label{eq:ddim_step_linear}
\end{aligned}
\end{equation}
\end{restatable}

Since the update matrix \((\alpha_{t-1}\alpha_t\bm{\Lambda}+\sigma_{t-1}\sigma_t\bm{I})(\alpha_t^2\bm{\Lambda}+\sigma_t^2\bm{I})^{-1}\) is diagonal, each coordinate is updated independently. Let \(\lambda_i\) denote the \(i\)th diagonal entry of \(\bm{\Lambda}\), and let \((\bm{z}_t)_i\) and \((\bm{z}_{t-1})_i\) denote the \(i\)th coordinates of \(\bm{z}_t\) and \(\bm{z}_{t-1}\), respectively. Define the signal-noise vector 
\begin{equation}
\bm{v}_t^i \coloneqq (\alpha_t\sqrt{\lambda_i}, \sigma_t)^\top \in \mathbb{R}^2,
\label{eq:signal_noise_vector}
\end{equation}
which compactly encodes the relative strengths of signal and noise. Then the update can be written as
\begin{equation}
\label{eq:single_operator_formula}
(\bm{z}_{t-1})_i = \dfrac{\langle \bm{v}_{t-1}^i, \bm{v}_t^i\rangle}{\lVert\bm{v}_t^i\rVert_2^2}\cdot (\bm{z}_t)_i,
\end{equation}
where \(\langle \cdot, \cdot \rangle\) denotes the standard Euclidean inner product. This reveals that each coordinate is rescaled by the coefficient of the orthogonal projection of \(\bm{v}_{t-1}^i\) onto \(\bm{v}_t^i\). An illustration is provided in~\Cref{fig:vector_diagram}. By iterating this recurrence, the composite operator over \(k\) reverse steps in~\Cref{def:composite_teacher_operator} takes the form
\begin{equation}
\label{eq:composite_operator_formula}
(\mathcal{T}_k(\bm{z}_t))_i = \biggl(\prod_{j=t-k+1}^{t}\dfrac{\langle \bm{v}_{j-1}^i, \bm{v}_j^i\rangle}{\lVert\bm{v}_j^i\rVert_2^2}\biggr)\cdot (\bm{z}_t)_i.
\end{equation}

\begin{figure}[ht]
\centering
\includegraphics[width=1\linewidth]{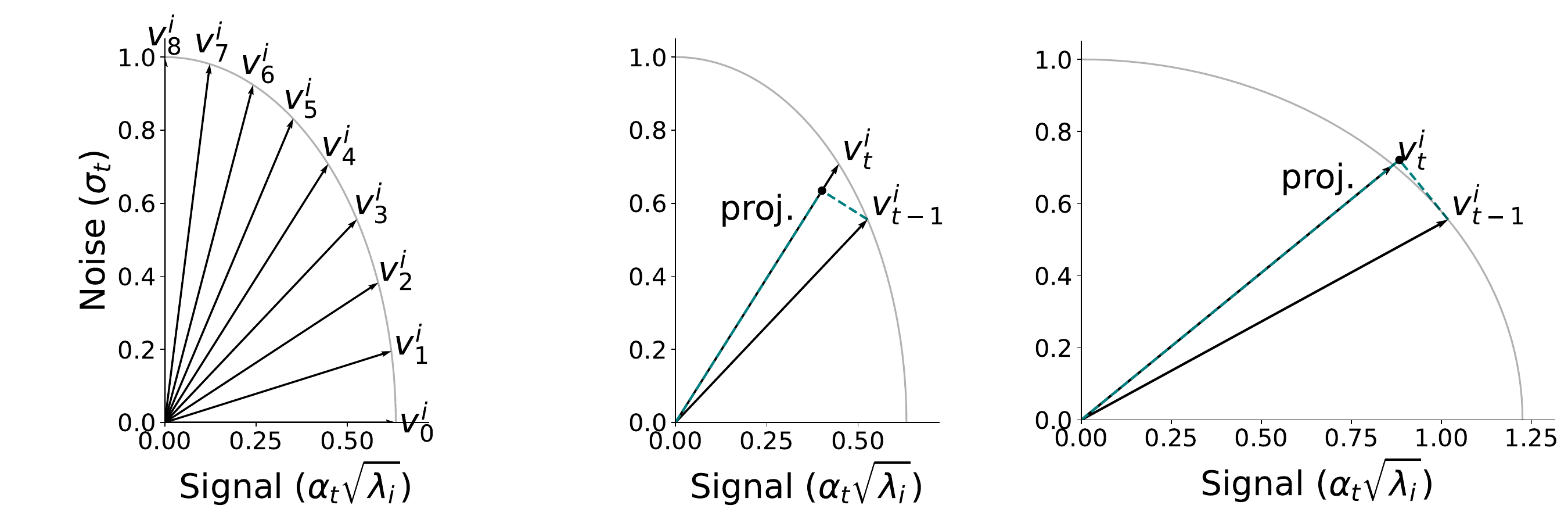}
\caption{
Geometric interpretation of the signal-noise vectors \(\bm{v}_t^i = (\alpha_t\sqrt{\lambda_i}, \sigma_t)\), which lies on an ellipse. The DDIM update~\eqref{eq:single_operator_formula} corresponds to projecting \(\bm{v}_{t-1}^i\) onto \(\bm{v}_t^i\) and computing its ratio with \(\lVert\bm{v}_t^i\rVert_2\). \textbf{Left}: full sequence of vectors (\(T=8\)). \textbf{Middle}: projection when \(\lambda_i < 1\). \textbf{Right}: projection when \(\lambda_i > 1\).}
\label{fig:vector_diagram}
\end{figure}

An important consequence of~\eqref{eq:composite_operator_formula} is that the teacher model cannot fully recover the original data distribution, even under perfect training (i.e., the denoising loss~\eqref{eq:denoising_loss_function} is minimized). This limitation arises from \emph{discretization error}: approximating the continuous-time reverse process via a finite sequence of DDIM steps inherently introduces information loss. As a result, the composite operator contracts the variance of the signal, which is formalized in~\Cref{prop:teacher_contracts_covariance}.

\begin{proposition}[Teacher model contracts the covariance]
\label{prop:teacher_contracts_covariance}
Assume that the noise schedule \(\{\alpha_t\}\) and \(\{\sigma_t\}\) satisfies~\Cref{as:noise_schedule}. Under the composite operator defined in~\eqref{eq:composite_operator_formula}, each output coordinate satisfies \(\lvert(\mathcal{T}_T(\bm{z}_T))_i\rvert \leq \sqrt{\lambda_i} \cdot \lvert(\bm{z}_T)_i\rvert\), with equality if and only if \((\bm{z}_T)_i = 0\). Consequently, if \(\bm{z}_T \sim \mathcal{N}(\bm{0}, \bm{I})\), then the teacher model induces a contracted output covariance, i.e., \(\mathrm{Cov}\bigl(\mathcal{T}_T(\bm{z}_T)\bigr) \preceq \bm{\Lambda}\), where the inequality is strict if \(\bm{\Lambda}\succ\bm{0}\).
\end{proposition}

\begin{proof}[Proof of~\Cref{prop:teacher_contracts_covariance}]
Fix a coordinate \(i\). We consider the evolution of the \(i\)th component in the DDIM process. From~\eqref{eq:composite_operator_formula}, we know that the teacher applies the operator
\begin{equation}
(\mathcal{T}_T(\bm{z}_T))_i = \biggl( \prod_{t=1}^{T} \frac{\langle \bm{v}_{t-1}^i, \bm{v}_t^i \rangle}{\lVert\bm{v}_t^i\rVert_2^2} \biggr) \cdot (\bm{z}_T)_i,
\end{equation}
where \(\bm{v}_t^i = (\alpha_t \sqrt{\lambda_i}, \sigma_t)^\top\in \mathbb{R}^2\) is the signal-noise vector defined in~\eqref{eq:signal_noise_vector}. By the Cauchy--Schwarz inequality, we have for all \(i\):
\begin{equation}
\frac{|\langle \bm{v}_{t-1}^i, \bm{v}_t^i \rangle|}{\lVert\bm{v}_t^i\rVert_2^2}
\le \frac{\lVert\bm{v}_{t-1}^i\rVert_2 \cdot \lVert\bm{v}_t^i\rVert_2}{\lVert\bm{v}_t^i\rVert_2^2}
= \frac{\lVert\bm{v}_{t-1}^i\rVert_2}{\lVert\bm{v}_t^i\rVert_2}.
\end{equation}
Taking the product from \(t=1\) to \(T\), we obtain
\begin{equation}
\prod_{t=1}^{T} \frac{\langle \bm{v}_{t-1}^i, \bm{v}_t^i \rangle}{\lVert\bm{v}_t^i\rVert_2^2}
\le \prod_{t=1}^{T} \frac{\lVert\bm{v}_{t-1}^i\rVert_2}{\lVert\bm{v}_t^i\rVert_2}
= \frac{\lVert\bm{v}_0^i\rVert_2}{\lVert\bm{v}_T^i\rVert_2}.
\end{equation}
Using the boundary conditions \(\alpha_0 = 1\), \(\sigma_0 = 0\), \(\alpha_T = 0\), and \(\sigma_T = 1\), we find
\begin{equation}
\lVert\bm{v}_0^i\rVert_2 = \sqrt{\lambda_i}, \quad \lVert\bm{v}_T^i\rVert_2 = 1,
\end{equation}
which implies
\begin{equation}
\lvert(\mathcal{T}_T(\bm{z}_T))_i\rvert \le \sqrt{\lambda_i} \cdot \lvert(\bm{z}_T)_i\rvert.
\end{equation}
Equality holds if and only if equality holds in the Cauchy--Schwarz inequality at each step, which occurs if and only if all vectors \(\bm{v}_{j-1}^i\) and \(\bm{v}_j^i\) are positively collinear. Since the signal-noise vectors vary across time steps due to the schedule \((\alpha_t, \sigma_t)\), this happens only when the projection scalar acts on zero input, i.e., \((\bm{z}_T)_i = 0\).

Now assume \(\bm{z}_T \sim \mathcal{N}(\bm{0}, \bm{I})\). Then
\begin{equation}
\mathrm{Var}\bigl((\mathcal{T}_T(\bm{z}_T))_i\bigr)
= \Biggl( \prod_{j=1}^{T} \frac{\langle \bm{v}_{j-1}^i, \bm{v}_j^i \rangle}{\lVert\bm{v}_j^i\rVert_2^2} \Biggr)^2
\le \lambda_i,
\end{equation}
with strict inequality whenever \((\bm{z}_T)_i \ne 0\). As this holds for each coordinate independently, we conclude that
\begin{equation}
\mathrm{Cov}(\mathcal{T}_T(\bm{z}_T)) \preceq \bm{\Lambda},
\end{equation}
with strict inequality if \(\bm{\Lambda} \succ \bm{0}\).
\end{proof}

%%%%%%%%%%%%%%%%%%%%%%%%%%%%%%%%%%%%%%%%%%%%%%%%%%%%%%%%%%%%
\subsection{Optimization dynamics of the student}
\label{subsec:optimization_dynamics_of_the_student}

Following~\Cref{cor:optimal_single_operator}, we use the parameterization \(\bm{f}_{\bm{\eta}}\) in~\eqref{eq:denosing_process} and model both the teacher and student as time-dependent linear operators. We assume that the student model takes the form
\begin{equation}
\label{eq:student_model}
\bm{f}^{\mathrm{st}}(\bm{z}_t, t) = \bm{A}_t^{\mathrm{st}} \bm{z}_t, \quad \bm{A}_t^{\mathrm{st}} = \mathrm{Diag}(a_t^1, \dotsc, a_t^d),
\end{equation}
with learnable diagonal weights. Following standard practice, we initialize the student model to match the single-step optimal operator (i.e., the teacher model)~\citep{luhman2021knowledge,salimans2022progressive},
\begin{equation}
\label{eq:teacher_model}
\bm{A}_t^\star \coloneqq (\alpha_{t-1}\alpha_t\bm{\Lambda}+\sigma_{t-1}\sigma_t\bm{I})(\alpha_t^2\bm{\Lambda}+\sigma_t^2\bm{I})^{-1},
\end{equation}
which corresponds to the optimal single-step DDIM update in~\Cref{cor:optimal_single_operator}. To train the student model, we define a trajectory distillation loss in~\Cref{as:trajectory_distillation_loss} that encourages the student to match the teacher's behavior across multiple steps. Specifically, the loss at time \(t\) measures the mean squared error between the student prediction and the multi-step composite operator output \(\mathcal{T}_k(\bm{z}_t)\).

\begin{assumption}[Trajectory distillation loss]
\label{as:trajectory_distillation_loss}
Let \(\bm{z}_t \sim p_t = (\alpha_t p_0) * \mathcal{N}(\bm{0}, \sigma_t^2 \bm{I})\) be a noisy input at time \(t\), and let the student model be defined by~\eqref{eq:student_model}. We assume that the distillation loss at time \(t\) is given by
\begin{equation}
\mathcal{L}_t(\bm{A}_t^{\mathrm{st}}) \coloneqq \mathbb{E}_{\bm{z}_t \sim p_t}\bigl[\lVert\bm{A}_t^{\mathrm{st}}\bm{z}_t - \mathcal{T}_k(\bm{z}_t)\rVert_2^2\bigr].
\end{equation}
\end{assumption}

For each coordinate \(i\), let \(a_t^i(s)\) denote the student model's value at training time \(s\) (distinct from the diffusion time step \(t\), which indexes the noise level). The student aims to approximate the \(i\)th entry of the composite operator, denoted \((\mathcal{T}_k)_i\)\footnote{We slightly abuse notation by identifying the linear operator \(\mathcal{T}_k\), which acts on \(\bm{z}_t\), with its corresponding matrix form such that \(\mathcal{T}_k(\bm{z}_t) = \mathcal{T}_k \bm{z}_t\). In particular, \((\mathcal{T}_k)_i\) refers to the \(i\)th diagonal entry of this matrix.}. The gradient flow of the distillation loss in~\Cref{as:trajectory_distillation_loss} leads to the following scalar ordinary differential equation (ODE) for each coordinate:
\begin{equation}
\label{eq:gradient_flow_ode}
\frac{\mathrm{d}}{\mathrm{d}s} a_t^i(s) = -\nabla_{a_t^i}\mathcal{L}_t(\bm{A}_t^{\mathrm{st}}), \quad a_t^i(0) = (\bm{A}_t^\star)_i.
\end{equation}
This ODE admits a closed-form solution, which is presented in~\Cref{prop:student_gradient_flow}.

\begin{proposition}[Student interpolation under gradient flow]
\label{prop:student_gradient_flow}
Let \(\lambda_i\) be the \(i\)th diagonal entry of \(\bm{\Lambda}\), and let \((\bm{A}_t^\star)_i\) be the corresponding entry of the teacher operator. Under the gradient flow ODE~\eqref{eq:gradient_flow_ode}, the student operator \(a_t^i(s)\) evolves as
\begin{equation}
a_t^i(s) = (1 - \gamma_t^i(s)) \cdot (\mathcal{T}_k)_i + \gamma_t^i(s) \cdot (\bm{A}_t^\star)_i,
\end{equation}
where the interpolation weight \(\gamma_t^i(s) = \exp(-2s\lVert\bm{v}_t^i\rVert_2^2)\), with \(\bm{v}_t^i = (\alpha_t\sqrt{\lambda_i}, \sigma_t)^\top \in \mathbb{R}^2\) and \(\lVert\bm{v}_t^i\rVert_2^2 = \alpha_t^2 \lambda_i + \sigma_t^2\).
\end{proposition}

\begin{proof}[Proof of~\Cref{prop:student_gradient_flow}]
Fix a coordinate \(i\). Let the student operator at training time \(s\) be \(a_t^i(s)\). Recall that the composite teacher target is \((\mathcal{T}_k)_i\), and that the student is initialized from the single-step teacher
\begin{equation}
a_t^i(0) = (\bm{A}_t^\star)_i.
\end{equation}
From~\Cref{as:trajectory_distillation_loss}, the scalar loss in coordinate \(i\) is
\begin{equation}
\mathcal{L}_t^i(a_t^i) = \mathbb{E}_{\bm{z}_t \sim p_t} \bigl[ \bigl( a_t^i (\bm{z}_t)_i - (\mathcal{T}_k(\bm{z}_t))_i \bigr)^2 \bigr].
\end{equation}
Since \(\bm{z}_t \sim \mathcal{N}(\bm{0}, \alpha_t^2 \bm{\Lambda} + \sigma_t^2 \bm{I})\), we have \((\bm{z}_t)_i \sim \mathcal{N}(0, \alpha_t^2 \lambda_i + \sigma_t^2)\). The loss then becomes
\begin{equation}
\mathcal{L}_t^i(a_t^i) = (\alpha_t^2 \lambda_i + \sigma_t^2) \cdot (a_t^i)^2 - 2(\alpha_t^2 \lambda_i + \sigma_t^2)(\mathcal{T}_k)_i \cdot a_t^i + (\alpha_t^2 \lambda_i + \sigma_t^2)(\mathcal{T}_k)_i^2.
\end{equation}
Taking the gradient with respect to \(a_t^i\) and applying gradient flow gives
\begin{equation}
\frac{\mathrm{d}}{\mathrm{d}s} a_t^i(s) = -\frac{\mathrm{d}}{\mathrm{d}a_t^i} \mathcal{L}_t^i(a_t^i)
= -2(\alpha_t^2 \lambda_i + \sigma_t^2) \cdot \bigl(a_t^i(s) - (\mathcal{T}_k)_i\bigr).
\end{equation}
This is a first-order linear ODE
\begin{equation}
\frac{\mathrm{d}}{\mathrm{d}s} a_t^i(s) + 2(\alpha_t^2 \lambda_i + \sigma_t^2) a_t^i(s) = 2(\alpha_t^2 \lambda_i + \sigma_t^2)(\mathcal{T}_k)_i,
\end{equation}
with initial condition \(a_t^i(0) = (\bm{A}_t^\star)_i\). Using the integrating factor method, we solve
\begin{equation}
a_t^i(s) = (\mathcal{T}_k)_i + \bigl((\bm{A}_t^\star)_i - (\mathcal{T}_k)_i\bigr) \cdot \exp\bigl(-2(\alpha_t^2 \lambda_i + \sigma_t^2)s\bigr).
\end{equation}
Rewriting this as a convex combination gives
\begin{equation}
a_t^i(s) = (1 - \gamma_t^i(s)) \cdot (\mathcal{T}_k)_i + \gamma_t^i(s) \cdot (\bm{A}_t^\star)_i,
\end{equation}
where
\begin{equation}
\gamma_t^i(s) = \exp\bigl(-2(\alpha_t^2 \lambda_i + \sigma_t^2)s\bigr) = \exp\bigl(-2s\lVert\bm{v}_t^i\rVert_2^2\bigr),
\end{equation}
with \(\bm{v}_t^i = (\alpha_t \sqrt{\lambda_i}, \sigma_t)^\top\). This concludes the proof.
\end{proof}

\Cref{prop:student_gradient_flow} reveals a fundamental limitation of trajectory distillation. Under finite training time \(s\), the student does not fully match the composite operator \((\mathcal{T}_k)_i\), but instead interpolates between the composite operator and the single-step operator \((\bm{A}_t^\star)_i\). The interpolation weight \(\gamma_t^i(s) \in (0,1)\) depends exponentially on both \(s\) and the signal-to-noise ratio \(\lVert\bm{v}_t^i\rVert_2^2 = \alpha_t^2 \lambda_i + \sigma_t^2\), and quantifies the deviation from the target.

%%%%%%%%%%%%%%%%%%%%%%%%%%%%%%%%%%%%%%%%%%%%%%%%%%%%%%%%%%%%
\subsection{A Wasserstein metric and a variance-corrected target}
\label{subsec:w2_metric_and_surrogate_target}

With the teacher model's inherent covariance contraction (\Cref{prop:teacher_contracts_covariance}) and the student model's training dynamics (\Cref{prop:student_gradient_flow}) established, we now define a metric to rigorously evaluate the distilled student. Our goal is to compress the entire trajectory into a single diagonal operator that maps noise \(\bm{z}_T \sim \mathcal{N}(\bm{0}, \bm{I})\) to a clean sample. In this linear Gaussian setting, the squared \(2\)-Wasserstein distance (reviewed in~\Cref{subsec:wasserstein_distance_for_gaussian_outputs}) is the natural measure of discrepancy.

However, directly minimizing the distance to the composite teacher \(\mathcal{T}_T\) is problematic. As shown in~\Cref{prop:teacher_contracts_covariance}, \(\mathcal{T}_T\) strictly shrinks the covariance when \(\bm \Lambda\succ\bm0\). Treating \(\mathcal{T}_T\) as the ground truth would therefore force the student to reproduce these discretization errors. Ideally, the student model should replicate the teacher model's denoising behavior in dimensions where the signal naturally decays (\(\lambda_i \le 1\)), while correcting for artificial shrinkage in dimensions with high data variance (\(\lambda_i > 1\)). To achieve this, we introduce the \emph{surrogate composite operator} \(\widetilde{\mathcal{T}}_T\) in~\Cref{def:surrogate_composite_operator}.

\begin{definition}[Surrogate composite operator]
\label{def:surrogate_composite_operator}
Let \(\widetilde{\mathcal{T}}_T\) be a surrogate composite operator defined coordinate-wise. For each dimension \(i \in \{1, \dotsc, d\}\), we define
\begin{equation}
(\widetilde{\mathcal{T}}_T)_i = \prod_{t=1}^T (\widetilde{\bm{A}}_t^\star)_i,
\end{equation}
where the \((\widetilde{\bm{A}}_t^\star)_i\)'s are adjusted to prevent magnitude reduction in high-variance dimensions:
\begin{equation}
(\widetilde{\bm{A}}_t^\star)_i \coloneqq 
\begin{cases}
(\bm{A}_t^\star)_i, & \text{if } \lambda_i \leq 1 \text{ or } (\bm{A}_t^\star)_i \geq 1, \\
1, & \text{if } \lambda_i > 1 \text{ and } (\bm{A}_t^\star)_i < 1.
\end{cases}
\end{equation}
\end{definition}

We evaluate the final student operator, denoted as \(\bm{A}_{1{:}T}\)\footnote{The subscript \(1{:}T\) indicates that the student approximates the merge of the full sequence of \(T\) teacher steps. A formal definition of this notation is provided in~\Cref{sec:variance_driven_optimal_merge_plans_and_pareto_dp_in_the_gaussian_regime}.}, by its distance to this surrogate. Using the coordinate-wise decomposition of the Wasserstein distance from~\eqref{eq:w2_diag_linear_maps}, we define the per-coordinate squared error \(\delta_i\) and the total objective \(\mathcal{L}_{W_2}\) as
\begin{equation}
\delta_i \coloneqq \bigl( (\widetilde{\mathcal{T}}_T)_i - (\bm{A}_{1{:}T})_i \bigr)^2, \quad \mathcal{L}_{W_2}\coloneqq\sum_{i=1}^{d}\delta_i.
\end{equation}
This metric \(\mathcal{L}_{W_2}\) separates the distillation error from the teacher model's intrinsic discretization bias, providing a fair standard for comparing different merging strategies.

%%%%%%%%%%%%%%%%%%%%%%%%%%%%%%%%%%%%%%%%%%%%%%%%%%%%%%%%%%%%%%%%%%%%%%%%%%%%%%%%
\section{Variance-driven optimal merge plans and Pareto DP in the Gaussian regime}
\label{sec:variance_driven_optimal_merge_plans_and_pareto_dp_in_the_gaussian_regime}

In~\Cref{sec:operator_merging_foundations_in_the_gaussian_regime}, we established the theoretical foundations for our analysis: we showed that DDIM updates act as coordinate-wise linear operators, derived the student's shrinkage dynamics under gradient flow, and defined a variance-corrected Wasserstein metric to isolate distillation error. Building on these results, this section casts trajectory distillation as an \emph{operator merging} problem. Our goal is to find the optimal sequence of merge operations that minimizes the squared Wasserstein distance to the surrogate teacher operator defined in~\Cref{def:surrogate_composite_operator}.

We formalize this as a recursive merging process. Let \(\bm{v}_t^i = (\alpha_t \sqrt{\lambda_i}, \sigma_t)^\top\) be the signal-noise vector for coordinate \(i\) defined in~\eqref{eq:single_operator_formula}. We initialize the student operators at step \(k=0\) to match the single-step teacher operators:
\begin{equation}
\label{eq:def_of_At}
(\bm{A}_t^{(0)})_i \coloneqq (\bm{A}_t^\star)_i = \frac{\langle \bm{v}_{t-1}^i, \bm{v}_t^i \rangle}{\lVert\bm{v}_t^i\rVert_2^2}, \quad t = 1, \dotsc, T.
\end{equation}
At each planning step \(k \geq 1\), the student merges a \emph{contiguous} block of updates from \(t_1\) to \(t_2\). Based on~\Cref{prop:student_gradient_flow}, the resulting merged operator \(\bm{A}_{t_1{:}t_2}^{(k)}\) is defined coordinate-wise as
\begin{equation}
(\bm{A}_{t_1{:}t_2}^{(k)})_i = (1 - \gamma_{t_2}^i) \cdot \prod_{t = t_1}^{t_2} (\bm{A}_t^{(k-1)})_i + \gamma_{t_2}^i \cdot (\bm{A}_{t_2}^{(k-1)})_i.
\label{eq:recursive_merge}
\end{equation}
Here, \(\gamma_{t_2}^i = \exp(-2s \lVert\bm{v}_{t_2}^i\rVert_2^2)\) is the shrinkage factor derived in~\Cref{prop:student_gradient_flow}, where \(s > 0\) represents the optimization time per merge. This process repeats until a complete trajectory \(\bm{A}_{1{:}T}^{(K)}\) is constructed. We then evaluate the quality of a merging strategy using the objective \(\mathcal{L}_{W_2}\) (defined in~\Cref{subsec:w2_metric_and_surrogate_target}), which measures the deviation of the final student operator from the surrogate teacher operator \(\widetilde{\mathcal{T}}_T\).

In~\Cref{subsec:phase_transitions_in_the_optimal_strategy}, we analyze the scalar setting and uncover a sharp phase transition governed by data variance. The optimal merging strategy shifts from sequential BOOT in low-variance regimes to vanilla distillation in high-variance regimes. Guided by this theoretical characterization, \Cref{subsec:pareto_dynamic_programming_for_optimal_merging} addresses the general high-dimensional setting, where different dimensions may favor conflicting strategies, by proposing a Pareto dynamic programming algorithm that systematically computes the globally optimal merge plan.

%%%%%%%%%%%%%%%%%%%%%%%%%%%%%%%%%%%%%%%%%%%%%%%%%%%%%%%%%%%%
\subsection{Phase transition in the optimal strategy}
\label{subsec:phase_transitions_in_the_optimal_strategy}

We first analyze the scalar setting (\(d=1\)). Let \(\bm{\Lambda}=\lambda\) and denote the scalar counterparts of the operator and shrinkage factors as \(A_t^\star\) and \(\gamma_t\). In this simplified regime, we identify a sharp phase transition in the optimal merging strategy governed by the data variance \(\lambda\). Specifically, the optimal plan recovers \emph{sequential BOOT} when \(\lambda \le 1\) and collapses to \emph{vanilla trajectory distillation} when \(\lambda \gg 1\). We formalize these claims in~\Cref{thm:sequential_boot_optimality,thm:vanilla_distillation_optimality}.

When \(\lambda \gg 1\), the optimal merging strategy is to perform a one-shot merge over all denoising steps, corresponding to vanilla trajectory distillation. In this regime, the intermediate single-step operators satisfy \(A_t^\star > 1\) for all \(t < T\), indicating that each step amplifies the signal with the exception of the final step. Merging these steps early helps preserve this amplification. The final step \(A_T^\star < 1\) introduces a small contraction, but this effect is minimal when preceded by strong signal gains. We formally establish the optimality of vanilla trajectory distillation in~\Cref{thm:vanilla_distillation_optimality}.

\begin{theorem}[Vanilla trajectory distillation is optimal when \(\lambda \gg 1\)]
\label{thm:vanilla_distillation_optimality}
Consider the recursive operator merging process defined in~\Cref{sec:variance_driven_optimal_merge_plans_and_pareto_dp_in_the_gaussian_regime} restricted to the scalar setting (\(d=1\)). And assume \(\bm{\Lambda} = \lambda \gg 1\). Then, for the sequence of operators \(\{A_1^\star, A_2^\star, \dotsc, A_T^\star\}\) in~\eqref{eq:def_of_At}, the optimal merging strategy that minimizes the squared error to the surrogate target \(\prod_{t=1}^T \widetilde{A}_t^\star\) is the vanilla trajectory distillation: all steps are merged in a single one-shot update without intermediate composition.
\end{theorem}

\begin{proof}[Proof of~\Cref{thm:vanilla_distillation_optimality}]
We begin by showing that when \(\lambda\) is sufficiently large, the sequence of operators \(\{A_1^\star, A_2^\star, \dotsc, A_T^\star\}\) satisfies
\begin{equation}
A_1^\star, A_2^\star, \dotsc, A_{T-1}^\star > 1, \quad A_T^\star < 1.
\end{equation}
By definition, 
\begin{equation}
A_t^\star = \frac{\alpha_{t-1}\alpha_t \lambda + \sigma_{t-1}\sigma_t}{\alpha_t^2 \lambda + \sigma_t^2}.
\end{equation}
Since all quantities in the numerator and denominator are positive (note that for \(t<T\) we have \(\alpha_t>0\) and \(\sigma_t>0\)), the inequality \(A_t^\star > 1\) is equivalent to
\begin{equation}
\alpha_{t-1}\alpha_t \lambda + \sigma_{t-1}\sigma_t > \alpha_t^2 \lambda + \sigma_t^2.
\end{equation}
Rearrange this inequality by subtracting \(\alpha_t^2 \lambda + \sigma_{t-1}\sigma_t\) from both sides:
\begin{equation}
\alpha_{t-1}\alpha_t \lambda - \alpha_t^2 \lambda > \sigma_t^2 - \sigma_{t-1}\sigma_t.
\end{equation}
Factor out common factors:
\begin{equation}
\alpha_t (\alpha_{t-1} - \alpha_t)  \lambda > \sigma_t (\sigma_t - \sigma_{t-1}).
\end{equation}
Thus,
\begin{equation}
\lambda > \frac{\sigma_t(\sigma_t - \sigma_{t-1})}{\alpha_t(\alpha_{t-1} - \alpha_t)} .
\end{equation}
Define
\begin{equation}
\lambda_0(t) \coloneqq \frac{\sigma_t(\sigma_t - \sigma_{t-1})}{\alpha_t(\alpha_{t-1} - \alpha_t)}.
\end{equation}
Then, if
\begin{equation}
\lambda > \lambda_0(t),
\end{equation}
we have \(A_t^\star > 1\) for that particular \(t\).

Since the noise schedule is fixed and \(t\) takes values in the finite set \(\{1, 2, \dotsc, T-1\}\), we may define
\begin{equation}
\lambda_0^* \coloneqq \max_{1\le t\le T-1}\lambda_0(t).
\end{equation}
Hence, if
\begin{equation}
\lambda > \lambda_0^*,
\end{equation}
it follows that
\begin{equation}
A_t^\star > 1, \quad t = 1,2,\dots,T-1.
\end{equation}
Now, consider \(t = T\). By the boundary conditions, \(\alpha_T = 0\) and \(\sigma_T = 1\). Thus,
\begin{equation}
A_T^\star = \frac{\alpha_{T-1}\alpha_T \lambda + \sigma_{T-1}\sigma_T}{\alpha_T^2 \lambda + \sigma_T^2} 
= \frac{0 + \sigma_{T-1}\cdot 1}{0 + 1} 
= \sigma_{T-1}.
\end{equation}
Since \(\sigma_T = 1\) and \(\sigma_t\) is strictly increasing, we have \(\sigma_{T-1} < 1\). Therefore,
\begin{equation}
A_T^\star = \sigma_{T-1} < 1.
\end{equation}

We aim to show that the optimal merging strategy, meaning the one whose merged outcome is closest to the surrogate target \(\prod_{t=1}^T\widetilde{A}_t^\star\), is given by vanilla trajectory distillation. We proceed by induction on the block size \(\ell = t_2 - t_1\) where we require \(1\leq t_1<t_2\leq T\). When \(\ell = 1\), the strategy is trivially optimal since there is only one way to merge the operators. For the base case \(\ell = 2\), we let \(t \coloneqq t_2\) and analyze three consecutive operators:
\begin{equation}
A_{t-2}^\star, A_{t-1}^\star, A_t^\star,
\end{equation}
satisfying
\begin{equation}
A_{t-2}^\star, A_{t-1}^\star, A_t^\star > 1.
\end{equation}
Our goal is to approximate the surrogate target
\begin{equation}
\prod_{t=1}^T\widetilde{A}_t^\star.
\end{equation}
We apply shrinkage when a merge ends at time \(t\), denoted by \(\gamma_t \in (0,1)\), and similarly let \(\gamma_{t-1}\) be the shrinkage applied at time \(t-1\). We compare three candidate merging strategies. The first is vanilla trajectory distillation, which merges all three operators in one step:
\begin{equation}
A_{\text{vanilla}} \coloneqq (1-\gamma_t)A_{t-2}^\star A_{t-1}^\star A_t^\star + \gamma_t A_t^\star.
\end{equation}
The second is the sequential BOOT strategy, which splits after \(A_{t-2}^\star\), merges \(\{A_{t-1}^\star, A_t^\star\}\) as a block, and treats \(A_{t-2}^\star\) separately:
\begin{equation}
A_{\text{BOOT}} \coloneqq \bigl((1-\gamma_t)A_{t-1}^\star A_t^\star + \gamma_t A_t^\star\bigr)\bigl((1-\gamma_t)A_{t-2}^\star + \gamma_t\bigr).
\end{equation}
The third is the sequential consistency strategy, which first merges \(\{A_{t-2}^\star, A_{t-1}^\star\}\) and then merges the result with \(A_t^\star\):
\begin{equation}
A_{\text{consistency}} \coloneqq A_t^\star\bigl((1-\gamma_t)\bigl((1-\gamma_{t-1})A_{t-2}^\star A_{t-1}^\star + \gamma_{t-1} A_{t-1}\bigr) + \gamma_t\bigr).
\end{equation}
We compare the approximation errors of these strategies. Since all \(A_{t-2}^\star, A_{t-1}^\star, A_t^\star > 1\), each merged outcome is strictly smaller than the target. Therefore, the strategy that yields the largest merged value is the closest to the target and results in the lowest error. First, consider the difference between the vanilla and BOOT outcomes
\begin{equation}
\begin{aligned}
&A_{\text{vanilla}} - A_{\text{BOOT}}\\
=\;& (1-\gamma_t)(A_{t-2}^\star A_{t-1}^\star A_t^\star) - \gamma_t A_t^\star \\
&- \bigl((1-\gamma_t)A_{t-1}^\star A_t^\star + \gamma_t A_t^\star\bigr)\bigl((1-\gamma_t)A_{t-2}^\star + \gamma_t\bigr)\\
=\;& \gamma_t(1-\gamma_t)A_{t-2}^\star A_{t-1}^\star A_t^\star\\
&-\gamma_t(1-\gamma_t)(A_t^\star A_{t-2}^\star+A_t^\star A_{t-1}^\star)-(\gamma_t^2-\gamma_t)A_t^\star\\
=\;&\gamma_t(1-\gamma_t)A_t^\star(A_{t-2}^\star-1)(A_{t-1}^\star-1)\\
\geq\;&0.
\end{aligned}
\end{equation}
Next, compare the vanilla and consistency outcomes:
\begin{equation}
\begin{aligned}
&A_{\text{vanilla}} - A_{\text{consistency}}\\ =\;& (1-\gamma_t)\gamma_{t-1}A_t^\star A_{t-1}^\star A_{t-2}^\star-(1-\gamma_t)\gamma_{t-1}A_t^\star A_{t-1}^\star\\
\geq\;&0.
\end{aligned}
\end{equation}
In both cases, the vanilla strategy produces the largest outcome and best approximates the surrogate target. This completes the base case.

Now assume the vanilla strategy is optimal for all block sizes of length \(\ell\). We prove that it remains optimal for block size \(\ell + 1\). Let \(t \coloneqq t_2\) be the endpoint of the interval. Consider any alternative strategy that splits the block at position \(t - \ell + k\) for some \(0 < k < \ell\). By the induction hypothesis, both the left and right sub-blocks are best approximated by direct merges. The resulting merge is therefore equivalent to a two-step merge between two composite operators, each strictly larger than 1, which reduces to the \(\ell = 3\) case already analyzed. This implies that the vanilla strategy remains optimal for all intervals \(1 \le t_1 < t_2 < T\).

However, since \(A_T^\star < 1\), we cannot directly apply this result to the full interval \([1, T]\). Nevertheless, it can be shown that the optimal strategy over \([1, T]\) must take one of two possible forms: (\romannumeral1) a one-shot merge of all steps from \(1\) to \(T\), or  
(\romannumeral2) a split at some index \(s\), where \([1, s]\) is merged, \([s+1, T]\) is merged, and the two results are then combined. For any such split index \(s\in\{1,\dots,T-1\}\), define
\begin{equation}
P_L = \prod_{t=1}^s A_t^\star, \quad P_R = \prod_{t=s+1}^{T} A_t^\star.
\end{equation}
In the split strategy, the left block is merged using shrinkage factor \(\gamma_s\):
\begin{equation}
A_L = (1 - \gamma_s) P_L + \gamma_s A_s^\star,
\end{equation}
and the right block is merged using \(\gamma_T\):
\begin{equation}
A_R = (1 - \gamma_T) P_R + \gamma_T A_T^\star.
\end{equation}
These two results are then merged using \(\gamma_T\) to yield
\begin{equation}
A_{\text{split}} = A_R \cdot \bigl((1 - \gamma_T) A_L + \gamma_T\bigr).
\end{equation}
In contrast, the vanilla one-shot merge yields
\begin{equation}
A_{\text{direct}} = (1 - \gamma_T)P_L P_R + \gamma_T A_T^\star.
\end{equation}
We now show that \(A_{\text{split}} < A_{\text{direct}}\). Computing the difference:
\begin{equation}
\begin{aligned}
&A_{\text{split}} - A_{\text{direct}}\\ 
=&\; (1-\gamma_T)\bigl((1-\gamma_s)P_L+\gamma_s A_s^\star\bigr)\bigl((1-\gamma_T)P_R+\gamma_T A_T^\star\bigr)\\
&+\gamma_T(1-\gamma_T)P_R+\gamma_T^2 A_T^\star-\gamma_T A_T^\star-(1-\gamma_T)P_LP_R\\
=&\; (1-\gamma_T)\bigl((\gamma_s\gamma_T-\gamma_s-\gamma_T)P_LP_R+\gamma_s(1-\gamma_T)A_s^\star P_R\\
&+\gamma_T(1-\gamma_s)P_L A_T^\star+\gamma_T P_R-\gamma_T A_T^\star\bigr)\\
\leq &\; \gamma_T(1-\gamma_T)\bigl(-P_LP_R+(1-\gamma_s)P_L A_T^\star\\
&+(1-\gamma_T)\gamma_sA_s^\star A_T^\star+P_R-A_T^\star\bigr)\\
\leq &\; \gamma_T(1-\gamma_T)\bigl(-P_L A_T^\star+(1-\gamma_s)P_L A_T^\star+(1-\gamma_s)\gamma_s A_s^\star A_T^\star\bigr)\\
\leq &\; 0.
\end{aligned}
\end{equation}
This shows that any split strategy is suboptimal and that the vanilla strategy yields the largest merged value. The proof is complete.
\end{proof}

In contrast, when \(\lambda \leq 1\), all the single-step operators satisfy \(A_t^\star \leq 1\), and the trajectory contracts monotonically. In this regime, the best approach is to merge steps gradually in a reverse order, starting from the highest noise level. This corresponds to the sequential BOOT strategy, where the student is trained to merge \(\{A_T^\star, A_{T-1}^\star\}\), then merges the result with \(A_{T-2}^\star\), and so on. This strategy minimizes compounding error by ensuring that each merge occurs over a shorter horizon with limited shrinkage. The following~\Cref{thm:sequential_boot_optimality} establishes the optimality of this strategy.

\begin{theorem}[Sequential BOOT is optimal when \(\lambda \leq 1\)]
\label{thm:sequential_boot_optimality}
Consider the recursive operator merging process defined in~\Cref{sec:variance_driven_optimal_merge_plans_and_pareto_dp_in_the_gaussian_regime} restricted to the scalar setting (\(d=1\)). And assume \(\bm{\Lambda} = \lambda \leq 1\). Then, for the sequence of operators \(\{A_1^\star, A_2^\star, \dotsc, A_T^\star\}\) in~\eqref{eq:def_of_At}, the optimal merging strategy that minimizes the squared error to the surrogate target \(\prod_{t=1}^T A_t^\star\) (which coincides with the original composite operator when \(\lambda \leq 1\)) is sequential BOOT: starting from \(A_T^\star\), iteratively merge it with \(A_{T-1}^\star\), then with \(A_{T-2}^\star\), and so on, until all steps are merged into a single operator.
\end{theorem}

\begin{proof}[Proof of~\Cref{thm:sequential_boot_optimality}]
We prove by induction on the block size \(\ell = t_2 - t_1\) where we require \(1\leq t_1<t_2<T\). When \(\ell = 1\), the merge plan is trivially optimal, as there is only one valid way to merge the operators. We now consider the case \(\ell = 2\), and denote \(t \coloneqq t_2\). We analyze the behavior of three consecutive single-step operators:
\begin{equation}
A_{t-2}^\star, A_{t-1}^\star, A_t^\star,
\end{equation}
satisfying the assumption
\begin{equation}
0 < A_{t-2}^\star, A_{t-1}^\star, A_t^\star < 1.
\end{equation}
Our goal is to approximate the surrogate target given by
\begin{equation}
A_{t-2}^\star A_{t-1}^\star A_t^\star.
\end{equation}
The merge process applies a final shrinkage when the interval ends at time \(t\), parameterized by a shrinkage factor \(\gamma_t \in (0,1)\). Similarly, we denote by \(\gamma_{t-1}\) the shrinkage applied to merges ending at time \(t-1\). We compare three candidate merging strategies. The first is vanilla distillation, which performs a one-shot merge of all three operators. Its merged outcome is
\begin{equation}
A_{\text{vanilla}} \coloneqq (1-\gamma_t)A_{t-2}^\star A_{t-1}^\star A_t^\star + \gamma_t A_t^\star.
\end{equation}
In the second strategy, sequential BOOT distillation, we split after \(A_{t-2}^\star\), treating it as a left block, and merge \(\{A_{t-1}^\star, A_t^\star\}\) as the right block. The overall merged operator is
\begin{equation}
A_{\text{BOOT}} \coloneqq \bigl((1-\gamma_t)A_{t-1}^\star A_t^\star + \gamma_t A_t^\star\bigr)\bigl((1-\gamma_t)A_{t-2}^\star + \gamma_t\bigr).
\end{equation}
The third strategy corresponds to sequential consistency distillation, where we first merge the left block \(\{A_{t-2}^\star, A_{t-1}^\star\}\), and then merge it with \(A_t^\star\). The resulting merged operator is
\begin{equation}
A_{\text{consistency}} \coloneqq A_t^\star\bigl((1-\gamma_t)\bigl((1-\gamma_{t-1})A_{t-2}^\star A_{t-1}^\star + \gamma_{t-1} A_{t-1}^\star\bigr) + \gamma_t\bigr).
\end{equation}
We now compare the approximation errors associated with these three strategies. Since \(0 < A_{t-2}^\star, A_{t-1}^\star, A_t^\star < 1\), all merge outcomes will strictly exceed the target. Therefore, the strategy with the smallest value is closest to the target, and hence yields the lowest error. First, consider the difference between the BOOT outcome and the vanilla outcome:
\begin{equation}
\begin{aligned}
&A_{\text{BOOT}} - A_{\text{vanilla}}\\
=\;& \bigl((1-\gamma_t)A_{t-1}^\star A_t^\star + \gamma_t A_t^\star\bigr)\bigl((1-\gamma_t)A_{t-2}^\star + \gamma_t\bigr) \\
&- (1-\gamma_t)(A_{t-2}^\star A_{t-1}^\star A_t^\star) - \gamma_t A_t^\star\\
=\;& -\gamma_t(1-\gamma_t)A_{t-2}^\star A_{t-1}^\star A_t^\star+\gamma_t(1-\gamma_t)(A_t^\star A_{t-2}^\star+ A_t^\star A_{t-1}^\star)\\
&+(\gamma_t^2-\gamma_t)A_t^\star\\
=\;&-\gamma_t(1-\gamma_t)A_t^\star(A_{t-2}^\star-1)(A_{t-1}^\star-1)\\
\leq\;&0.
\end{aligned}
\end{equation}
Next, we compare the BOOT outcome with the consistency outcome:
\begin{equation}
\begin{aligned}
&A_{\text{BOOT}} - A_{\text{consistency}}\\
=\;& \bigl((1-\gamma_t)A_{t-1}^\star A_t^\star + \gamma_t A_t^\star\bigr)\bigl((1-\gamma_t)A_{t-2}^\star + \gamma_t\bigr) \\
&- A_t^\star\bigl((1-\gamma_t)\bigl((1-\gamma_{t-1})A_{t-2}^\star A_{t-1}^\star + \gamma_{t-1} A_{t-1}^\star\bigr) + \gamma_t\bigr)\\
=\;& (1-\gamma_t)(\gamma_{t-1}-\gamma_t)A_t^\star A_{t-1}^\star(A_{t-2}^\star-1)+\gamma_t(1-\gamma_t)A_t^\star(A_{t-2}^\star-1)\\
=\;& (1-\gamma_t)A_t^\star(A_{t-2}^\star-1)(A_{t-1}^\star(\gamma_{t-1}-\gamma_t)+\gamma_t)\\
\leq\;&0.
\end{aligned}
\end{equation}
Both differences are non-positive. Hence, among the three possible merge outcomes, the sequential BOOT strategy achieves the smallest value and thus best approximates the surrogate target. This completes the base case.

Assume that the sequential BOOT plan is optimal for all block sizes of length \(\ell\). We now prove that it remains optimal for block size \(\ell + 1\). Let \(t \coloneqq t_2\) be the end index of the interval. Consider an alternative merge plan that splits the interval at position \(t - \ell + k\), where \(0 < k < \ell\). We aim to show that the final merged outcome of this split plan is strictly larger than that of the sequential BOOT strategy, and hence suboptimal. The merged outcome of the split plan can be written as
\begin{equation}
\begin{aligned}
A_{\text{split}} \coloneqq\;& \Bigl(A_t^\star\prod_{i=1}^{\ell-k-1}\bigl((1-\gamma_t)A_{t-i}^\star+\gamma_t\bigr)\Bigr)\\
&\cdot\Bigl((1-\gamma_t)A_{t-\ell+k}^\star\prod_{j=\ell-k+1}^{\ell}\bigl((1-\gamma_{t-\ell+k})A_{t-j}^\star+\gamma_{t-\ell+k}\bigr)+\gamma_t\Bigr)\\
\geq\;& \Bigl(A_t^\star\prod_{i=1}^{\ell-k-1}\bigl((1-\gamma_t)A_{t-i}^\star+\gamma_t\bigr)\Bigr)\\
&\cdot\bigl((1-\gamma_t)A_{t-\ell+k}^\star+\gamma_t\bigr)\\
&\cdot\prod_{j=\ell-k+1}^{\ell}\bigl((1-\gamma_{t-\ell+k})A_{t-j}^\star+\gamma_{t-\ell+k}\bigr)\\
\geq\;& \Bigl(A_t^\star\prod_{i=1}^{\ell-k-1}\bigl((1-\gamma_t)A_{t-i}^\star+\gamma_t\bigr)\Bigr)\\
&\cdot\bigl((1-\gamma_t)A_{t-\ell+k}^\star+\gamma_t\bigr)\cdot\prod_{j=\ell-k+1}^{\ell}\bigl((1-\gamma_t)A_{t-j}^\star+\gamma_t\bigr)\\
=\;& A_t^\star\prod_{i=1}^{\ell}\bigl((1-\gamma_t)A_{t-i}^\star+\gamma_t\bigr).
\end{aligned}
\end{equation}
The final expression on the right-hand side corresponds precisely to the merged outcome under the sequential BOOT strategy. Since the split plan yields a result that is at least as large, and since all operators lie in the interval \((0, 1)\), this implies that the split outcome is strictly further from the surrogate target. Therefore, the sequential BOOT merge remains optimal for block size \(\ell + 1\), completing the induction step.
\end{proof}

%%%%%%%%%%%%%%%%%%%%%%%%%%%%%%%%%%%%%%%%%%%%%%%%%%%%%%%%%%%%
\subsection{Pareto dynamic programming for optimal merging}
\label{subsec:pareto_dynamic_programming_for_optimal_merging}

In high-dimensional settings, if all data dimensions fall into the same variance regime (i.e., all \(\lambda_i \le 1\) or all \(\lambda_i \gg 1\)), the coordinate-wise optimal strategies coincide, reducing the problem to the scalar cases analyzed in~\Cref{subsec:phase_transitions_in_the_optimal_strategy}. The core difficulty arises in the mixed regime, where the data distribution contains both low-variance dimensions (favoring gradual merging) and high-variance dimensions (favoring immediate merging). Since a single merge plan must be applied globally to the entire state vector, these conflicting objectives cannot be maximized simultaneously by any single scalar strategy.

To resolve this trade-off, we propose a set-valued dynamic programming algorithm (see~\Cref{alg:pareto_dp_merge}). Instead of storing a single ``best'' operator for each sub-interval, we maintain a \emph{Pareto frontier} of non-dominated candidate operators.

\begin{algorithm}[t]
\caption{Pareto dynamic programming for optimal operator merging}
\label{alg:pareto_dp_merge}
\small
\begin{algorithmic}[1]
\Require Single-step operators \(\{\bm{A}_t^\star\}_{t=1}^T\), shrinkage matrices \(\{\bm{\Gamma}_t\}_{t=1}^T\), surrogate target \(\widetilde{\mathcal{T}}_T\), variances \(\{\lambda_i\}_{i=1}^d\)

\State Initialize \(\mathcal{S}[t,t] \gets \{\bm{A}_t^\star\}\) for all \(t\in\{1,\dotsc,T\}\)

\For{\(\ell = 2\) to \(T\)} \Comment{Iterate over interval lengths}
\For{\(t_1 = 1\) to \(T-\ell+1\)} \Comment{Iterate over start times}
\State \(t_2 \gets t_1+\ell-1\)
\State \(\mathcal{S}[t_1,t_2] \gets \emptyset\)

\State \textbf{Candidate 1 (Direct):} \(\bm{A} \gets (\bm{I}-\bm{\Gamma}_{t_2}) \cdot \prod_{t=t_1}^{t_2} \bm{A}_t^\star + \bm{\Gamma}_{t_2}\cdot \bm{A}_{t_2}^\star\)
\State \texttt{InsertAndPrune}(\(\mathcal{S}[t_1,t_2], \bm{A}, \{\lambda_i\}\))

\State \textbf{Candidate 2 (Split):}
\For{\(m=t_1\) to \(t_2-1\)}
\ForAll{\(\bm{A}_L \in \mathcal{S}[t_1,m]\)}
\ForAll{\(\bm{A}_R \in \mathcal{S}[m+1,t_2]\)}
\State \(\bm{A} \gets (\bm{I}-\bm{\Gamma}_{t_2})\cdot \bm{A}_L\cdot \bm{A}_R + \bm{\Gamma}_{t_2}\cdot \bm{A}_R\)
\State \texttt{InsertAndPrune}(\(\mathcal{S}[t_1,t_2], \bm{A}, \{\lambda_i\}\))
\EndFor
\EndFor
\EndFor
\EndFor
\EndFor
\State \Return \(\bm A^\star=\arg\min_{\bm{A}\in\mathcal{S}[1,T]} \sum_{i=1}^d \bigl((\widetilde{\mathcal{T}}_T)_i-(\bm{A})_i\bigr)^2\)
\end{algorithmic}
\end{algorithm}

\paragraph{DP state}
For each interval \([t_1,t_2]\), let \(\mathcal{S}[t_1,t_2]\) be the set of valid merged diagonal operators covering that interval. Each element \(\bm{A} \in \mathcal{S}[t_1,t_2]\) is a candidate diagonal matrix \(\bm{A}_{t_1{:}t_2}\) representing a specific merging strategy.

\paragraph{Pareto pruning rule}
To keep the set size manageable, we enforce a Pareto efficiency constraint. We define a preference direction \(\rho_i\) for each coordinate based on its variance regime:
\begin{equation}
\rho_i \coloneqq
\begin{cases}
+1, & \text{if } \lambda_i > 1 \quad (\text{prefer larger values}),\\
-1, & \text{if } \lambda_i \le 1 \quad (\text{prefer smaller values}).
\end{cases}
\label{eq:preference}
\end{equation}
We say a candidate \(\bm{B}\) \emph{dominates} \(\bm{C}\) if \(\rho_i (\bm{B})_i \ge \rho_i (\bm{C})_i\) for all \(i=1,\dots,d\), with strict inequality for at least one coordinate. The procedure \texttt{InsertAndPrune}(\(\mathcal{S}, \bm{A}_{\text{new}}, \{\lambda_i\}\)) in~\Cref{alg:pareto_dp_merge} updates the set \(\mathcal{S}\) with a new candidate \(\bm{A}_{\text{new}}\) as follows:
\begin{enumerate}
    \item If \(\bm{A}_{\text{new}}\) is dominated by any existing \(\bm{B} \in \mathcal{S}\), discard \(\bm{A}_{\text{new}}\).
    \item If \(\bm{A}_{\text{new}}\) dominates any existing \(\bm{B} \in \mathcal{S}\), remove those \(\bm{B}\) from \(\mathcal{S}\).
    \item If \(\bm{A}_{\text{new}}\) is neither dominated nor dominates, add it to \(\mathcal{S}\).
\end{enumerate}
This ensures that \(\mathcal{S}\) always represents a valid Pareto frontier.

\paragraph{Recurrence}
We compute \(\mathcal{S}[t_1,t_2]\) recursively by exploring two types of merge operations for each interval. Let \(\bm{\Gamma}_{t_2}\coloneqq\mathrm{Diag}(\gamma_t^1(s),\dotsc,\gamma_t^d(s))\) be the shrinkage matrix at the interval's end.
\begin{enumerate}
\item \textbf{Direct Merge:} We consider merging the entire sequence \(t_1, \dots, t_2\) in one shot:
\begin{equation}
\bm{A}^{\mathrm{direct}} = (\bm{I}-\bm{\Gamma}_{t_2}) \cdot \prod_{t=t_1}^{t_2} \bm{A}_t^\star + \bm{\Gamma}_{t_2}\cdot \bm{A}_{t_2}^\star.
\end{equation}
\item \textbf{Split Merge:} We combine optimal sub-plans from a split point \(m\). For every split \(m \in \{t_1, \dots, t_2-1\}\), and every pair of candidates \(\bm{A}_L \in \mathcal{S}[t_1,m]\) and \(\bm{A}_R \in \mathcal{S}[m+1,t_2]\), we form:
\begin{equation}
\bm{A}^{\mathrm{split}} = (\bm{I}-\bm{\Gamma}_{t_2})\cdot \bm{A}_L\cdot \bm{A}_R + \bm{\Gamma}_{t_2}\cdot \bm{A}_R.
\end{equation}
\end{enumerate}

\paragraph{Selecting the optimal plan}
After populating the table up to the full interval \(\mathcal{S}[1,T]\), we select the final operator by minimizing the objective \(\mathcal{L}_{W_2}\) defined in~\Cref{subsec:w2_metric_and_surrogate_target}:
\begin{equation}
\bm{A}^* = \operatorname*{argmin}_{\bm{A} \in \mathcal{S}[1,T]} \sum_{i=1}^d \bigl((\widetilde{\mathcal{T}}_T)_i-(\bm{A})_i\bigr)^2.
\end{equation}

\rev{We remark that, in~\Cref{alg:pareto_dp_merge}, the diagonal Gaussian assumption is mainly a choice of coordinates. For a general centered Gaussian \( \mathcal N(\bm 0,\bm\Sigma) \), an orthogonal change of basis diagonalizes \( \bm\Sigma \), and under the same transformation the DDIM operators, shrinkage matrices, recursive merge rule, and surrogate objective are all preserved. Hence, the algorithm and its optimality result extend directly to non-diagonal centered Gaussian distributions. We defer the formal derivation to~\ref{app:additional_experimental_results_on_pareto_dynamic_programming}.}

We can show that~\Cref{alg:pareto_dp_merge} finds the optimal merging strategy, as formalized in~\Cref{thm:pareto_dp_merge_optimal}.

\begin{theorem}[Optimality of Pareto dynamic programming]
\label{thm:pareto_dp_merge_optimal}
Under the recursive operator merging process defined in~\Cref{sec:variance_driven_optimal_merge_plans_and_pareto_dp_in_the_gaussian_regime}, \Cref{alg:pareto_dp_merge} returns a merged operator that globally minimizes the objective \(\mathcal{L}_{W_2}\) defined in~\Cref{subsec:w2_metric_and_surrogate_target} among all possible merging strategies.
\end{theorem}

\begin{proof}[Proof of~\Cref{thm:pareto_dp_merge_optimal}]
Let \(\mathcal{A}_{1:T}\) be the set of all valid recursively merged operators on the interval \([1, T]\). The objective function \(\mathcal{L}_{W_2} = \sum_{i=1}^d \bigl((\widetilde{\mathcal{T}}_T)_i - (\bm{A})_i\bigr)^2\) is separable coordinate-wise. To prove that \Cref{alg:pareto_dp_merge} returns the global minimizer, it suffices to show that the final set \(\mathcal{S}[1,T]\) contains at least one operator \(\bm{A}^*\) lying on the Pareto frontier defined by the preference directions \(\bm{\rho}\), and that minimizing the coordinate-wise error is equivalent to optimizing along these directions.

First, we establish the relationship between the Wasserstein error and the preference directions. Recall from~\eqref{eq:preference} that \(\rho_i = +1\) if \(\lambda_i > 1\) and \(\rho_i = -1\) if \(\lambda_i \le 1\). In the high-variance regime where \(\lambda_i > 1\), the surrogate target \((\widetilde{\mathcal{T}}_T)_i\) is a product of terms strictly greater than or equal to \(1\). Since any student operator \((\bm{A})_i\) is formed by convex combinations involving shrinkage factors \(\gamma \in (0,1)\), it effectively interpolates between the target and smaller values, implying \((\bm{A})_i \le (\widetilde{\mathcal{T}}_T)_i\). Consequently, minimizing the squared distance \(((\widetilde{\mathcal{T}}_T)_i - (\bm{A})_i)^2\) is equivalent to maximizing \((\bm{A})_i\), which aligns with \(\rho_i = +1\). Conversely, in the low-variance regime where \(\lambda_i \le 1\), the student operator introduces resets to the single-step teacher values which prevents full contraction, implying \((\bm{A})_i \ge (\widetilde{\mathcal{T}}_T)_i\). Here, minimizing the error requires minimizing \((\bm{A})_i\), aligning with \(\rho_i = -1\). Thus, finding the global minimizer of \(\mathcal{L}_{W_2}\) is equivalent to finding a Pareto-optimal element with respect to \(\bm{\rho}\).

Next, we prove that the Pareto frontier can be constructed via dynamic programming by showing the monotonicity of the merge operation. Consider the split merge operation defined in~\Cref{subsec:pareto_dynamic_programming_for_optimal_merging}. For a specific coordinate \(i\), let \(x = (\bm{A}_L)_i\) and \(y = (\bm{A}_R)_i\) be the values of the left and right sub-blocks, and let \(\gamma = (\bm{\Gamma}_{t_2})_{ii}\) be the shrinkage factor. The merged value is given by the function \(f(x, y) = (1-\gamma)xy + \gamma y\). The partial derivatives are \(\frac{\partial f}{\partial x} = (1-\gamma)y\) and \(\frac{\partial f}{\partial y} = (1-\gamma)x + \gamma\). Since the operators \(x, y\) are positive (being products of positive variance scalings) and \(\gamma \in (0,1)\), these derivatives are strictly positive. This implies that the merge function is strictly increasing in both arguments. Therefore, if a sub-plan \(\bm{B} \in \mathcal{S}[t_1, m]\) dominates another sub-plan \(\bm{C} \in \mathcal{S}[t_1, m]\) according to \(\bm{\rho}\), merging \(\bm{B}\) with any fixed block \(\bm{D}\) yields a result that dominates the merge of \(\bm{C}\) with \(\bm{D}\).

This monotonicity guarantees the principle of optimality: a dominated sub-plan cannot be part of a non-dominated global plan. By iteratively constructing candidates and pruning only those that are strictly dominated, the algorithm ensures that \(\mathcal{S}[1,T]\) retains the Pareto frontier of all possible merge plans. Since the global minimizer of \(\mathcal{L}_{W_2}\) corresponds to a point on this frontier, the algorithm correctly identifies it.
\end{proof}

%%%%%%%%%%%%%%%%%%%%%%%%%%%%%%%%%%%%%%%%%%%%%%%%%%%%%%%%%%%%%%%%%%%%%%%%%%%%%%%%
\section{Approximation error and error propagation in the Gaussian mixture regime}
\label{sec:approximation_error_and_error_propagation_in_the_gaussian_mixture_regime}

To capture the multi-modality of real-world distributions, we extend our analysis to Gaussian mixture models in this section. While the optimal denoising estimator for Gaussian mixtures still has a closed-form expression (\Cref{subsec:optimal_denoising_for_gmm}), the resulting teacher operators are nonlinear. This nonlinearity precludes an exact analysis of the training dynamics. Consequently, we quantify the difficulty of distillation by first characterizing the approximation error of distilling multi-step composite teachers with a finite-capacity \(K\)-component affine mixture-of-experts (MoE) student in~\Cref{subsec:gmm_approximation_error}. We then analyze how these local errors propagate and accumulate across successive merges in~\Cref{subsec:error_propagation_across_merges}.

%%%%%%%%%%%%%%%%%%%%%%%%%%%%%%%%%%%%%%%%%%%%%%%%%%%%%%%%%%%%
\subsection{Optimal denoising estimator for Gaussian mixtures}
\label{subsec:optimal_denoising_for_gmm}

In this section, we assume that the data distribution \(p_0\) is a mixture of \(K\) Gaussian distributions with distinct means \(\{\bm{\mu}_k\}_{k=1}^K\) and covariance matrices \(\{\bm{\Lambda}_k\}_{k=1}^K\). Under this assumption, the optimal denoising estimator becomes a convex combination of component-wise affine estimators, as given by~\Cref{prop:optimal_mixture_denoising_estimator}.

\begin{proposition}[Optimal denoising estimator for Gaussian mixtures]
\label{prop:optimal_mixture_denoising_estimator}
Assume the real data distribution \(p_0\) is a Gaussian mixture distribution
\begin{equation}
p_0(\bm{x}_0)=\sum_{k=1}^K \pi_k \mathcal{N}\bigl(\bm{x}_0; \bm{\mu}_k, \bm{\Lambda}_k\bigr), \quad \pi_k > 0, \quad \sum_{k=1}^K \pi_k = 1,
\label{eq:gmm_assumption}
\end{equation}
where each \(\bm{\Lambda}_k\succeq\bm0\). Assume the forward process follows~\eqref{eq:forward_process} and the denoising estimator minimizes the denoising loss~\eqref{eq:denoising_loss_function}. Then the optimal denoising estimator for \(\bm{z}_t\) is 
\begin{equation}
\begin{aligned}
\hat{\bm{x}}_0^\star(\bm{z}_t, t)
&=\mathbb{E}[\bm{x}_0 | \bm{z}_t]\\
&=\sum_{k=1}^K
\gamma_{k,t}(\bm{z}_t)
\bigl(
\bm{\mu}_k
+
\alpha_t \bm{\Lambda}_k
\bigl(\alpha_t^2 \bm{\Lambda}_k + \sigma_t^2 \bm{I}\bigr)^{-1}
(\bm{z}_t - \alpha_t \bm{\mu}_k)
\bigr),
\label{eq:gmm_denoiser}
\end{aligned}
\end{equation}
where the weighting functions \(\gamma_{k,t}(\bm{z}_t)\) are given by
\begin{equation}
\gamma_{k,t}(\bm{z}_t)
=
\frac{
\pi_k \mathcal{N}\bigl(\bm{z}_t; \alpha_t \bm{\mu}_k, \alpha_t^2 \bm{\Lambda}_k + \sigma_t^2 \bm{I}\bigr)
}{
\sum_{\ell=1}^K
\pi_\ell \mathcal{N}\bigl(\bm{z}_t; \alpha_t \bm{\mu}_\ell, \alpha_t^2 \bm{\Lambda}_\ell + \sigma_t^2 \bm{I}\bigr)
}.
\label{eq:weighting_functions}
\end{equation}
\end{proposition}

\begin{proof}[Proof of~\Cref{prop:optimal_mixture_denoising_estimator}]
Let \(k \in \{1,\dots,K\}\) denote the latent mixture component. By the law of total expectation,
\begin{equation}
\mathbb{E}[\bm{x}_0 | \bm{z}_t] = \sum_{k=1}^K p(k | \bm{z}_t) \mathbb{E}[\bm{x}_0 | \bm{z}_t, k].
\end{equation}
Conditioned on \(k\), we have \(\bm{x}_0 | k \sim \mathcal{N}(\bm{\mu}_k, \bm{\Lambda}_k)\) and \(\bm{z}_t | \bm{x}_0 \sim \mathcal{N}(\alpha_t \bm{x}_0, \sigma_t^2 \bm{I})\). This implies the joint Gaussian
\begin{equation}
\begin{bmatrix}
\bm{x}_0 \\
\bm{z}_t
\end{bmatrix}
\bigg| k \sim \mathcal{N}\Biggl(
\begin{bmatrix}
\bm{\mu}_k \\
\alpha_t \bm{\mu}_k
\end{bmatrix},
\begin{bmatrix}
\bm{\Lambda}_k & \alpha_t \bm{\Lambda}_k \\
\alpha_t \bm{\Lambda}_k & \alpha_t^2 \bm{\Lambda}_k + \sigma_t^2 \bm{I}
\end{bmatrix}
\Biggr).
\end{equation}
Applying Gaussian conditioning yields
\begin{equation}
\mathbb{E}[\bm{x}_0 | \bm{z}_t, k]
=
\bm{\mu}_k
+
\alpha_t \bm{\Lambda}_k
\bigl(\alpha_t^2 \bm{\Lambda}_k + \sigma_t^2 \bm{I}\bigr)^{-1}
(\bm{z}_t - \alpha_t \bm{\mu}_k).
\end{equation}
Next, Bayes' rule gives
\begin{equation}
p(k | \bm{z}_t)
=
\frac{p(\bm{z}_t | k) \pi_k}{\sum_{\ell=1}^K p(\bm{z}_t | \ell) \pi_\ell},
\end{equation}
and marginalizing the joint distribution above yields \(\bm{z}_t | k \sim \mathcal{N}(\alpha_t \bm{\mu}_k, \alpha_t^2 \bm{\Lambda}_k + \sigma_t^2 \bm{I})\). Substituting these expressions and setting \(\gamma_{k,t}(\bm{z}_t) \coloneqq p(k | \bm{z}_t)\) gives~\eqref{eq:gmm_denoiser}.
\end{proof}

Substituting~\eqref{eq:gmm_denoiser} into the DDIM update rule~\eqref{eq:denosing_process} shows that the single-step operator is a mixture of affine transformations. The proof is straightforward, so we omit it.

\begin{corollary}[Single-step operator for Gaussian mixtures]
\label{cor:gmm_single_step}
Under the assumptions of~\Cref{prop:optimal_mixture_denoising_estimator}, the DDIM sampling process~\eqref{eq:denosing_process} can be written as
\begin{equation}
\bm{z}_{t-1}
=
\sum_{k=1}^K \gamma_{k,t}(\bm{z}_t)\bigl(\bm{A}_{k,t}\bm{z}_t + \bm{b}_{k,t}\bigr),
\label{eq:gmm_single_operator_moe}
\end{equation}
where, for each component \(1\leq k\leq K\),
\begin{equation}
\begin{aligned}
\bm{A}_{k,t}
&\coloneqq
(\alpha_{t-1}\alpha_t\bm{\Lambda}_k+\sigma_{t-1}\sigma_t\bm{I})
(\alpha_t^2\bm{\Lambda}_k+\sigma_t^2\bm{I})^{-1},\\
\bm{b}_{k,t}
&\coloneqq
(\alpha_{t-1}\bm{I}-\alpha_t \bm{A}_{k,t})\bm{\mu}_k.
\label{eq:gmm_affine_params}
\end{aligned}
\end{equation}
\end{corollary}

%%%%%%%%%%%%%%%%%%%%%%%%%%%%%%%%%%%%%%%%%%%%%%%%%%%%%%%%%%%%
\subsection{Approximation error under affine MoE distillation}
\label{subsec:gmm_approximation_error}

Motivated by the structure of the single-step operator derived in~\Cref{cor:gmm_single_step}, we parameterize both the teacher and the student models as \(K\)-component affine mixture-of-experts (MoE). Specifically, we assume that the single-step teacher operator at time \(t\) takes the form
\begin{equation}
\bm{f}^{\star}(\bm{z}_t, t)
\coloneqq
\sum_{k=1}^K \gamma_{k,t}^\star(\bm{z}_t)\bigl(\bm{A}_{k,t}^\star \bm{z}_t + \bm{b}_{k,t}^\star\bigr),
\label{eq:gmm_teacher_moe}
\end{equation}
where \(\gamma_{k,t}^\star(\cdot)\coloneqq\gamma_{k,t}(\cdot)\), \(\bm{A}_{k,t}^\star\coloneqq\bm A_{k,t}\), and \(\bm{b}_{k,t}^\star\coloneqq\bm b_{k,t}\) are the parameters given by~\eqref{eq:weighting_functions} and~\eqref{eq:gmm_affine_params}. Since the student model is usually initialized from the teacher model~\citep{luhman2021knowledge,salimans2022progressive}, we parameterize it as
\begin{equation}
\bm{f}^{\mathrm{st}}(\bm{z}_t, t)
\coloneqq
\sum_{k=1}^K \gamma_{k,t}^{\mathrm{st}}(\bm{z}_t)\bigl(\bm{A}_{k,t}^{\mathrm{st}} \bm{z}_t + \bm{b}_{k,t}^{\mathrm{st}}\bigr),
\label{eq:gmm_student_moe}
\end{equation}
where \(\gamma_{k,t}^{\mathrm{st}}(\bm{z}) \ge 0\) and \(\sum_{k=1}^K \gamma_{k,t}^{\mathrm{st}}(\bm{z}) = 1\) for all \(\bm{z}\). The learnable parameters are \(\{\gamma_{k,t}^{\mathrm{st}}\}\), \(\{\bm{A}_{k,t}^{\mathrm{st}}\}\), and \(\{\bm{b}_{k,t}^{\mathrm{st}}\}\).

We focus on the trajectory distillation objective in~\Cref{as:trajectory_distillation_loss_for_gaussian_mixtures}, where the student operator attempts to match the multi-step composite teacher operator.

\begin{assumption}[Trajectory distillation loss for Gaussian mixtures]
\label{as:trajectory_distillation_loss_for_gaussian_mixtures}
Let \(\bm{z}_t \sim p_t = (\alpha_t p_0) * \mathcal{N}(\bm{0}, \sigma_t^2 \bm{I})\) be a noisy input at time \(t\). We assume that the distillation loss at time \(t\) is given by
\begin{equation}
\mathcal{L}_t\bigl(\{\gamma_{k,t}^{\mathrm{st}}\}, \{\bm{A}_{k,t}^{\mathrm{st}}\}, \{\bm{b}_{k,t}^{\mathrm{st}}\}\bigr) = \mathbb{E}_{\bm{z}_t \sim p_t}\bigl[\lVert\bm{f}^{\mathrm{st}}(\bm{z}_t,t) - \mathcal{T}_k(\bm{z}_t)\rVert_2^2\bigr].
\label{eq:distillation_loss_gmm}
\end{equation}
Here, \(\mathcal{T}_k(\bm{z}_t)\) denotes the \(k\)-step composite teacher operator starting from time \(t\), as in~\Cref{def:composite_teacher_operator}.
\end{assumption}

Directly minimizing this loss is challenging because the composite teacher operator \(\mathcal{T}_k\) is functionally more complex than the single-step student. For clarity, we focus on the case \(k=2\) in the rest of this subsection, but the result easily generalizes. We show in~\Cref{lem:two_step_moe_expansion} that composing two \(K\)-component affine MoE operators induces \(K^2\) expert components, indexed by ordered pairs \((i,j)\), where \(i\) selects the expert at time \(t\) and \(j\) selects the expert at time \(t-1\) evaluated at the intermediate output.

\begin{lemma}[Two-step composite operator has \(K^2\) components]
\label{lem:two_step_moe_expansion}
Assume that the teacher single-step operators at times \(t\) and \(t-1\) are \(K\)-component affine MoE as defined in~\eqref{eq:gmm_teacher_moe}. Recall the two-step composite operator \(\mathcal{T}_2(\bm{z}_t)=\bm{f}^{\star}(\bm{f}^{\star}(\bm{z}_t,t),t-1)\). Then \(\mathcal{T}_2\) admits the following expansion
\begin{equation}
\mathcal{T}_2(\bm{z}_t)=\sum_{i=1}^{K}\sum_{j=1}^{K} w_{ij,t}^{(2)}(\bm{z}_t)\bigl(\bm{A}^{(2)}_{ij,t}\bm{z}_t+\bm{b}^{(2)}_{ij,t}\bigr),
\end{equation}
where the composite weights and affine parameters are given by
\begin{equation}
w_{ij,t}^{(2)}(\bm{z}_t)
=
\gamma_{i,t}^\star(\bm{z}_t)
\gamma_{j,t-1}^\star\bigl(\bm{f}^{\star}(\bm{z}_t, t)\bigr),
\end{equation}
\begin{equation}
\bm{A}^{(2)}_{ij,t}=\bm{A}_{j,t-1}^\star\bm{A}_{i,t}^\star,
\quad
\bm{b}^{(2)}_{ij,t}=\bm{A}_{j,t-1}^\star\bm{b}_{i,t}^\star+\bm{b}_{j,t-1}^\star.
\end{equation}
Moreover \(\sum_{i=1}^{K}\sum_{j=1}^{K} w_{ij,t}^{(2)}(\bm{z}_t)=1\) for all \(\bm{z}_t\).
\end{lemma}

\begin{proof}[Proof of~\Cref{lem:two_step_moe_expansion}]
We use the shorthand \(\bm{y}\coloneqq\bm{f}^{\star}(\bm{z}_t,t)\). By definition,
\begin{equation}
\mathcal{T}_2(\bm{z}_t)=\bm{f}^{\star}(\bm{y},t-1)
=\sum_{j=1}^{K}\gamma_{j,t-1}^\star(\bm{y})\bigl(\bm{A}_{j,t-1}^\star\bm{y}+\bm{b}_{j,t-1}^\star\bigr).
\end{equation}
Substitute the MoE form of \(\bm{y}\) into the affine term. Note that \(\gamma_{j,t-1}^\star(\bm{y})\) acts as a common scalar for the inner sum. We distribute the operation over the sum on \(i\) and obtain
\begin{equation}
\bm{A}_{j,t-1}^\star\bm{y} + \bm{b}_{j,t-1}^\star
=
\bm{A}_{j,t-1}^\star \sum_{i=1}^{K}\gamma_{i,t}^\star(\bm{z}_t)\bigl(\bm{A}_{i,t}^\star\bm{z}_t+\bm{b}_{i,t}^\star\bigr) + \sum_{i=1}^K \gamma_{i,t}^\star(\bm{z}_t) \bm{b}_{j,t-1}^\star,
\end{equation}
where we used \(\sum_{i=1}^K \gamma_{i,t}^\star(\bm{z}_t) = 1\) to absorb the bias term \(\bm{b}_{j,t-1}^\star\). Grouping terms by index \(i\), we get
\begin{equation}
\bm{A}_{j,t-1}^\star\bm{y} + \bm{b}_{j,t-1}^\star
=
\sum_{i=1}^{K}\gamma_{i,t}^\star(\bm{z}_t)\bigl(\bm{A}_{j,t-1}^\star\bm{A}_{i,t}^\star\bm{z}_t+\bm{A}_{j,t-1}^\star\bm{b}_{i,t}^\star+\bm{b}_{j,t-1}^\star\bigr).
\end{equation}
Substituting this back into the expression for \(\mathcal{T}_2(\bm{z}_t)\) yields the double summation with weights \(w_{ij,t}^{(2)}(\bm{z}_t) = \gamma_{j,t-1}^\star(\bm{y}) \gamma_{i,t}^\star(\bm{z}_t)\).
\end{proof}

The two-step expansion in~\Cref{lem:two_step_moe_expansion} extends to \(k\) steps. In particular, by induction on the number of composition steps, the \(k\)-step composite teacher operator \(\mathcal{T}_k\) admits a expansion with \(K^k\) components, indexed by ordered \(k\)-tuples of experts. Consequently, when the student is restricted to \(K\) components, approximating \(\mathcal{T}_k\) requires compressing \(K^k\) teacher components into \(K\) student components. Combined with~\Cref{thm:two_step_merging_error}, this exponential growth in the number of teacher components suggests that the approximation error generally increases with the composite horizon \(k\), making distillation harder.

We can now state our main result for this section. Since the student model only has \(K\) components (see~\eqref{eq:gmm_student_moe}), it cannot perfectly represent the \(K^2\) components of the composite teacher operator derived in~\Cref{lem:two_step_moe_expansion}. The following~\Cref{thm:two_step_merging_error} shows that the approximation error measured in distillation loss~\eqref{eq:distillation_loss_gmm} is upper-bounded by the residual error of clustering these \(K^2\) teacher components into \(K\) student components.

\begin{theorem}[Approximation error for compressing a two-step \(K^2\)-component teacher]
\label{thm:two_step_merging_error}
Fix a time \(t\) and let \(\bm{z}_t\sim p_t\). Assume that the two-step composite operator \(\mathcal{T}_2\) admits the expansion from \(\Cref{lem:two_step_moe_expansion}\):
\begin{equation}
\mathcal{T}_2(\bm{z}_t)=\sum_{i=1}^{K}\sum_{j=1}^{K} w_{ij,t}^{(2)}(\bm{z}_t)\bm{g}_{ij,t}(\bm{z}_t),
\quad
\bm{g}_{ij,t}(\bm{z}_t)=\bm{A}^{(2)}_{ij,t}\bm{z}_t+\bm{b}^{(2)}_{ij,t}.
\end{equation}
Define the best-achievable \(K\)-component student approximation error
\begin{equation}
\varepsilon_{t-1{:}t}
\coloneqq
\inf_{\gamma_{k,t}^{\mathrm{st}}, \bm A_{k,t}^{\mathrm{st}}, \bm b_{k,t}^{\mathrm{st}}}
\mathbb{E}_{\bm{z}_t\sim p_t}\Bigl[\bigl\lVert\bm{f}^{\mathrm{st}}(\bm{z}_t,t)-\mathcal{T}_2(\bm{z}_t)\bigr\rVert_2^2\Bigr].
\end{equation}
Then, for any partition \(\Pi=\{S_1,\dotsc,S_K\}\) of the \(K^2\) index pairs \((i,j)\),
\begin{equation}
\varepsilon_{t-1{:}t}
\leq
\sum_{k=1}^{K}\inf_{\bm{A},\bm{b}}
\mathbb{E}_{\bm{z}_t\sim p_t}\Bigl[
\sum_{(i,j)\in S_k} w_{ij,t}^{(2)}(\bm{z}_t)\bigl\lVert\bm{A}\bm{z}_t+\bm{b}-\bm{g}_{ij,t}(\bm{z}_t)\bigr\rVert_2^2
\Bigr].
\end{equation}
Furthermore, the inner objective for each cluster \(S_k\) decomposes into a bias (distance to the cluster centroid) and a variance term as
\begin{multline}
\sum_{(i,j)\in S_k} w_{ij,t}^{(2)}(\bm{z}_t)\bigl\lVert\bm{A}\bm{z}_t+\bm{b}-\bm{g}_{ij,t}(\bm{z}_t)\bigr\rVert_2^2
=\\
W_{k,t}(\bm{z}_t)\bigl\lVert\bm{A}\bm{z}_t+\bm{b}-\bar{\bm{g}}_{k,t}(\bm{z}_t)\bigr\rVert_2^2
+
\sum_{(i,j)\in S_k} w_{ij,t}^{(2)}(\bm{z}_t)\bigl\lVert\bm{g}_{ij,t}(\bm{z}_t)-\bar{\bm{g}}_{k,t}(\bm{z}_t)\bigr\rVert_2^2,
\label{eq:bias_variance_decomp}
\end{multline}
where \(W_{k,t}(\bm{z}_t)=\sum_{(i,j)\in S_k} w_{ij,t}^{(2)}(\bm{z}_t)\) is the total weight of cluster \(k\), and 
\begin{equation}
\bar{\bm{g}}_{k,t}(\bm{z}_t)
=
\frac{1}{W_{k,t}(\bm{z}_t)}
\sum_{(i,j)\in S_k} w_{ij,t}^{(2)}(\bm{z}_t)\bm{g}_{ij,t}(\bm{z}_t)
\end{equation}
is the weighted centroid of the teacher operators in that cluster.
\end{theorem}

\begin{proof}[Proof of~\Cref{thm:two_step_merging_error}]
Fix a partition \(\Pi=\{S_1,\dots,S_K\}\) of the index pairs \((i,j) \in \{1,\dots,K\}^2\). We construct a candidate student model \(\bm{f}^{\mathrm{st}}\) by assigning a single affine expert \(\bm{A}_k \bm{z}_t + \bm{b}_k\) to each group \(S_k\). We set the student's weighting for expert \(k\) to be the sum of the teacher's weights in that group, i.e., \(W_{k,t}(\bm{z}_t) = \sum_{(i,j) \in S_k} w_{ij,t}^{(2)}(\bm{z}_t)\). Since the original weights sum to \(1\), these aggregated weights define a valid MoE.

Using the expansion of \(\mathcal{T}_2\) and grouping terms by \(S_k\), the pointwise difference between the teacher and this candidate student is
\begin{equation}
\mathcal{T}_2(\bm{z}_t) - \bm{f}^{\mathrm{st}}(\bm{z}_t, t) = \sum_{k=1}^K \sum_{(i,j) \in S_k} w_{ij,t}^{(2)}(\bm{z}_t) \bigl(\bm{g}_{ij,t}(\bm{z}_t) - (\bm{A}_k \bm{z}_t + \bm{b}_k)\bigr).
\end{equation}
We now bound the squared norm of this difference. Since \(x \mapsto \lVert x\rVert_2^2\) is a convex function and the weights \(w_{ij,t}^{(2)}(\bm{z}_t)\) sum to 1, we can apply Jensen's inequality directly to the summation over all \(K^2\) pairs:
\begin{multline}
\Bigl\lVert\sum_{k=1}^K \sum_{(i,j) \in S_k} w_{ij,t}^{(2)}(\bm{z}_t) \bigl(\bm{g}_{ij,t}(\bm{z}_t) - (\bm{A}_k \bm{z}_t + \bm{b}_k)\bigr)\Bigr\rVert_2^2
\\\le
\sum_{k=1}^K \sum_{(i,j) \in S_k} w_{ij,t}^{(2)}(\bm{z}_t) \bigl\lVert\bm{g}_{ij,t}(\bm{z}_t) - (\bm{A}_k \bm{z}_t + \bm{b}_k)\bigr\rVert_2^2.
\end{multline}
Taking the expectation over \(\bm{z}_t \sim p_t\) and taking the infimum over all student parameters yields the upper bound \(\varepsilon_{t-1{:}t}\).

To derive the decomposition of the inner objective, fix a group \(S_k\) and define the weighted centroid
\begin{equation}
\bar{\bm{g}}_{k,t}(\bm{z}_t)
=
\frac{1}{W_{k,t}(\bm{z}_t)}
\sum_{(i,j)\in S_k} w_{ij,t}^{(2)}(\bm{z}_t)\bm{g}_{ij,t}(\bm{z}_t),
\end{equation}
where \(W_{k,t}(\bm{z}_t)=\sum_{(i,j)\in S_k} w_{ij,t}^{(2)}(\bm{z}_t)\).
For any affine parameters \((\bm{A},\bm{b})\), expand the squared norm by introducing \(\bar{\bm{g}}_{k,t}(\bm{z}_t)\):
\begin{multline}
\lVert\bm{A}\bm{z}_t+\bm{b}-\bm{g}_{ij,t}(\bm{z}_t)\rVert_2^2
=
\lVert\bm{A}\bm{z}_t+\bm{b}-\bar{\bm{g}}_{k,t}(\bm{z}_t)\rVert_2^2
\\+
\lVert\bar{\bm{g}}_{k,t}(\bm{z}_t)-\bm{g}_{ij,t}(\bm{z}_t)\rVert_2^2
+
2\bigl(\bm{A}\bm{z}_t+\bm{b}-\bar{\bm{g}}_{k,t}(\bm{z}_t)\bigr)^\top
\bigl(\bar{\bm{g}}_{k,t}(\bm{z}_t)-\bm{g}_{ij,t}(\bm{z}_t)\bigr).
\end{multline}
Multiplying by \(w_{ij,t}^{(2)}(\bm{z}_t)\) and summing over \((i,j)\in S_k\), the cross-term vanishes because
\begin{equation}
\sum_{(i,j)\in S_k} w_{ij,t}^{(2)}(\bm{z}_t)\bigl(\bar{\bm{g}}_{k,t}(\bm{z}_t)-\bm{g}_{ij,t}(\bm{z}_t)\bigr)=\bm{0}
\end{equation}
by the definition of \(\bar{\bm{g}}_{k,t}(\bm{z}_t)\).
Moreover, \(\sum_{(i,j)\in S_k} w_{ij,t}^{(2)}(\bm{z}_t)\lVert\bm{A}\bm{z}_t+\bm{b}-\bar{\bm{g}}_{k,t}(\bm{z}_t)\rVert_2^2
=
W_{k,t}(\bm{z}_t)\lVert\bm{A}\bm{z}_t+\bm{b}-\bar{\bm{g}}_{k,t}(\bm{z}_t)\rVert_2^2\).
The remaining terms therefore give exactly the claimed bias-plus-variance decomposition.
\end{proof}

The decomposition in~\Cref{thm:two_step_merging_error} reveals the fundamental trade-off in distilling MoE models. The error~\eqref{eq:bias_variance_decomp} separates into an \emph{alignment error} and an \emph{intrinsic variance}. The first term measures how well a student component can match the average behavior (the centroid) of a group of teacher components. This error is reducible because a sufficiently expressive student component can learn to approximate this centroid perfectly. The second term, \(\lVert\bm{g}-\bar{\bm{g}}\rVert_2^2\), represents the structural diversity of the teacher components within a chosen group. This error is irreducible for a fixed partition because no single affine map can simultaneously match physically distinct components. In short, unlike the single Gaussian case discussed in~\Cref{sec:operator_merging_foundations_in_the_gaussian_regime}, where the composition of linear maps remains linear and yields zero approximation error, the multimodal nature of Gaussian mixtures forces a compression of \(K^2\) teacher components into \(K\) student components, inducing nonzero approximation error.

%%%%%%%%%%%%%%%%%%%%%%%%%%%%%%%%%%%%%%%%%%%%%%%%%%%%%%%%%%%%
\subsection{Error propagation across merges}
\label{subsec:error_propagation_across_merges}

In the previous subsection we showed that in the Gaussian-mixture setting where both teacher and student are parameterized as \(K\)-component affine MoE operators, the student incurs an intrinsic approximation error when it is asked to match a composite teacher operator. This subsection explains how these approximation errors accumulate when the overall \(k\)-step composite teacher \(\mathcal{T}_k\) is approximated in multiple stages and then re-merged. Concretely, we consider the scenario where one first distill a student operator \(\bm{f}^{\mathrm{st}}_{t-k_1+1:t}(\bm z_t, t)\) to approximate \(\mathcal{T}_{k_1}(\bm z_t)\), then distill another operator \(\bm{f}^{\mathrm{st}}_{t-k+1:t-k_1}(\bm z_{t-k_1}, t-k_1)\) to approximate \(\mathcal{T}_{k_2}(\bm z_{t-k_1})\) where \(k=k_1+k_2\), and finally merge the two student operators into a single operator \(\bm{f}^{\mathrm{st}}_{t-k+1:t}(\bm z_t, t)\). The following~\Cref{thm:error_propagation_two_stage_merge} makes the error propagation explicit.

\begin{theorem}[Error propagation across a two-stage merge]
\label{thm:error_propagation_two_stage_merge}
Fix \(t\) and let \(\bm{z}_t\sim p_t\). Let \(k_1,k_2\in\mathbb{N}\) and set \(k=k_1+k_2\). Let \(\bm{f}^{\mathrm{st}}_{t-k_1+1:t}\) and \(\bm{f}^{\mathrm{st}}_{t-k+1:t-k_1}\) denote the two stagewise student operators, and let \(\bm{f}^{\mathrm{st}}_{t-k+1:t}\) denote the re-merged student operator that approximates their composition. Define the first-stage distillation loss
\begin{equation}
\varepsilon_{t-k_1+1{:}t}
\coloneqq
\mathbb{E}_{\bm{z}_t\sim p_t}\bigl[\lVert\bm{f}^{\mathrm{st}}_{t-k_1+1:t}(\bm{z}_t)-\mathcal{T}_{k_1}(\bm{z}_t)\rVert_2^2\bigr].
\end{equation}
Define the merge error
\begin{equation}
\varepsilon^{\mathrm{merge}}_{t-k+1{:}t}
\coloneqq
\mathbb{E}_{\bm{z}_t\sim p_t}\Bigl[\bigl\lVert\bm{f}^{\mathrm{st}}_{t-k+1:t}(\bm{z}_t)-\bigl(\bm{f}^{\mathrm{st}}_{t-k+1:t-k_1}\circ \bm{f}^{\mathrm{st}}_{t-k_1+1:t}\bigr)(\bm{z}_t)\bigr\rVert_2^2\Bigr].
\end{equation}
Assume the composite teacher operator \(\mathcal{T}_{k_2}\) is \(L^{\star}\)-Lipschitz, meaning
\begin{equation}
\lVert\mathcal{T}_{k_2}(\bm{u})-\mathcal{T}_{k_2}(\bm{v})\rVert_2 \le L^{\star}\lVert\bm{u}-\bm{v}\rVert_2
\end{equation}
for any \(\bm u, \bm v\in \mathbb{R}^d\). Then the final approximation error to the full composite teacher operator satisfies
\begin{multline}
\mathbb{E}_{\bm{z}_t\sim p_t}\bigl[\lVert\bm{f}^{\mathrm{st}}_{t-k+1:t}(\bm{z}_t)-\mathcal{T}_{k}(\bm{z}_t)\rVert_2^2\bigr]
\le
2\varepsilon^{\mathrm{merge}}_{t-k+1{:}t}
\\+
4\mathbb{E}_{\bm{z}_t\sim p_t}\Bigl[\bigl\lVert\bm{f}^{\mathrm{st}}_{t-k+1:t-k_1}\bigl(\bm{f}^{\mathrm{st}}_{t-k_1+1:t}(\bm{z}_t)\bigr)-\mathcal{T}_{k_2}\bigl(\bm{f}^{\mathrm{st}}_{t-k_1+1:t}(\bm{z}_t)\bigr)\bigr\rVert_2^2\Bigr]
\\+4(L^{\star})^2\varepsilon_{t-k_1+1{:}t}.
\end{multline}
\end{theorem}

\begin{proof}[Proof of~\Cref{thm:error_propagation_two_stage_merge}]
Define
\begin{equation}
\varepsilon^{\mathrm{final}}_{t-k+1{:}t}
\coloneqq
\mathbb{E}_{\bm{z}_t\sim p_t}\bigl[\lVert\bm{f}^{\mathrm{st}}_{t-k+1:t}(\bm{z}_t)-\mathcal{T}_{k}(\bm{z}_t)\rVert_2^2\bigr].
\end{equation}
Start from the pointwise decomposition
\begin{multline}
\bm{f}^{\mathrm{st}}_{t-k+1:t}(\bm{z}_t)-\mathcal{T}_{k}(\bm{z}_t)
=
\bigl(\bm{f}^{\mathrm{st}}_{t-k+1:t}(\bm{z}_t)-\bigl(\bm{f}^{\mathrm{st}}_{t-k+1:t-k_1}\circ \bm{f}^{\mathrm{st}}_{t-k_1+1:t}\bigr)(\bm{z}_t)\bigr)
\\+
\bigl(\bigl(\bm{f}^{\mathrm{st}}_{t-k+1:t-k_1}\circ \bm{f}^{\mathrm{st}}_{t-k_1+1:t}\bigr)(\bm{z}_t)-\mathcal{T}_{k}(\bm{z}_t)\bigr).
\end{multline}
Applying \(\lVert a+b\rVert_2^2\le 2\lVert a\rVert_2^2+2\lVert b\rVert_2^2\) and taking expectation under \(\bm{z}_t\sim p_t\) yields
\begin{equation}
\varepsilon^{\mathrm{final}}_{t-k+1{:}t}
\le
2\varepsilon^{\mathrm{merge}}_{t-k+1{:}t}
+
2\mathbb{E}_{\bm{z}_t\sim p_t}\Bigl[\bigl\lVert\bigl(\bm{f}^{\mathrm{st}}_{t-k+1:t-k_1}\circ \bm{f}^{\mathrm{st}}_{t-k_1+1:t}\bigr)(\bm{z}_t)-\mathcal{T}_{k}(\bm{z}_t)\bigr\rVert_2^2\Bigr].
\end{equation}

Using the teacher composition identity \(\mathcal{T}_{k}=\mathcal{T}_{k_2}\circ \mathcal{T}_{k_1}\), insert and subtract \(\mathcal{T}_{k_2}(\bm{f}^{\mathrm{st}}_{t-k_1+1:t}(\bm{z}_t))\) to obtain the pointwise equality
\begin{multline}
\bigl(\bm{f}^{\mathrm{st}}_{t-k+1:t-k_1}\circ \bm{f}^{\mathrm{st}}_{t-k_1+1:t}\bigr)(\bm{z}_t)-\mathcal{T}_{k}(\bm{z}_t)
=\\
\bigl(\bm{f}^{\mathrm{st}}_{t-k+1:t-k_1}\bigl(\bm{f}^{\mathrm{st}}_{t-k_1+1:t}(\bm{z}_t)\bigr)-\mathcal{T}_{k_2}\bigl(\bm{f}^{\mathrm{st}}_{t-k_1+1:t}(\bm{z}_t)\bigr)\bigr)
\\+
\Bigl(\mathcal{T}_{k_2}\bigl(\bm{f}^{\mathrm{st}}_{t-k_1+1:t}(\bm{z}_t)\bigr)-\mathcal{T}_{k_2}\bigl(\mathcal{T}_{k_1}(\bm{z}_t)\bigr)\Bigr).
\end{multline}
Applying \(\lVert a+b\rVert_2^2\le 2\lVert a\rVert_2^2+2\lVert b\rVert_2^2\), taking expectation under \(p_t\), and combining with the previous inequality yields
\begin{multline}
\varepsilon^{\mathrm{final}}_{t-k+1{:}t}
\le
2\varepsilon^{\mathrm{merge}}_{t-k+1{:}t}
\\+
4\mathbb{E}_{\bm{z}_t\sim p_t}\Bigl[\bigl\lVert\bm{f}^{\mathrm{st}}_{t-k+1:t-k_1}\bigl(\bm{f}^{\mathrm{st}}_{t-k_1+1:t}(\bm{z}_t)\bigr)-\mathcal{T}_{k_2}\bigl(\bm{f}^{\mathrm{st}}_{t-k_1+1:t}(\bm{z}_t)\bigr)\bigr\rVert_2^2\Bigr]
\\+
4\mathbb{E}_{\bm{z}_t\sim p_t}\Bigl[\bigl\lVert\mathcal{T}_{k_2}\bigl(\bm{f}^{\mathrm{st}}_{t-k_1+1:t}(\bm{z}_t)\bigr)-\mathcal{T}_{k_2}\bigl(\mathcal{T}_{k_1}(\bm{z}_t)\bigr)\bigr\rVert_2^2\Bigr].
\end{multline}

By the Lipschitz property of \(\mathcal{T}_{k_2}\),
\begin{multline}
\bigl\lVert\mathcal{T}_{k_2}\bigl(\bm{f}^{\mathrm{st}}_{t-k_1+1:t}(\bm{z}_t)\bigr)-\mathcal{T}_{k_2}\bigl(\mathcal{T}_{k_1}(\bm{z}_t)\bigr)\bigr\rVert_2
\\\le
L^{\star}\bigl\lVert\bm{f}^{\mathrm{st}}_{t-k_1+1:t}(\bm{z}_t)-\mathcal{T}_{k_1}(\bm{z}_t)\bigr\rVert_2.
\end{multline}
Squaring and taking expectation under \(p_t\) yields
\begin{equation}
\mathbb{E}_{\bm{z}_t\sim p_t}\Bigl[\bigl\lVert\mathcal{T}_{k_2}\bigl(\bm{f}^{\mathrm{st}}_{t-k_1+1:t}(\bm{z}_t)\bigr)-\mathcal{T}_{k_2}\bigl(\mathcal{T}_{k_1}(\bm{z}_t)\bigr)\bigr\rVert_2^2\Bigr]
\le
\bigl(L^{\star}\bigr)^2\varepsilon_{t-k_1+1{:}t}.
\end{equation}
Substituting this bound into the preceding inequality gives the claimed result.
\end{proof}

\begin{figure*}[!ht]
\centering
\includegraphics[width=1\linewidth]{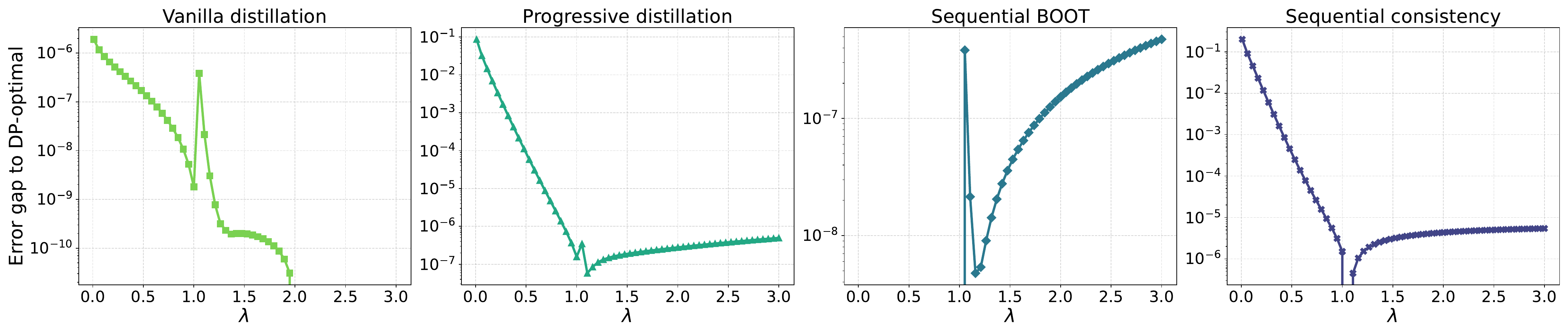}\\
\includegraphics[width=0.48\linewidth]{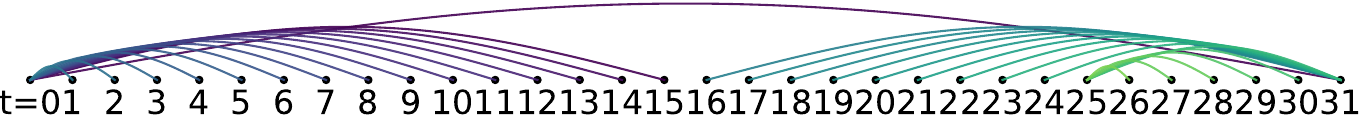}
\hfill
\includegraphics[width=0.48\linewidth]{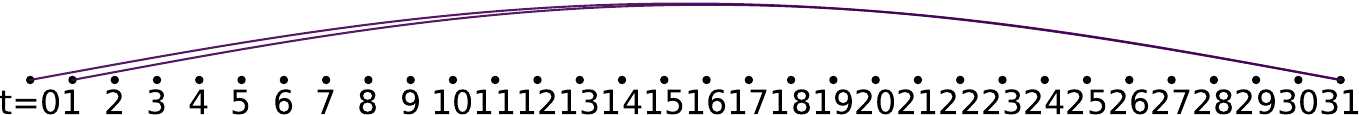}
\caption{\textbf{First row}: Error gap between four canonical strategies and the DP-optimal solution as a function of \(\lambda\), with \(T=32\) and \(s=6.4\). As predicted by~\Cref{thm:vanilla_distillation_optimality} and~\Cref{thm:sequential_boot_optimality}, sequential BOOT achieves optimality when \(\lambda \leq 1\), while vanilla trajectory distillation becomes optimal for sufficiently large \(\lambda>2\). \textbf{Second row}: Visualization of the DP-optimal merge plans at \(\lambda=1.08\) (left) and \(\lambda=2\) (right). Each arc represents a merge operation, with lighter colors indicating earlier merges and darker colors corresponding to later merges. For additional results under varying \(T\), \(s\), and covariance matrices \(\bm{\Lambda}\), please refer to~\ref{app:additional_experimental_results_on_pareto_dynamic_programming}.}
\label{fig:phase_transition}
\end{figure*}

\Cref{thm:error_propagation_two_stage_merge} separates three sources of error when approximating \(\mathcal{T}_{k_1+k_2}\) via stagewise distillation followed by a merge. The first term \(\varepsilon^{\mathrm{merge}}_{t-k+1{:}t}\) captures the loss incurred when the composed student \(\bm{f}^{\mathrm{st}}_{t-k+1:t-k_1}\circ \bm{f}^{\mathrm{st}}_{t-k_1+1:t}\), which is effectively \(K^2\)-component in the affine MoE class, is re-approximated by a single \(K\)-component operator \(\bm{f}^{\mathrm{st}}_{t-k+1:t}\). This is precisely the bias-variance trade-off governed by~\Cref{thm:two_step_merging_error}. The second term captures a ``distribution shift''. The second-stage student model is trained on the noisy data distribution \(p_{t-k_1}\), but in the composite pipeline it receives perturbed inputs from the imperfect first-stage student. This discrepancy means that even a well-trained student may fail if it cannot generalize to the shifted input distribution induced by upstream errors. Finally, the third term reveals a dynamics-induced amplification. Early approximation errors \(\varepsilon_{t-k_1+1{:}t}\) do not merely persist but are scaled by the teacher model's Lipschitz constant \((L^{\star})^2\). This implies that if the teacher operator is expansive (\(L^{\star} > 1\)), small errors in the early stages of the reverse process can grow exponentially.

%%%%%%%%%%%%%%%%%%%%%%%%%%%%%%%%%%%%%%%%%%%%%%%%%%%%%%%%%%%%%%%%%%%%%%%%%%%%%%%%
\section{Experiments}
\label{sec:experiments}

This section validates our theoretical framework through numerical analysis and empirical experimentation. In~\Cref{subsec:pareto_dynamic_programming_results}, we employ Pareto dynamic programming to numerically identify optimal merging strategies, confirming a phase transition governed by data variance. We verify these regimes and analyze the mechanics of error propagation using synthetic Gaussian and mixture datasets in~\Cref{subsec:experiments_on_synthetic_datasets}. Finally, in~\Cref{subsec:experiments_on_real_dataset}, we apply an MSSIMVAE~\citep{snell2017learning} to the CelebA~\citep{liu2015deep} dataset, demonstrating that the latent representations occupy the low-variance regime where sequential BOOT proves optimal.

%%%%%%%%%%%%%%%%%%%%%%%%%%%%%%%%%%%%%%%%%%%%%%%%%%%%%%%%%%%%
\subsection{Pareto dynamic programming results}
\label{subsec:pareto_dynamic_programming_results}

We implement the Pareto dynamic programming algorithm described in~\Cref{alg:pareto_dp_merge}. For visualization purposes, we set the total number of time steps to \(T=32\) and the student's training time per merge step to \(s=6.4\). Throughout our experiments, we adopt the cosine noise schedule, defined by \(\alpha_t = \cos((t/T)\cdot\pi/2)\) and \(\sigma_t = \sin((t/T)\cdot\pi/2)\), which is commonly used in prior works~\citep{nichol2021improved}. The results, showing the error gap between four canonical strategies and the DP-optimal solution, are presented in the first row of~\Cref{fig:phase_transition}. As predicted by~\Cref{thm:vanilla_distillation_optimality,thm:sequential_boot_optimality}, when \(\lambda \leq 1\), sequential BOOT achieves optimality (with zero error gap). As \(\lambda\) increases beyond \(2\), vanilla trajectory distillation becomes the optimal strategy. In the transitional regime between these extremes, the optimal merge plan exhibits a hybrid structure. In particular, parts of the merge path may align with sequential consistency (especially when \(\lambda\) is close to \(1\)), but the overall pattern deviates from any previously known strategy. We visualize two representative cases at \(\lambda=1.08\) and \(\lambda=2\) in the second row of~\Cref{fig:phase_transition}, where each arc denotes a merge operation and the arc color encodes the merge order. Lighter colors indicate earlier merges, while darker colors correspond to later ones. For ablation studies under varying \(T\), \(s\), and covariance matrices \(\bm{\Lambda}\), please refer to~\ref{app:additional_experimental_results_on_pareto_dynamic_programming}.

%%%%%%%%%%%%%%%%%%%%%%%%%%%%%%%%%%%%%%%%%%%%%%%%%%%%%%%%%%%%
\subsection{Experiments on synthetic datasets}
\label{subsec:experiments_on_synthetic_datasets}

\paragraph{Gaussian dataset}
To validate our theoretical findings in~\Cref{subsec:phase_transitions_in_the_optimal_strategy}, we conduct controlled simulations on a \(1\)-dimensional Gaussian dataset \(\mathcal{N}(0, \lambda)\). We implement the student model as a time-dependent diagonal affine operator, defined as \(\bm{f}^{\mathrm{st}}(\bm{z}_t, t) = \mathrm{Diag}(\bm{a}_t)\bm{z}_t\). To capture the temporal dynamics, the diagonal coefficients \(\bm{a}_t\) are predicted from a \(64\)-dimensional sinusoidal time embedding \(\mathrm{emb}(t/T)\) via a learnable linear projection. This setup mimics the time-conditioning mechanism of practical diffusion models while remaining amenable to analysis.

We train the student models using Stochastic Gradient Descent (SGD) with a learning rate of \(10^{-3}\) and a batch size of \(2{,}560\) for \(100\) epochs. To ensure a fair comparison across strategies, we initialize the student weights using a closed-form least-squares projection of the corresponding teacher operator. Performance is evaluated using the mean signed error between the distilled student operator and the surrogate composite teacher \(\widetilde{\mathcal{T}}_T\), averaged over \(10\) independent trials.

The results for \(T=32\) are presented in~\Cref{tab:synthetic_T_32}. Consistent with the phase transition predicted by~\Cref{thm:vanilla_distillation_optimality,thm:sequential_boot_optimality}, we observe that sequential BOOT achieves the lowest error in strictly low-variance regimes (\(\lambda < 1\)). At the boundary condition (\(\lambda=1\)), the theoretical gap closes and the methods perform comparably within optimization noise, while vanilla distillation dominates in high-variance regimes (\(\lambda > 1\)). Extended results for different \(T\) are provided in~\ref{app:additional_results_on_synthetic_datasets}.

\begin{table*}[htb]
\centering
\caption{Simulation results of four canonical trajectory distillation methods on synthetic datasets with varying covariance parameter \(\lambda\). Each entry reports the mean \emph{signed} error between the student operator and the surrogate composite operator, averaged over \(10\) independent trials. Error bars represent standard deviation across trials. The best result for each setting is underlined. Results are consistent with our theoretical predictions.}
\label{tab:synthetic_T_32}
\setlength{\tabcolsep}{4pt}
\smallskip
\scriptsize
\begin{tabular}{lcccc}
\toprule
\small \(\lambda\) & \small Vanilla & \small Progressive & \small BOOT & \small Consistency \\ 
\midrule
0.20 & \(7.3 \times 10^{-4} \pm 1.3 \times 10^{-4}\) & \(3.3 \times 10^{-2} \pm 5.2 \times 10^{-4}\) & \(\underline{4.7 \times 10^{-4}} \pm 7.1 \times 10^{-5}\) & \(1.3 \times 10^{-1} \pm 7.8 \times 10^{-3}\) \\ 
0.50 & \(4.5 \times 10^{-4} \pm 9.4 \times 10^{-5}\) & \(4.8 \times 10^{-3} \pm 1.1 \times 10^{-4}\) & \(\underline{3.9 \times 10^{-4}} \pm 6.1 \times 10^{-5}\) & \(2.6 \times 10^{-2} \pm 2.3 \times 10^{-3}\) \\ 
1.00 & \(\underline{4.7 \times 10^{-5}} \pm 6.4 \times 10^{-6}\) & \(1.3 \times 10^{-4} \pm 9.1 \times 10^{-6}\) & \(5.3 \times 10^{-5} \pm 7.0 \times 10^{-6}\) & \(-1.2 \times 10^{-4} \pm 1.1 \times 10^{-4}\) \\ 
1.02 & \(-2.8 \times 10^{-2} \pm 5.3 \times 10^{-6}\) & \(\underline{-2.8 \times 10^{-2}} \pm 5.5 \times 10^{-6}\) & \(-2.8 \times 10^{-2} \pm 3.8 \times 10^{-6}\) & \(-2.9 \times 10^{-2} \pm 1.1 \times 10^{-4}\) \\ 
2.00 & \(\underline{-2.1 \times 10^{-3}} \pm 6.2 \times 10^{-5}\) & \(-2.4 \times 10^{-3} \pm 7.1 \times 10^{-5}\) & \(-2.2 \times 10^{-3} \pm 1.1 \times 10^{-4}\) & \(-7.1 \times 10^{-3} \pm 6.9 \times 10^{-4}\) \\ 
5.00 & \(\underline{-4.1 \times 10^{-3}} \pm 3.7 \times 10^{-4}\) & \(-4.7 \times 10^{-3} \pm 1.7 \times 10^{-4}\) & \(-4.9 \times 10^{-3} \pm 4.8 \times 10^{-4}\) & \(-1.1 \times 10^{-2} \pm 8.3 \times 10^{-4}\) \\ 
\bottomrule
\end{tabular}
\end{table*}

\paragraph{Gaussian mixture dataset}
To validate our theoretical findings in~\Cref{sec:approximation_error_and_error_propagation_in_the_gaussian_mixture_regime}, we conduct controlled simulations on a \(2\)-dimensional Gaussian mixture dataset. The real distribution \(p_0\) is an isotropic mixture of \(K=8\) Gaussian modes with means \(\{\bm{\mu}_k\}_{k=1}^K\) uniformly spaced on a circle of radius \(R=5.0\) and covariance \(\bm{\Sigma}_k = 0.3^2 \bm{I}\). We implement the teacher model analytically using the closed-form optimal denoising estimator derived in~\Cref{prop:optimal_mixture_denoising_estimator}. And we implement the student model as a time-conditioned affine MoE, where the weighting function and the affine parameters are predicted from the state and sinusoidal time embedding. We train the student models using Adam~\citep{kingma2015adam} with a learning rate of \(10^{-3}\) and a batch size of \(2{,}560\). To ensure a fair comparison across strategies, we initialize the student by training it to match the single-step teacher operator for \(10{,}000\) epochs.

We first empirically evaluate the approximation error by fixing the student model to \(K=8\) experts and varying the number of denoising steps \(T \in \{2^0, \dotsc, 2^9\}\). In this experiment, we use vanilla distillation to merge the entire teacher trajectory into a single student step. As \(T\) increases, the number of components in the composite teacher operator \(\mathcal{T}_T(\bm{z}_T)\) grows according to the expansion logic in~\Cref{lem:two_step_moe_expansion}, making the distillation target increasingly complex. We visualize the generated samples for \(T \in \{8, 32, 128, 512\}\) in~\Cref{fig:gmm_vary_T} and report the mean distillation loss with standard deviations across 10 independent trials in~\Cref{fig:gmm_error_plot}. Consistent with the predictions of~\Cref{thm:two_step_merging_error}, the approximation error becomes non-zero for any \(T > 1\) and grows monotonically with the trajectory length. Qualitatively, the generated samples cannot fully match the real distribution and interpolate between two adjacent modes.

\begin{figure}[htb]
\centering
\includegraphics[width=1\linewidth]{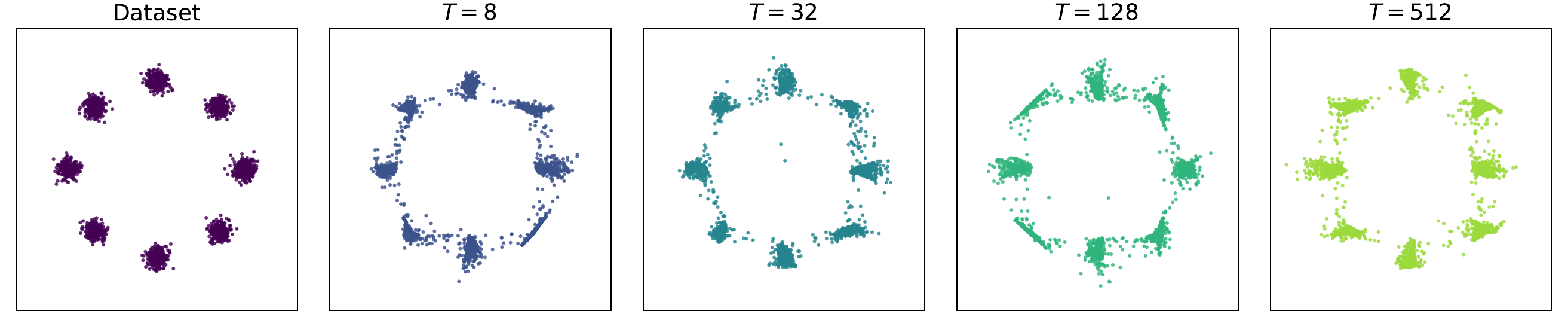}
\caption{Generated samples under different \(T\). We fix the student to \(K=8\) experts and distill a teacher trajectory of length \(T\) into a single student step. The leftmost panel shows real dataset samples. The remaining panels show student samples for \(T \in \{8,32,128,512\}\). Samples interpolate between adjacent modes, indicating a nonzero approximation error predicted by~\Cref{thm:two_step_merging_error}.}
\label{fig:gmm_vary_T}
\end{figure}

\begin{figure}[htb]
\centering
\includegraphics[width=0.8\linewidth]{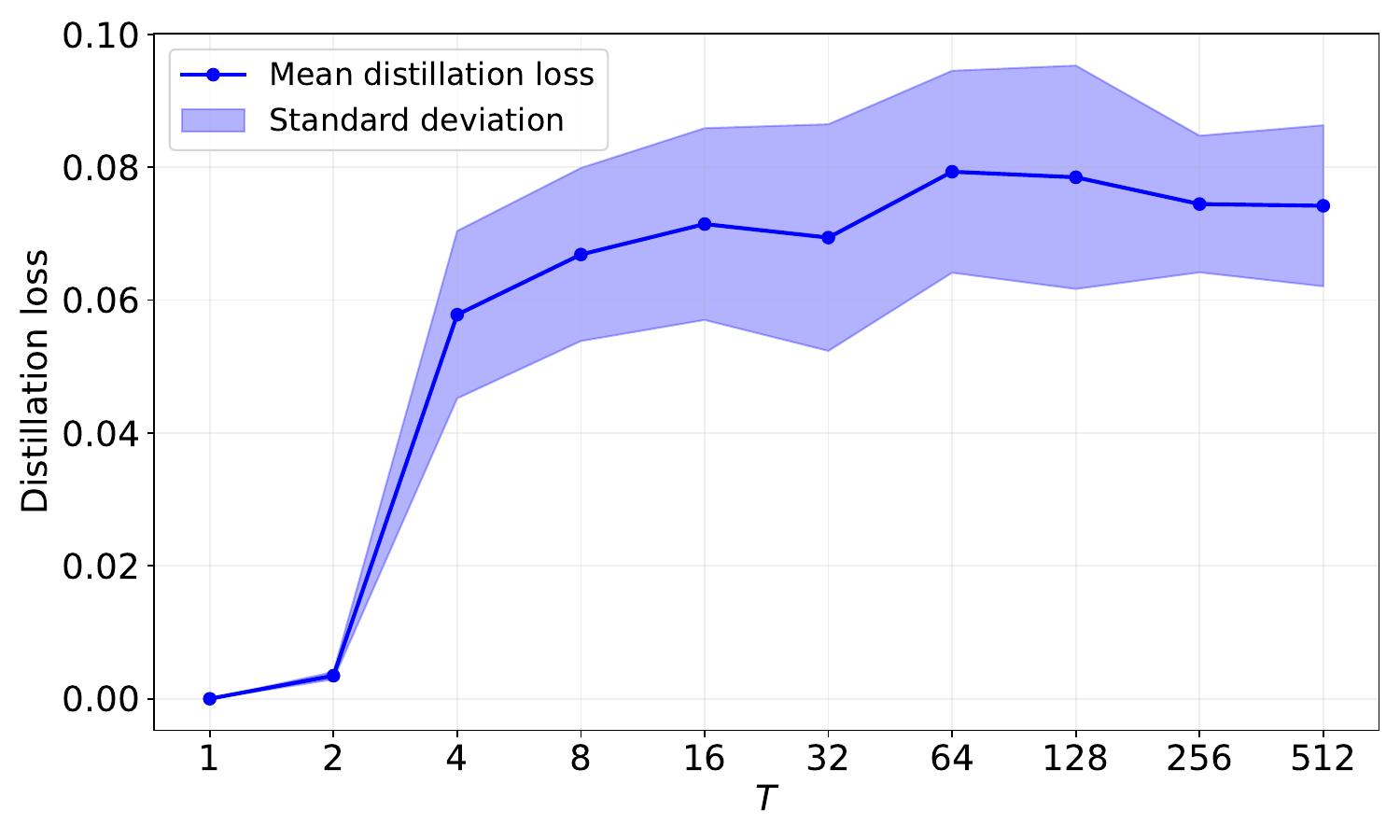}
\caption{Distillation loss versus \(T\). Mean vanilla-distillation loss \(\pm\) one standard deviation over \(10\) trials when merging a teacher trajectory of length \(T \in \{2^0,\dotsc,2^9\}\) into one student step with \(K=8\) experts. The loss becomes non-zero for \(T>1\) and increases monotonically with \(T\).}
\label{fig:gmm_error_plot}
\end{figure}

We then validate the propagation of approximation error by implementing canonical distillation strategies while fixing \(T=32\) and \(K=8\). As shown in~\Cref{fig:gmm_distillation}, the teacher model closely matches the data distribution, whereas all distilled students exhibit varying degrees of interpolation between adjacent modes. This qualitative degradation is consistent across strategies, indicating that approximation errors introduced by intermediate merges can accumulate and propagate to the final student operator. These observations support the mechanism formalized in~\Cref{thm:error_propagation_two_stage_merge}. Further details of the experimental settings are provided in~\ref{app:additional_results_on_synthetic_datasets}.

\begin{figure}[htb]
\centering
\includegraphics[width=1\linewidth]{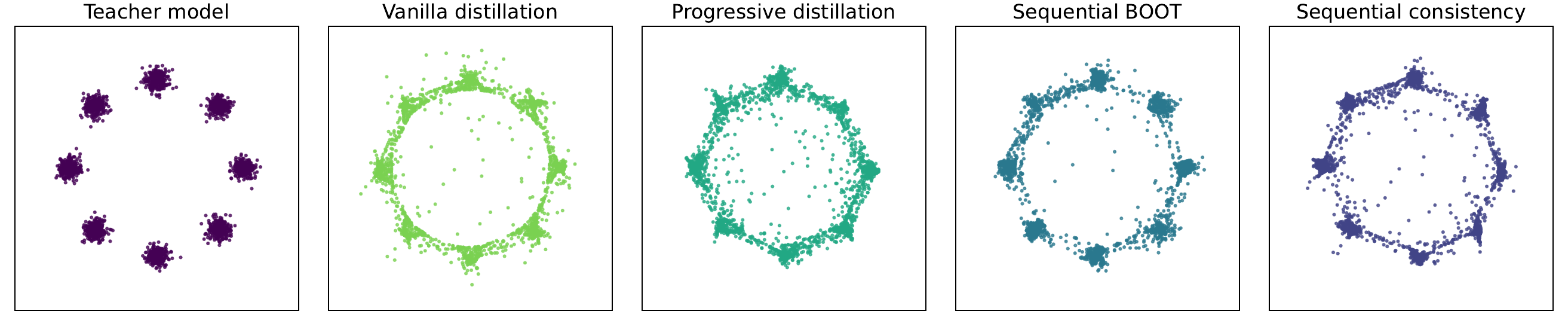}
\caption{Generated samples from the teacher and distilled students under canonical strategies. We fix \(T=32\) and \(K=8\). The teacher matches the dataset modes, whereas all strategies produce samples that interpolate between adjacent modes with varying strength, consistent with error propagation across merges (\Cref{thm:error_propagation_two_stage_merge}).}
\label{fig:gmm_distillation}
\end{figure}

%%%%%%%%%%%%%%%%%%%%%%%%%%%%%%%%%%%%%%%%%%%%%%%%%%%%%%%%%%%%
\subsection{Experiments on real dataset}
\label{subsec:experiments_on_real_dataset}

\paragraph{\rev{Approximate Gaussian structure in the latent space of CelebA}} To validate our findings in a realistic setting, we perform trajectory distillation in the latent space of a Multiscale Structural Similarity VAE (MSSIMVAE)~\citep{snell2017learning} pretrained on CelebA~\citep{liu2015deep}. We extract latent codes from the encoder and compute their empirical covariance. \Cref{fig:latent_cov_matrix} shows that the covariance is close to diagonal and that most per-dimension variances satisfy \(\lambda_i < 1\). This matches the low-variance regime studied in~\Cref{thm:sequential_boot_optimality}, in which sequential BOOT is optimal. See~\ref{app:additional_results_on_real_dataset} for the full experimental setup and ablations across different numbers of training epochs.

\begin{figure}[htb]
\centering
\includegraphics[width=0.48\linewidth]{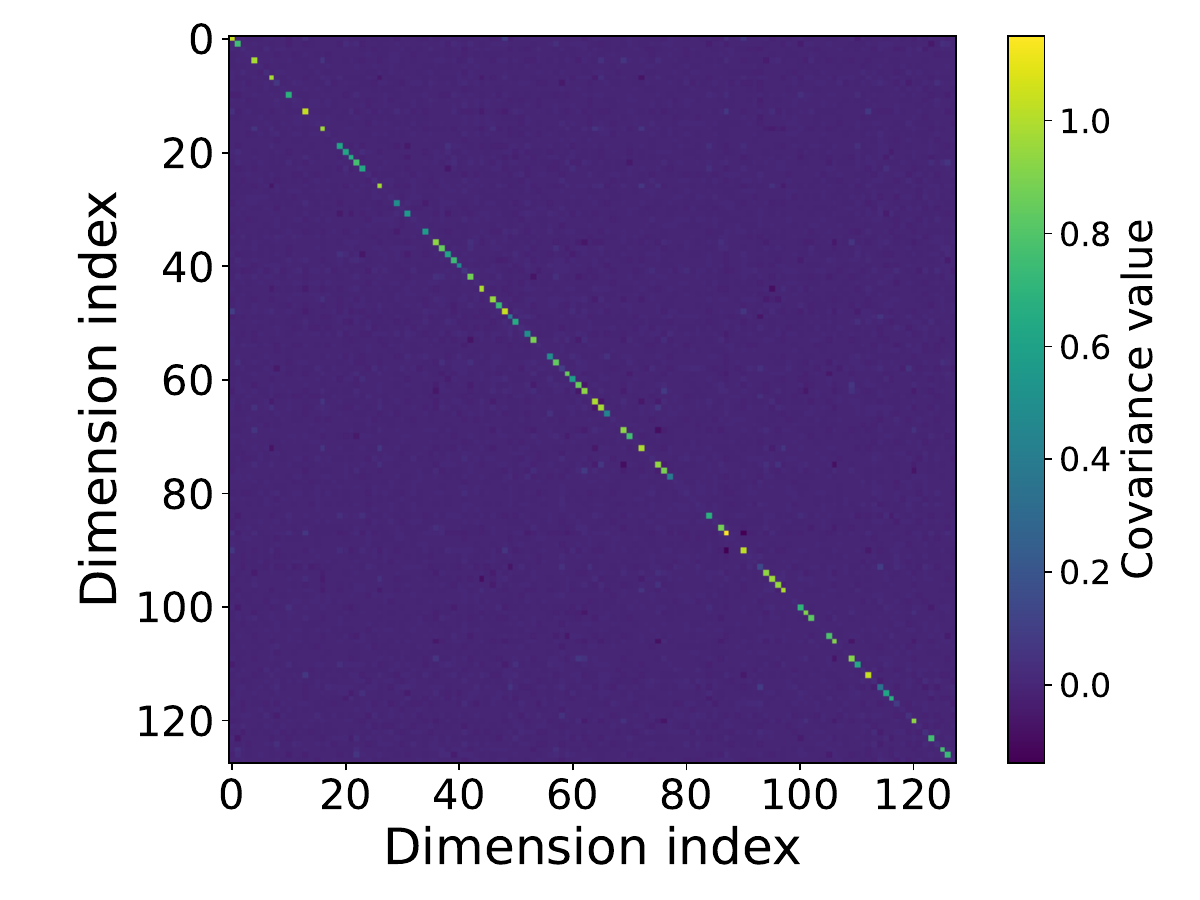}
\includegraphics[width=0.48\linewidth]{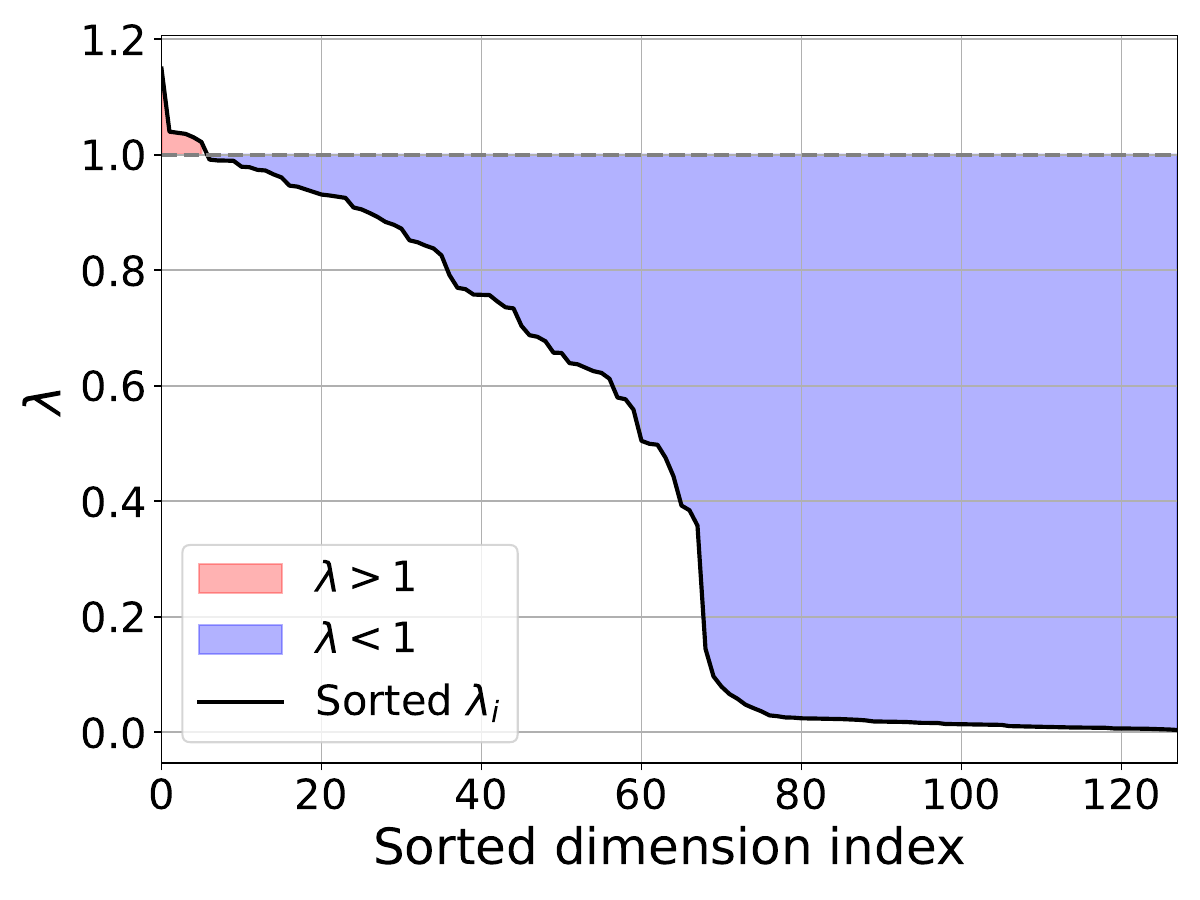}
\caption{Covariance structure of CelebA latent codes obtained using a pretrained MSSIMVAE. \textbf{Left}: Heatmap of the full empirical covariance matrix. \textbf{Right}: Sorted diagonal entries (variances) of the covariance matrix. While the covariance matrix is approximately diagonal, the variances are not all equal to \(1\), indicating that the latent distribution deviates from the standard Gaussian prior. Most diagonal entries are below \(1\), and Pareto dynamic programming analysis confirms that the theoretically optimal merging strategy in this setting is sequential BOOT.}
\label{fig:latent_cov_matrix}
\end{figure}

We compare four distillation strategies by sampling latent codes from each distilled student and decoding them into image space. We measure student quality using the pixel-wise \(L_2\) distance to the output of the surrogate teacher for the same input noise. As shown in~\Cref{fig:real_decoded_diff_epoch10000}, sequential BOOT achieves the smallest reconstruction error and most closely matches the teacher across samples. \rev{As further evidence beyond pixel-wise \(L_2\), \Cref{tab:celeba_fid_strategies} shows that sequential BOOT also achieves the lowest Fr\'echet Inception Distance (FID)~\citep{heusel2017gans} among the four strategies, indicating that its advantage is reflected in perceptual image quality as well.}

\begin{figure}[htb]
\centering
\includegraphics[width=0.9\linewidth]{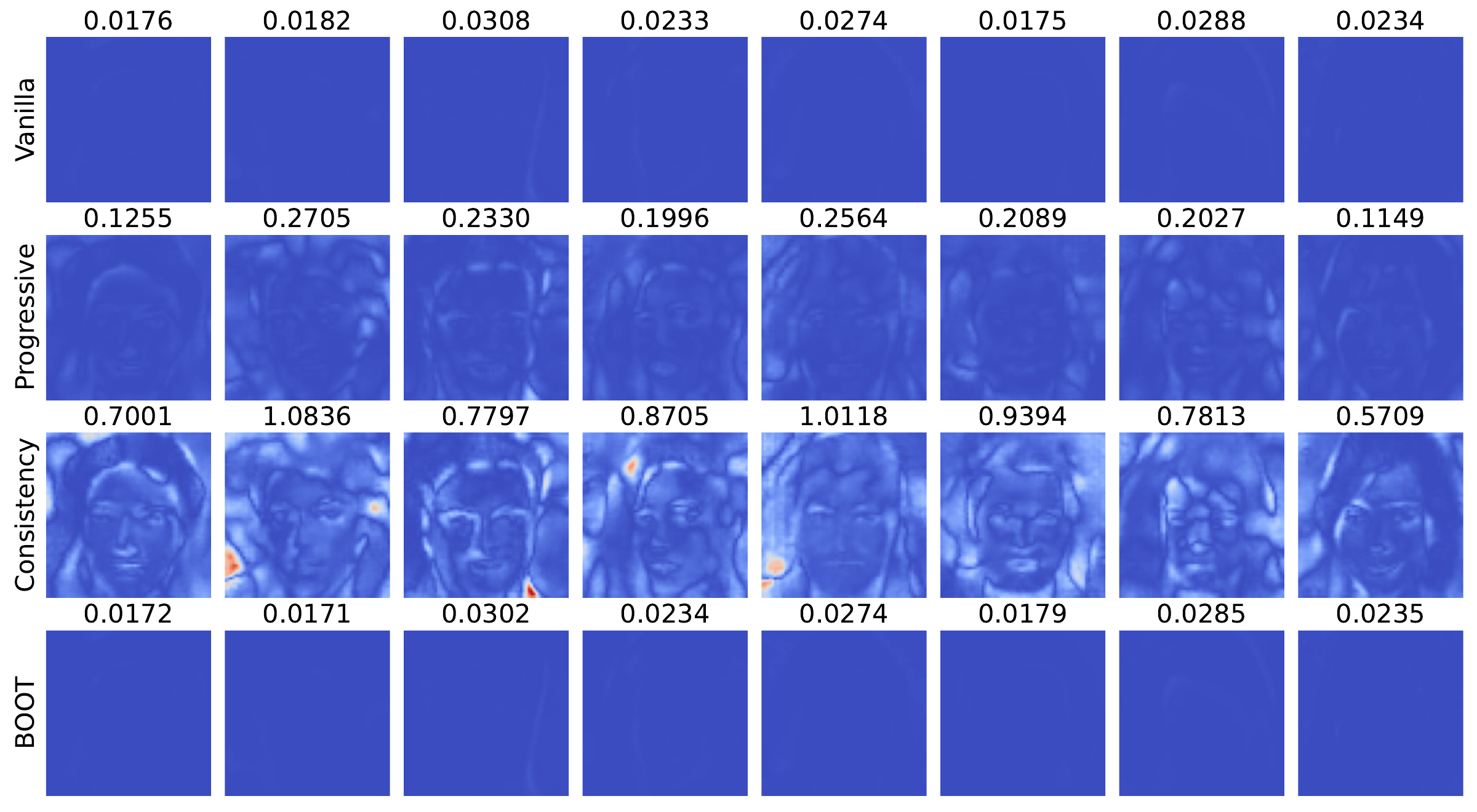}
\caption{Decoded error maps for different distillation strategies at epoch \(10{,}000\). Each row corresponds to a method. For each sample, we show the heatmap of the absolute pixel-wise difference to the surrogate teacher output. Larger values indicate greater deviation. The number above each sample reports the pixel-wise \(L_2\) distance. Sequential BOOT produces the lowest errors and the closest match to the teacher model.}
\label{fig:real_decoded_diff_epoch10000}
\end{figure}

\begin{table}[htb]
\centering
\small
\caption{\rev{Fr\'echet Inception Distance (FID)~\citep{heusel2017gans} on decoded images for the CelebA experiment, computed against the decoded outputs of the surrogate teacher. Lower is better. BOOT achieves the best FID, which supports that the theoretically favorable strategy also leads to improved perceptual quality beyond pixel-wise \(L_2\) matching.}}
\label{tab:celeba_fid_strategies}
\smallskip
\rev{\begin{tabular}{lcccc}
\toprule
Method & Vanilla & Progressive & Consistency & BOOT \\
\midrule
FID \(\downarrow\) & \(0.6382\) & \(1.2248\) & \(1.3419\) & \(\textbf{0.5776}\) \\
\bottomrule
\end{tabular}}
\end{table}

\paragraph{\rev{Approximate Gaussian mixture structure in CIFAR-10}}
\rev{We then perform trajectory distillation in the pixel space of a DDIM~\citep{song2021denoising} model pretrained on CIFAR-10~\citep{krizhevsky2009learning}. Unlike the approximately Gaussian latent distribution observed for CelebA, CIFAR-10 in pixel space exhibits a strongly multimodal structure, making it a natural testbed for the Gaussian-mixture regime studied in~\Cref{sec:approximation_error_and_error_propagation_in_the_gaussian_mixture_regime}.}

\rev{We instantiate all four canonical trajectory-distillation strategies using the same DDIM-wrapped \(x_0\)-predicting U-Net student architecture, always initialized from the pretrained teacher for a fair comparison. Vanilla distillation trains a one-shot map from \( \bm{x}_T \) to \( \bm{x}_0 \), progressive distillation recursively merges adjacent pairs of teacher steps, sequential BOOT keeps the input fixed at \( \bm{x}_T \) and gradually extends the target toward cleaner states, and sequential consistency progressively extends the direct map from \( \bm{x}_t \) to \( \bm{x}_0 \). In all cases, supervision is provided by deterministic teacher compositions, and each stage or merge is optimized with a mean squared error loss for \(10{,}000\) updates using AdamW~\citep{loshchilov2017decoupled}. We compare the distilled students by visual inspection and also report Fr\'echet Inception Distance (FID)~\citep{heusel2017gans} computed both against teacher-generated outputs and against the CIFAR-10 dataset. See~\ref{app:additional_results_on_real_dataset} for the detailed experimental setup.}

\begin{figure}[tb]
\centering
\includegraphics[width=0.8\linewidth]{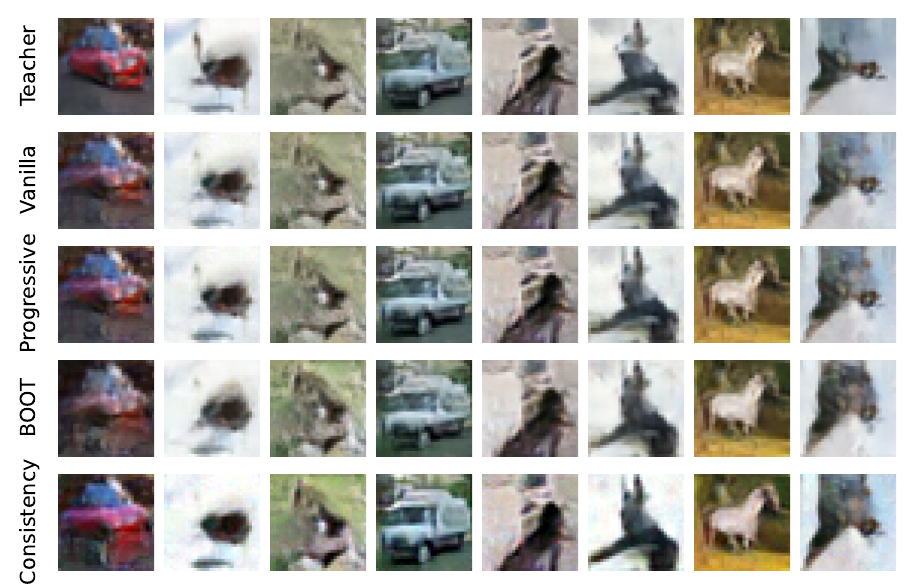}
\caption{\rev{Generated CIFAR-10 samples from the pretrained teacher and the four distilled students. All methods use the same DDIM-wrapped \(x_0\)-predicting U-Net architecture and are initialized from the teacher, but the distilled students exhibit visible deviations from the teacher outputs. This qualitative gap is consistent with the approximation error and error propagation mechanisms analyzed in~\Cref{thm:two_step_merging_error,thm:error_propagation_two_stage_merge}. Among the distilled methods, progressive distillation produces samples that are visually closest to the teacher.}}
\label{fig:comparison_cifar10}
\end{figure}

\begin{table}[b]
\centering
\small
\caption{\rev{Fr\'echet Inception Distance (FID)~\citep{heusel2017gans} for the CIFAR-10 experiment. The first row reports FID computed against teacher-generated samples, measuring how closely each distilled model matches the teacher distribution. The second row reports FID computed against the CIFAR-10 dataset, measuring realism with respect to the data distribution. Lower is better in both cases. Progressive distillation achieves the best FID among the distilled students, while the teacher itself attains the best FID against the dataset.}}
\label{tab:cifar10_fid_strategies}
\smallskip
\rev{\begin{tabular}{ccccc}
\toprule
Teacher & Vanilla & Progressive & BOOT & Consistency \\
\midrule
-- & \(20.7735\) & \(\textbf{17.6398}\) & \(24.6456\) & \(44.1261\) \\
\(\textbf{26.0862}\) & \(44.8568\) & \(39.7271\) & \(63.0115\) & \(55.5787\) \\
\bottomrule
\end{tabular}}
\end{table}

\rev{The results are consistent with the Gaussian-mixture analysis in~\Cref{sec:approximation_error_and_error_propagation_in_the_gaussian_mixture_regime}. In particular, \Cref{thm:two_step_merging_error} shows that when a multi-step composite teacher is compressed into a student of fixed capacity, a nonzero approximation error can arise from the mismatch between the complexity of the composite operator and that of the student. Moreover, \Cref{thm:error_propagation_two_stage_merge} shows that such local errors may further accumulate across successive merges. Consequently, unlike the single-Gaussian setting, trajectory distillation in the strongly multimodal pixel space of CIFAR-10 can incur irreducible distortion even when each merge is trained adequately. This behavior is visible in~\Cref{fig:comparison_cifar10}, where all distilled students differ qualitatively from the teacher-generated samples, and is also reflected quantitatively in~\Cref{tab:cifar10_fid_strategies}. Among the four distilled students, progressive distillation achieves the best FID both against the teacher outputs (\(17.6398\)) and against the CIFAR-10 dataset (\(39.7271\)), while the teacher itself attains the best dataset FID overall (\(26.0862\)). This empirical ranking does not contradict our theory. Our Gaussian-mixture analysis establishes the presence of approximation error and error propagation, but it does not claim that any one of the canonical strategies must be globally optimal for general multimodal data. In this experiment, progressive distillation appears to provide the most favorable trade-off, likely because pairwise merging limits the complexity growth of each target operator and mitigates the accumulation of errors across long horizons.}

%%%%%%%%%%%%%%%%%%%%%%%%%%%%%%%%%%%%%%%%%%%%%%%%%%%%%%%%%%%%%%%%%%%%%%%%%%%%%%%%
\section{Conclusion}

In this work, we studied diffusion trajectory distillation through the lens of operator merging. In the linear Gaussian regime, we showed that merging can induce signal shrinkage under finite training time and proved a variance-driven phase transition in optimal merge plans, favoring sequential BOOT in low-variance regimes and vanilla trajectory distillation in high-variance regimes. Building on this characterization, we introduced a Pareto dynamic programming algorithm that computes globally optimal merge plans in mixed-variance, high-dimensional settings. We further analyzed the Gaussian-mixture setting by modeling the teacher as an affine mixture-of-experts operator, showing that long-step compositions increase distillation-target complexity, leading to intrinsic approximation error for a fixed-capacity student and enabling error propagation across merges. Experiments on synthetic datasets and on CelebA latent diffusion using a pretrained MSSIMVAE corroborate these theoretical findings.

Several directions follow naturally from our analysis. On the theoretical side, the Gaussian-mixture setting warrants a full characterization. Because the nonlinearity of mixtures precludes the exact closed-form analysis possible in the linear regime, future work should focus on characterizing the student's optimization error and establishing optimal merge plans for these complex distributions. On the practical side, designing adaptive merge plans that exploit estimated variance structure and distillation-target complexity could improve robustness when the data distribution departs from our idealized assumptions.

%%%%%%%%%%%%%%%%%%%%%%%%%%%%%%%%%%%%%%%%%%%%%%%%%%%%%%%%%%%%%%%%%%%%%%%%%%%%%%%%
\section*{CRediT authorship contribution statement}

\textbf{Weiguo Gao:} Funding acquisition, Project administration, Supervision, Resources, Writing -- review \& editing; \textbf{Ming Li:} Conceptualization, Formal analysis, Investigation, Methodology, Validation, Visualization, Writing - original draft, Writing - review \& editing.

%%%%%%%%%%%%%%%%%%%%%%%%%%%%%%%%%%%%%%%%%%%%%%%%%%%%%%%%%%%%%%%%%%%%%%%%%%%%%%%%
\section*{Declaration of competing interest}

The authors declare that they have no known competing financial interests or personal relationships that could have appeared to influence the work reported in this paper.

%%%%%%%%%%%%%%%%%%%%%%%%%%%%%%%%%%%%%%%%%%%%%%%%%%%%%%%%%%%%%%%%%%%%%%%%%%%%%%%%
\section*{Acknowledgements}

The research of this work was partially supported by National Key R\&D Program of China under Grant No.~2021YFA1003305.

%%%%%%%%%%%%%%%%%%%%%%%%%%%%%%%%%%%%%%%%%%%%%%%%%%%%%%%%%%%%%%%%%%%%%%%%%%%%%%%%
\bibliography{references}
\bibliographystyle{abbrvnat}

%%%%%%%%%%%%%%%%%%%%%%%%%%%%%%%%%%%%%%%%%%%%%%%%%%%%%%%%%%%%%%%%%%%%%%%%%%%%%%%%
\appendix

% \paragraph{Roadmap} The appendix is organized as follows:
% \begin{itemize}
% \item \ref{app:diffusion_trajectory_distillation_methods} details the four canonical trajectory distillation methods, complementing~\Cref{tab:distillation_methods}.
% \item \ref{app:additional_experimental_results_on_pareto_dynamic_programming} presents extended results on dynamic programming, complementing~\Cref{subsec:pareto_dynamic_programming_results}.
% \item \ref{app:additional_results_on_synthetic_datasets} provides further experiments on synthetic datasets, complementing~\Cref{subsec:experiments_on_synthetic_datasets}.
% \item \ref{app:additional_results_on_real_dataset} provides further experiments on the real dataset, complementing~\Cref{subsec:experiments_on_real_dataset}.
% \end{itemize}

%%%%%%%%%%%%%%%%%%%%%%%%%%%%%%%%%%%%%%%%%%%%%%%%%%%%%%%%%%%%%%%%%%%%%%%%%%%%%%%%
\section{Diffusion trajectory distillation methods}
\label{app:diffusion_trajectory_distillation_methods}

Trajectory distillation accelerates sampling by training a student model to approximate a composite operator \(\mathcal{T}_k(\bm{z}_t)\) as defined in~\Cref{def:composite_teacher_operator}. We describe four canonical methods, including two widely adopted baselines and two sequential variants we design.

\paragraph{Vanilla distillation} 
Originally proposed by~\citet{luhman2021knowledge}, vanilla distillation trains a student model \(\bm{A}^{\mathrm{st}}\) to match the entire teacher trajectory in a single forward pass:
\begin{equation}
\bm{A}^{\mathrm{st}}(\bm{z}_T, T) \approx \mathcal{T}_T(\bm{z}_T) = \bm{z}_0.
\end{equation}
The goal is to learn a single-step mapping that directly transforms the noisy input \(\bm{z}_T\) into a clean sample \(\bm{z}_0\). See~\Cref{fig:vanilla_conceptual_illustration} for a conceptual illustration. While efficient during inference, this approach requires the student to capture a highly complex transformation, which can be challenging in practice.

\begin{figure}[htb]
\centering
\includegraphics[width=0.9\linewidth]{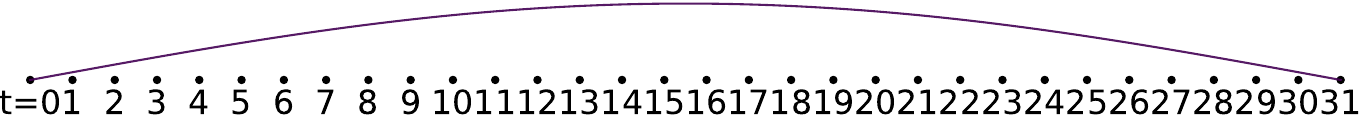}
\caption{Conceptual illustration of \emph{vanilla distillation}. The entire teacher trajectory is being merged into a single step.}
\label{fig:vanilla_conceptual_illustration}
\end{figure}

\paragraph{Progressive distillation} 
Progressive distillation~\citep{lin2024sdxl,meng2023distillation,salimans2022progressive} adopts a curriculum of merging two teacher steps at a time. In each round, the student model \(\bm{A}^{\mathrm{st}}\) is trained to approximate
\begin{equation}
\bm{A}^{\mathrm{st}}(\bm{z}_t, t) \approx \mathcal{T}_2(\bm{z}_t) = \bm{z}_{t-2}
\end{equation}
for even-indexed timesteps \(t\). See~\Cref{fig:progressive_conceptual_illustration} for a conceptual illustration. Once trained, the student replaces the teacher, and the number of sampling steps is halved. This procedure is repeated until only one step remains. The progressive reduction of the trajectory length eases optimization by distributing the learning burden over multiple stages.

\begin{figure}[htb]
\centering
\includegraphics[width=0.9\linewidth]{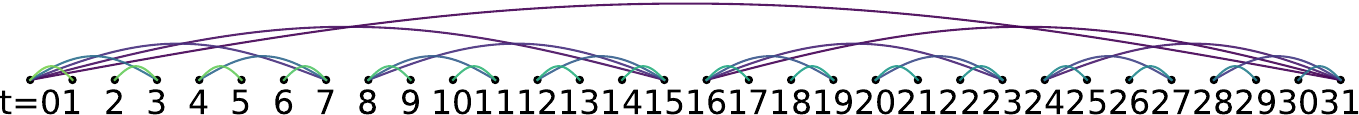}
\caption{Conceptual illustration of \emph{progressive distillation}. The student model is trained with a curriculum that progressively merges pairs of teacher steps into a single student update, thereby halving the number of sampling steps required at each stage.}
\label{fig:progressive_conceptual_illustration}
\end{figure}

\paragraph{Sequential consistency}
The original form of consistency distillation~\citep{li2024reward,song2023consistency} trains a student \(\bm{A}^{\mathrm{st}}\) to satisfy the self-consistency condition
\begin{equation}
\bm{A}^{\mathrm{st}}(\bm{z}_t, t) \approx \bm{A}^{\mathrm{st}}(\bm{z}_{t-1}, t-1),
\end{equation}
enforcing that all intermediate predictions follow a consistent trajectory toward the data distribution endpoint. We propose a sequential variant that adopts a curriculum strategy: the student model is explicitly trained to map each \(\bm{z}_t\) directly to the clean sample \(\bm{x}_0=\bm{z}_0\), using increasingly long trajectories. Formally, we train
\begin{equation}
\bm{A}^{\mathrm{st}}(\bm{z}_t, t) \approx \mathcal{T}_t(\bm{z}_t) = \bm{z}_0,
\end{equation}
for \(t = 1, 2, \dotsc, T\). See~\Cref{fig:consistency_conceptual_illustration} for a conceptual illustration. The student thus progressively learns to compress longer denoising trajectories into a single step, with early training focusing on short and easier sub-trajectories.

\begin{figure}[htb]
\centering
\includegraphics[width=0.9\linewidth]{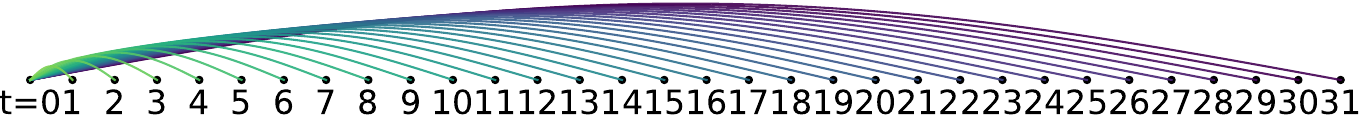}
\caption{Conceptual illustration of \emph{sequential consistency distillation}. The student model is trained to map each \(\bm{z}_t\) directly to the clean sample \(\bm{z}_0\), progressively learning to compress longer teacher trajectories into a single step.}
\label{fig:consistency_conceptual_illustration}
\end{figure}

\paragraph{Sequential BOOT}
BOOT~\citep{gu2023boot} is a bootstrapping-based method that predicts the output of a pretrained diffusion model teacher given any time-step, i.e.,
\begin{equation}
\bm{A}^{\mathrm{st}}(\bm{z}_T, t) \approx \bm{z}_t, \quad \bm{z}_T\sim \mathcal{N}(\bm{0}, \bm{I})
\end{equation}
Our sequential variant modifies this idea by keeping the input fixed at the noisiest level \(\bm{z}_T\), while gradually shifting the target toward cleaner samples. Specifically, the student is trained to match
\begin{equation}
\bm{A}^{\mathrm{st}}(\bm{z}_T, T) \approx \mathcal{T}_k(\bm{z}_T) = \bm{z}_{T-k},
\end{equation}
with \(k = 1, 2, \dotsc, T\). See~\Cref{fig:boot_conceptual_illustration} for a conceptual illustration. This formulation constructs a monotonic training trajectory from high to low noise levels, offering a smooth learning process.

\begin{figure}[htb]
\centering
\includegraphics[width=0.9\linewidth]{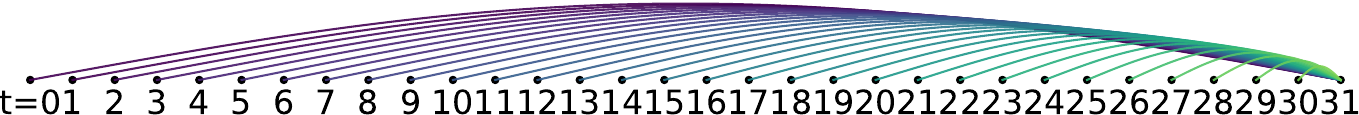}
\caption{Conceptual illustration of \emph{sequential BOOT distillation}. The student model is trained with input fixed at the noisiest level \(\bm{z}_T\), while the target is gradually shifted toward cleaner samples.}
\label{fig:boot_conceptual_illustration}
\end{figure}

%%%%%%%%%%%%%%%%%%%%%%%%%%%%%%%%%%%%%%%%%%%%%%%%%%%%%%%%%%%%%%%%%%%%%%%%%%%%%%%%
\section{Additional experimental results on Pareto dynamic programming}
\label{app:additional_experimental_results_on_pareto_dynamic_programming}

\rev{We first explain why the proposed Pareto-based dynamic programming algorithm also applies to a general centered Gaussian distribution with non-diagonal covariance. Let \(\mathcal N(0,\bm\Sigma)\) be such a distribution, and let \(\bm\Sigma = \bm U \bm\Lambda \bm U^\top\) be the eigendecomposition of \(\bm \Sigma\), where \(\bm U\) is orthogonal and \(\bm\Lambda\) is diagonal. Define the transformed coordinates \(\widetilde{\bm z}_t = \bm U^\top \bm z_t\). In this basis, the covariance becomes diagonal, and the optimal denoiser as well as the one-step DDIM operators become diagonal accordingly. In particular,
\begin{multline}
\bm A_t^\star = \bm U \widetilde{\bm A}_t^\star \bm U^\top, \\
\text{where }\widetilde{\bm A}_t^\star
=
(\alpha_{t-1}\alpha_t \bm\Lambda + \sigma_{t-1}\sigma_t \bm I)
(\alpha_t^2 \bm\Lambda + \sigma_t^2 \bm I)^{-1}.
\end{multline}
The same holds for the shrinkage matrices. We have \(\bm\Gamma_t = \bm U \widetilde{\bm\Gamma}_t \bm U^\top\), where \(\widetilde{\bm\Gamma}_t\) is diagonal in the same basis. Hence the recursive merge map is equivariant under the same orthogonal conjugation. Indeed,
\begin{equation}
\bm U^\top\bigl((\bm I-\bm\Gamma_{t_2})\bm A_L\bm A_R+\bm\Gamma_{t_2}\bm A_R\bigr)\bm U
=
(\bm I-\widetilde{\bm\Gamma}_{t_2})\widetilde{\bm A}_L\widetilde{\bm A}_R
+
\widetilde{\bm\Gamma}_{t_2}\widetilde{\bm A}_R,
\end{equation}
where \(\widetilde{\bm A}_L = \bm U^\top \bm A_L \bm U\) and \(\widetilde{\bm A}_R = \bm U^\top \bm A_R \bm U\). Therefore, every merge plan in the original coordinates corresponds bijectively to the same merge plan in the diagonal basis. Moreover, the surrogate objective is preserved under this orthogonal change of basis. If \(\widetilde{\mathcal T}_T\) denotes the surrogate target and \(\bm A\) is any candidate merged operator, then we have
\begin{equation}
\|\widetilde{\mathcal T}_T-\bm A\|_{\mathrm{F}}^2
=
\|\bm U^\top \widetilde{\mathcal T}_T \bm U-\bm U^\top \bm A \bm U\|_{\mathrm{F}}^2.
\end{equation}
Thus, a merge plan is optimal in the original coordinates if and only if its transformed counterpart is optimal in the diagonal basis. In this sense, the Pareto dynamic programming algorithm and its optimality result extend to general centered Gaussian distributions up to an orthogonal change of coordinates.}

We now present extended results on dynamic programming, complementing~\Cref{subsec:pareto_dynamic_programming_results}.

\paragraph{Dynamic programming results under different \(T\)}
We provide experimental results of the dynamic programming algorithm under varying values of \(T\), as shown in~\Cref{fig:dp_more_results_T_64_s_64,fig:dp_more_results_T_128_s_64,fig:dp_more_results_T_256_s_64,fig:dp_more_results_T_512_s_64}, complementing the base case \(T = 32\) discussed in~\Cref{subsec:pareto_dynamic_programming_results}. Each figure contains two panels: the top panel reports the squared \(2\)-Wasserstein distance between the student operator and the surrogate teacher under four canonical strategies, along with the DP-optimal solution across varying \(\lambda\); the bottom panel shows the corresponding error gap between each strategy and the DP baseline. Across all tested settings, we consistently observe that when \(\lambda \leq 1\), sequential BOOT incurs zero error and exactly recovers the optimal merge plan. In contrast, for sufficiently large \(\lambda\), vanilla trajectory distillation exactly matches the performance of the DP-optimal strategy. These trends validate the theoretical predictions of~\Cref{thm:vanilla_distillation_optimality,thm:sequential_boot_optimality}.

\begin{figure}[htb]
\centering
\includegraphics[width=1\linewidth]{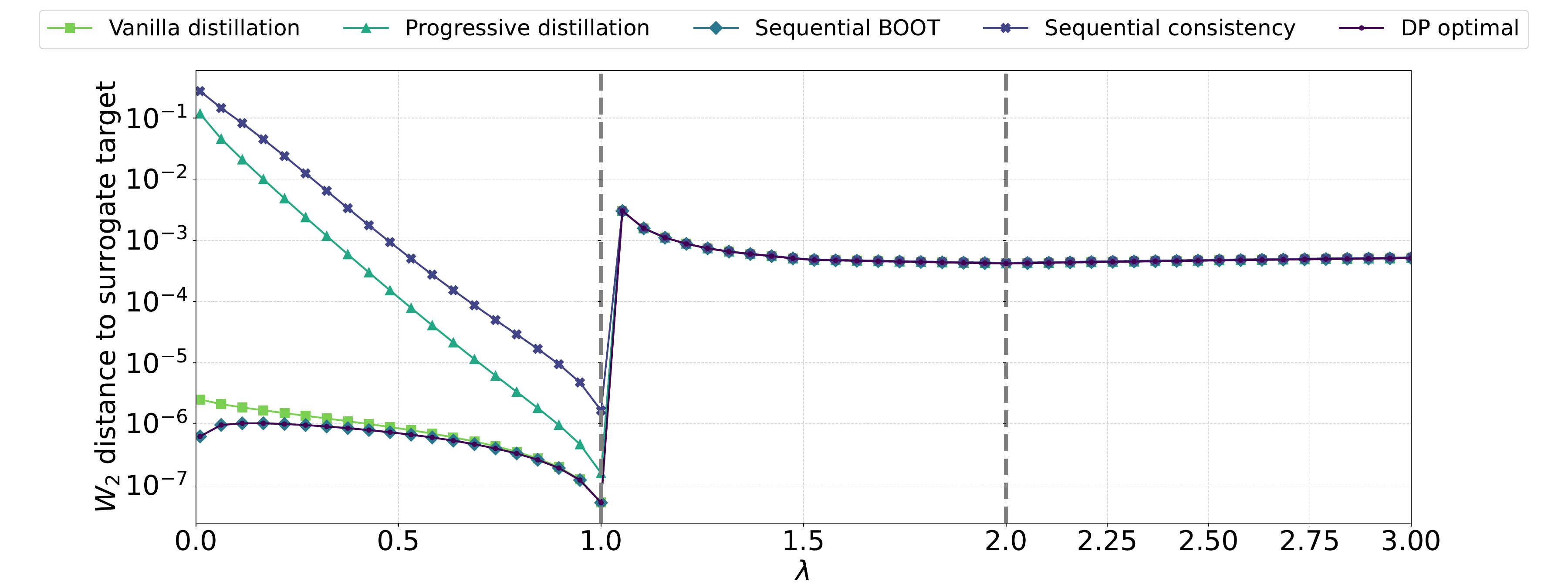}\\
\includegraphics[width=1\linewidth]{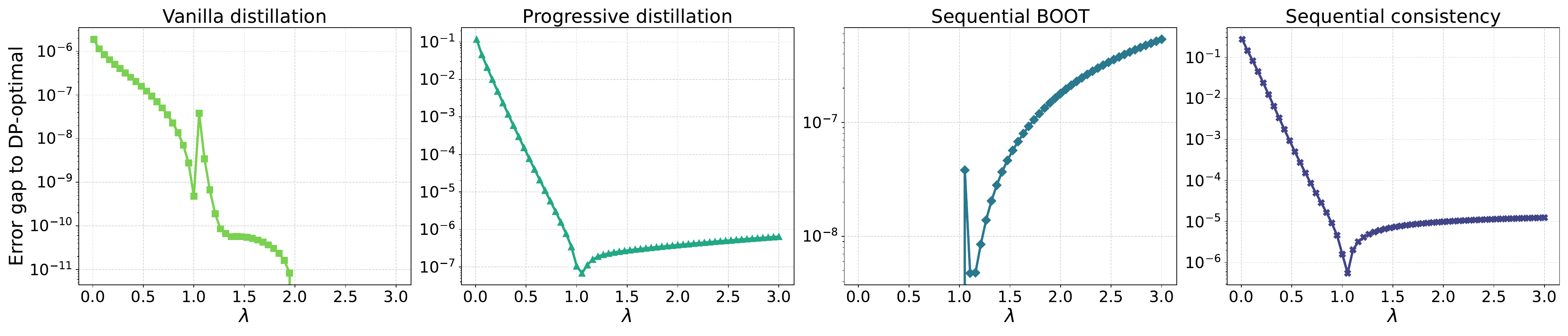}\\
\caption{Evaluation of the merging strategies for \(T = 64\) and \(s = 6.4\). \textbf{Top}: squared \(2\)-Wasserstein distance between the student operator and the surrogate teacher operator under four canonical strategies (vanilla, progressive, sequential BOOT, and sequential consistency), together with the DP-optimal solution, across varying values of \(\lambda\). \textbf{Bottom}: Error gap between each strategy and the DP-optimal baseline.}
\label{fig:dp_more_results_T_64_s_64}
\end{figure}

\begin{figure}[H]
\centering
\includegraphics[width=1\linewidth]{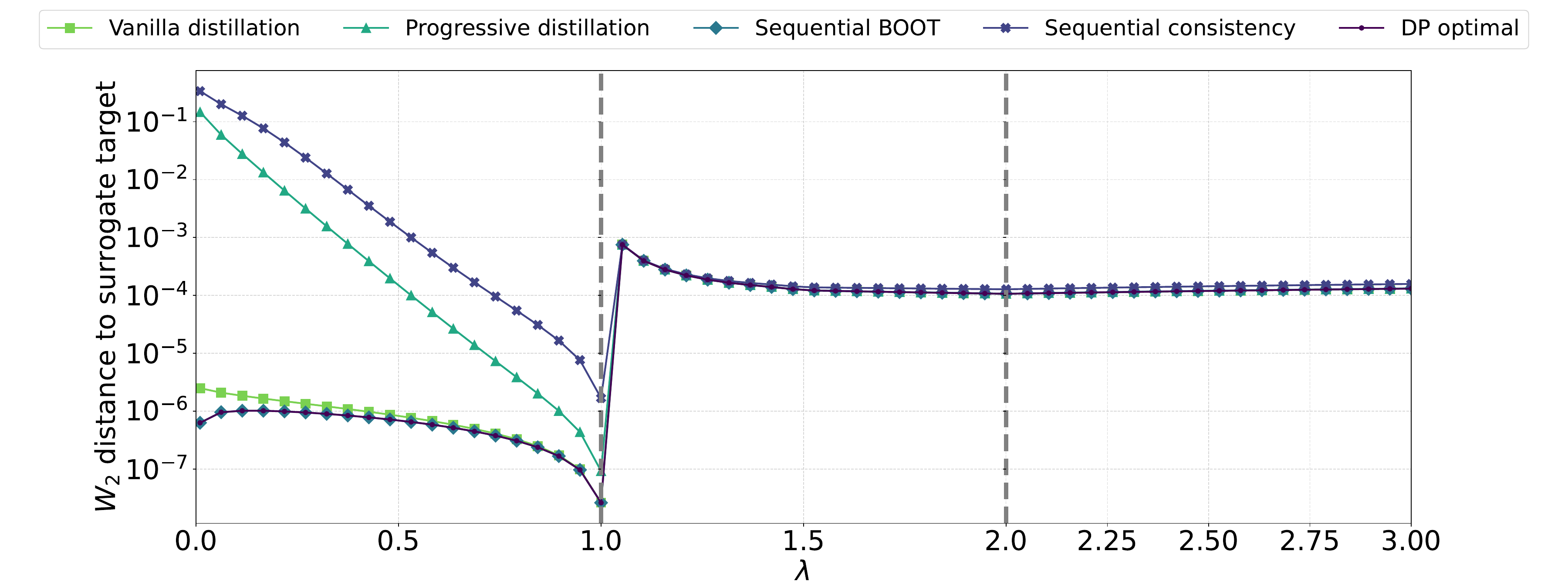}\\
\includegraphics[width=1\linewidth]{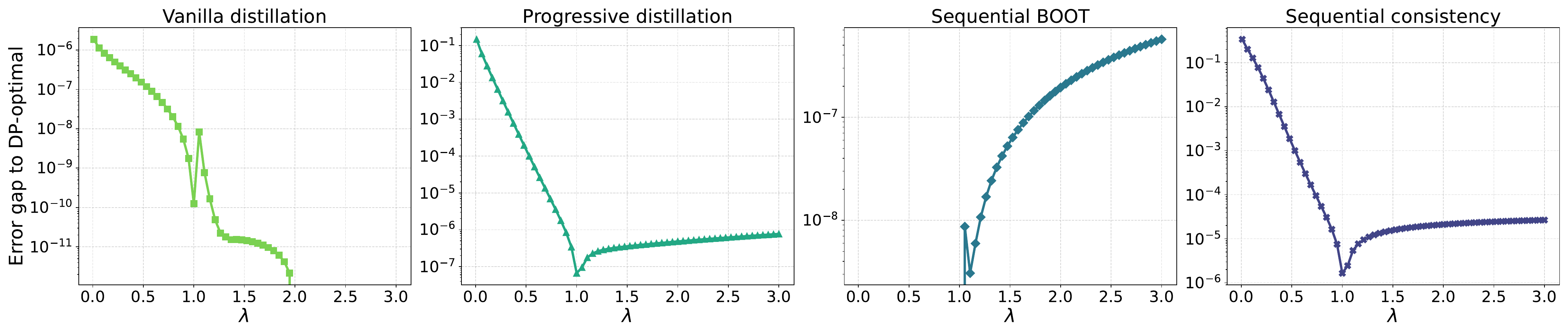}\\
\caption{Evaluation of the merging strategies for \(T = 128\) and \(s = 6.4\). This figure follows the same format as~\Cref{fig:dp_more_results_T_64_s_64}.}
\label{fig:dp_more_results_T_128_s_64}
\end{figure}

\begin{figure}[htb]
\centering
\includegraphics[width=1\linewidth]{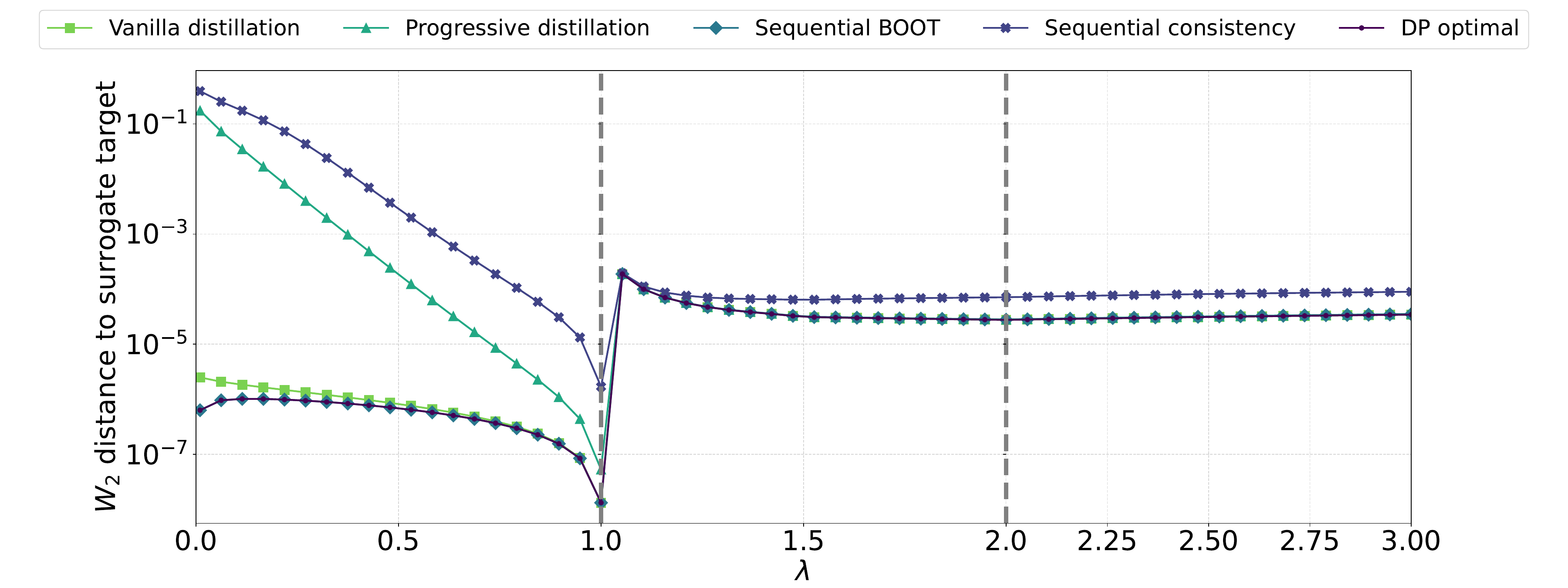}\\
\includegraphics[width=1\linewidth]{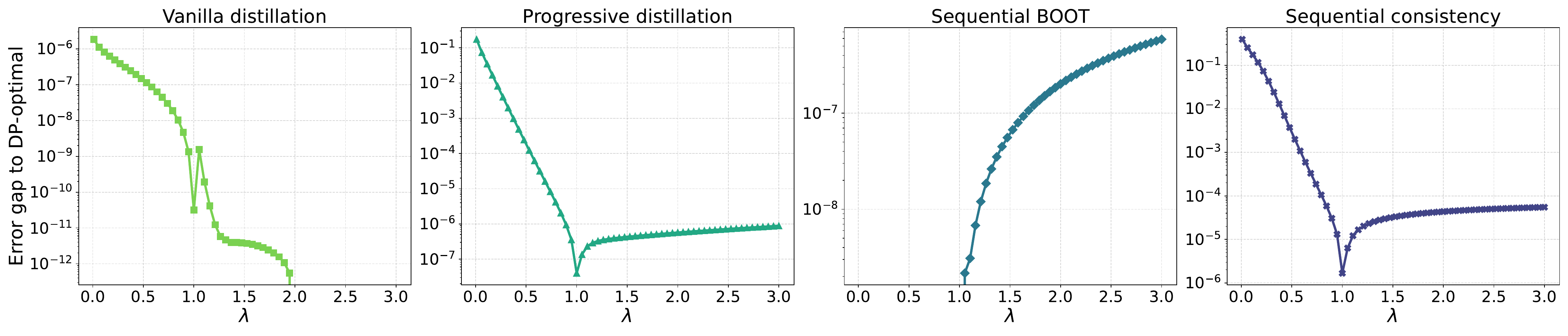}\\
\caption{Evaluation of the merging strategies for \(T = 256\) and \(s = 6.4\). This figure follows the same format as~\Cref{fig:dp_more_results_T_64_s_64}.}
\label{fig:dp_more_results_T_256_s_64}
\end{figure}

\begin{figure}[htb]
\centering
\includegraphics[width=1\linewidth]{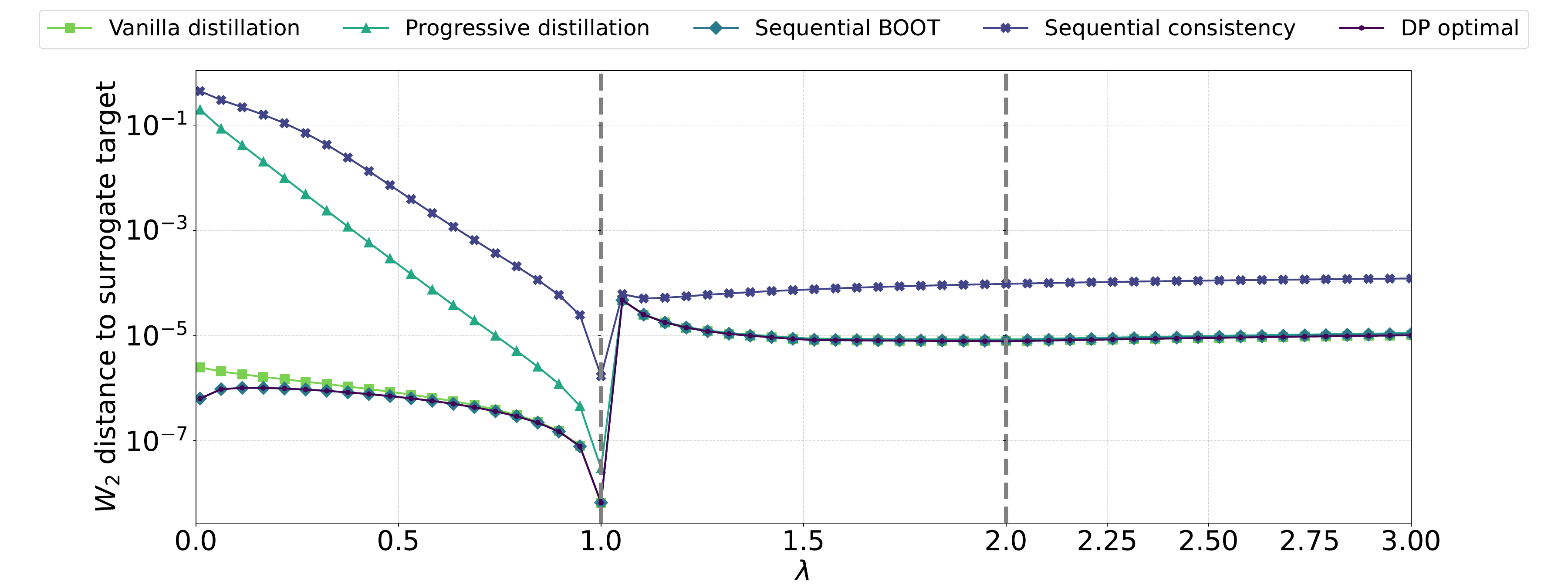}\\
\includegraphics[width=1\linewidth]{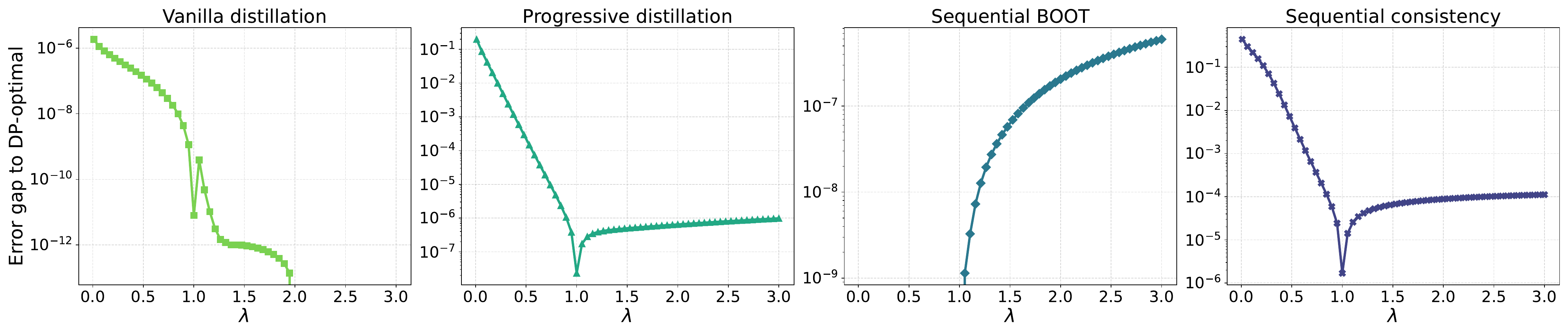}\\
\caption{Evaluation of the merging strategies for \(T = 512\) and \(s = 6.4\). This figure follows the same format as~\Cref{fig:dp_more_results_T_64_s_64}.}
\label{fig:dp_more_results_T_512_s_64}
\end{figure}

\paragraph{Dynamic programming results under different \(s\)}
We present experimental results of the dynamic programming algorithm under different values of \(s\), as shown in~\Cref{fig:dp_more_results_T_32_s_16,fig:dp_more_results_T_32_s_32}, complementing the special case \(s = 6.4\) discussed in~\Cref{subsec:pareto_dynamic_programming_results}. Each figure contains two panels: the top panel shows the squared \(2\)-Wasserstein distance between the student operator and the surrogate teacher under four canonical strategies and the DP-optimal solution across varying \(\lambda\); the bottom panel reports the corresponding error gap between each strategy and the DP baseline. Across all tested values of \(s\), we observe consistent qualitative trends as in the base setting. Notably, as \(s\) increases, the overall error systematically decreases across all strategies, indicating that extended training time per merge step improves the student's ability to approximate the teacher operator.

\begin{figure}[htb]
\centering
\includegraphics[width=1\linewidth]{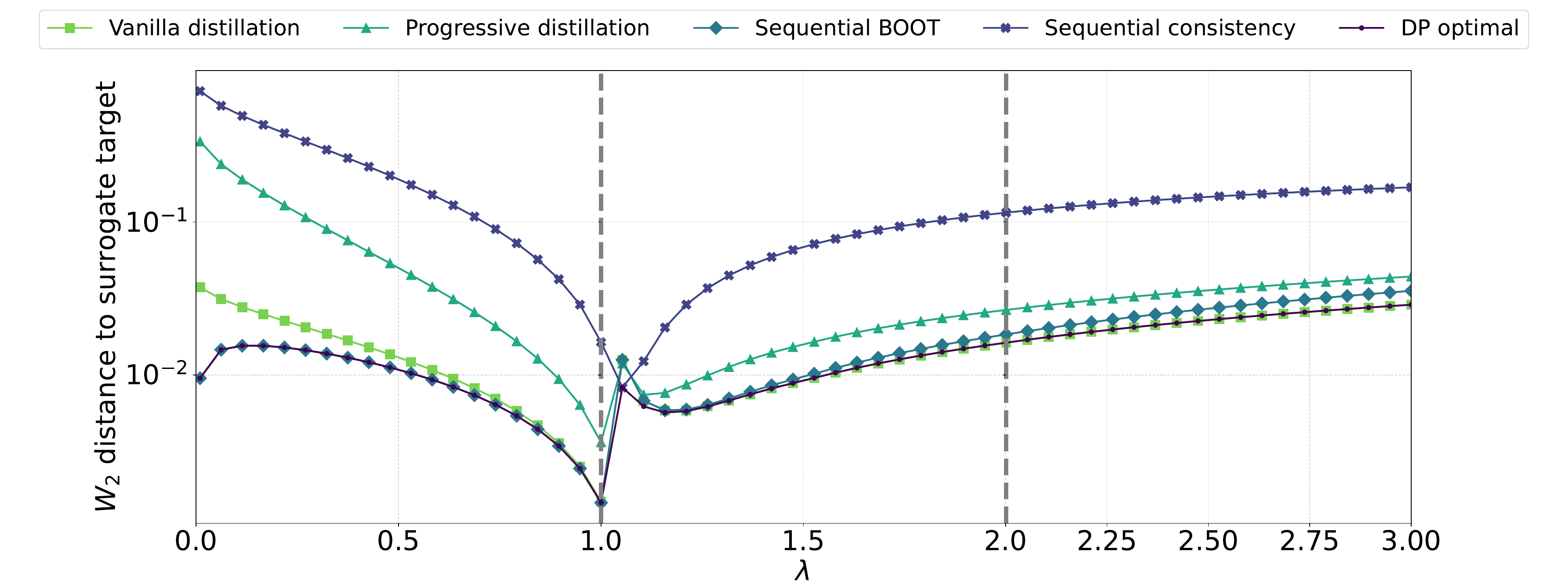}\\
\includegraphics[width=1\linewidth]{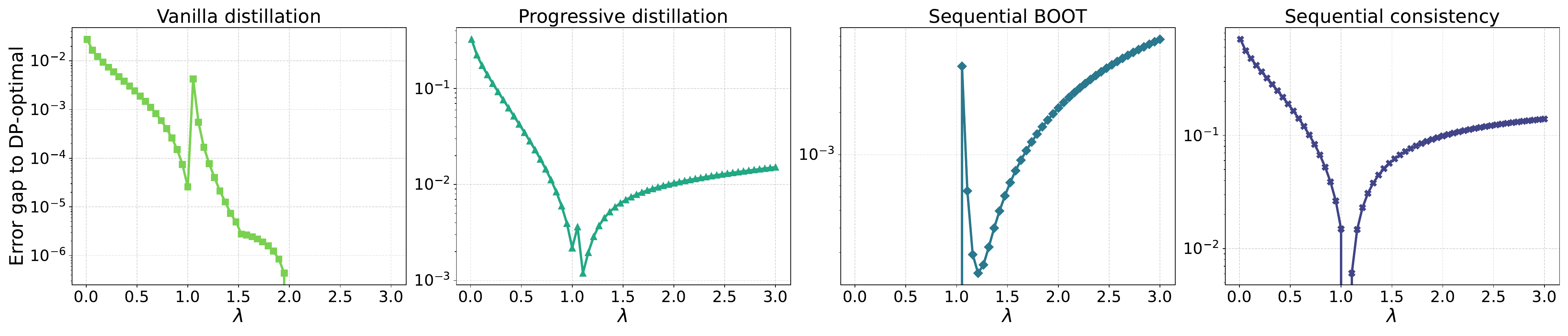}\\
\caption{Evaluation of the merging strategies for \(T = 32\) and \(s = 1.6\). This figure follows the same format as~\Cref{fig:dp_more_results_T_64_s_64}.}
\label{fig:dp_more_results_T_32_s_16}
\end{figure}

\begin{figure}[htb]
\centering
\includegraphics[width=1\linewidth]{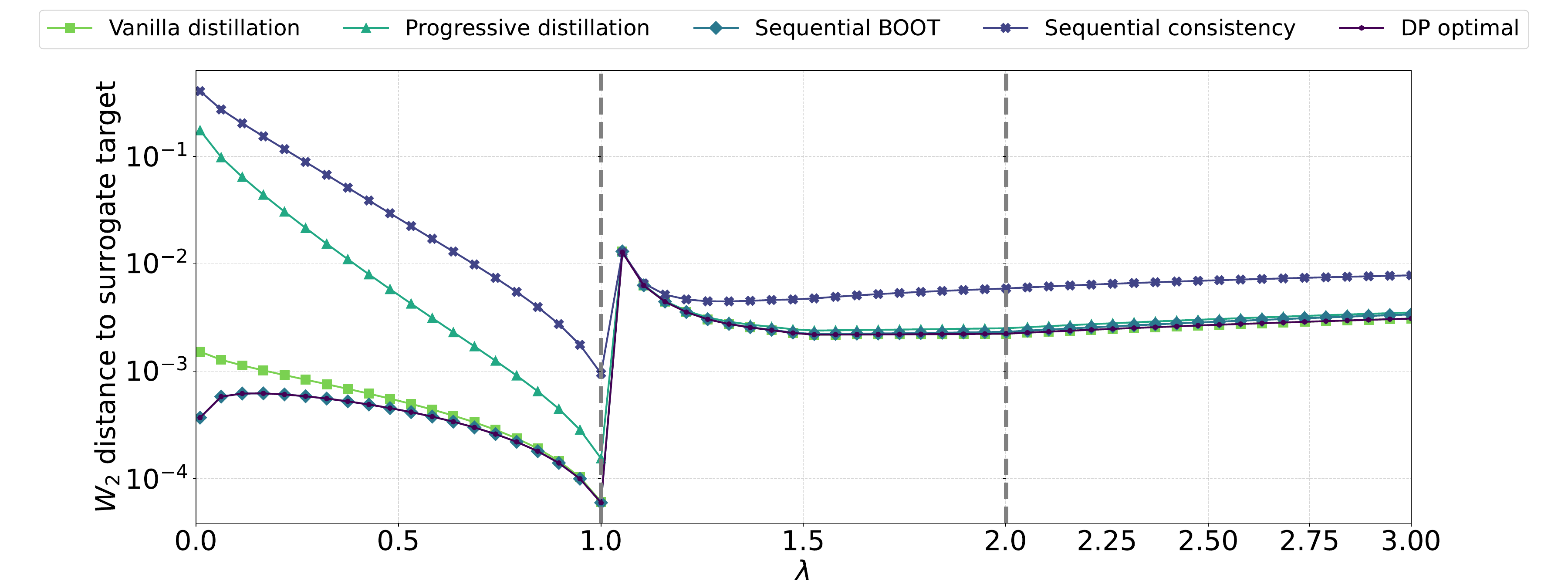}\\
\includegraphics[width=1\linewidth]{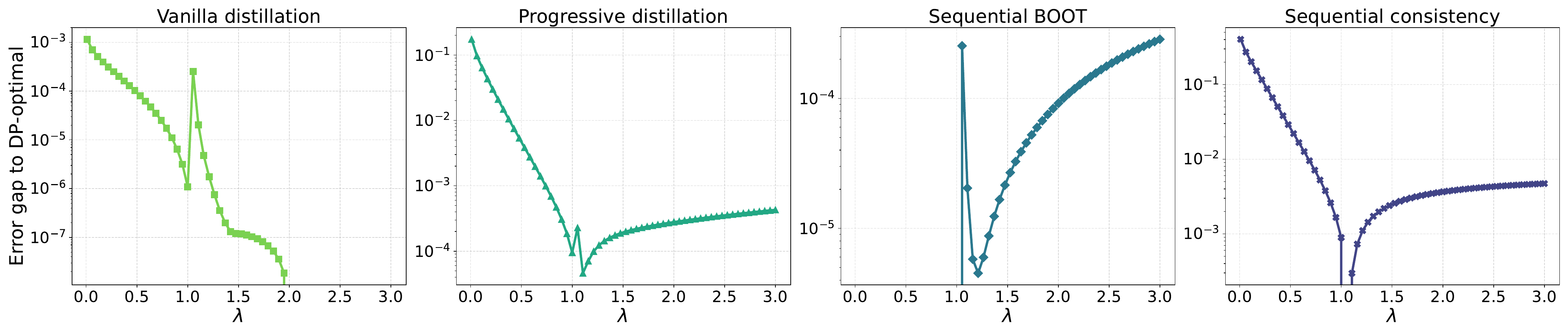}\\
\caption{Evaluation of the merging strategies for \(T = 32\) and \(s = 3.2\). This figure follows the same format as~\Cref{fig:dp_more_results_T_64_s_64}.}
\label{fig:dp_more_results_T_32_s_32}
\end{figure}

\paragraph{Dynamic programming results under different \(\bm{\Lambda}\)}
We further present experimental results of the dynamic programming algorithm under a range of non-scalar covariance matrices \(\bm{\Lambda} = \mathrm{Diag}(\lambda_1, \lambda_2)\), as illustrated in~\Cref{fig:dp_more_results_on_Lambda}. The specific choices of \((\lambda_1, \lambda_2)\) used in these experiments are summarized in~\Cref{tab:lambda_critical_choices}. These results complement the scalar case discussed in~\Cref{subsec:pareto_dynamic_programming_results}. In the examples shown, we focus on ``transitional'' regimes, where the diagonal entries of \(\bm{\Lambda}\) are neither all smaller than or equal to \(1\) nor all significantly larger. These boundary regimes are already well-understood: the optimal strategy reduces to either sequential BOOT distillation or vanilla distillation, as established in the scalar setting. In contrast, as observed in~\Cref{fig:dp_more_results_on_Lambda}, the transitional cases often give rise to highly intricate merging strategies that deviate from any canonical pattern.

\begin{table}[htb]
\centering
\caption{Transitional-regime covariance choices used in our experiments. Each pair satisfies: (\romannumeral1) \(\lambda_1 < \lambda_2\), (\romannumeral2) at least one entry is near \(1\), and (\romannumeral3) both are not simultaneously \(\leq 1\). The table is sorted by \(\lambda_1\).}
\label{tab:lambda_critical_choices}
\smallskip
\scriptsize
\begin{tabular}{c@{\hspace{2mm}}|@{\hspace{2mm}}c@{\hspace{2mm}}c@{\hspace{2mm}}c@{\hspace{2mm}}c@{\hspace{2mm}}c@{\hspace{2mm}}c@{\hspace{2mm}}c@{\hspace{2mm}}c@{\hspace{2mm}}c@{\hspace{2mm}}c}
\toprule
Index & 1 & 2 & 3 & 4 & 5 & 6 & 7 & 8 & 9 & 10 \\
\midrule
\(\lambda_1\) & \(0.95\) & \(0.95\) & \(0.97\) & \(0.98\) & \(0.99\) & \(1.00\) & \(1.02\) & \(1.02\) & \(1.05\) & \(1.08\) \\
\(\lambda_2\) & \(1.05\) & \(1.25\) & \(1.10\) & \(1.30\) & \(1.40\) & \(1.20\) & \(1.35\) & \(1.80\) & \(1.50\) & \(1.60\) \\
\bottomrule
\end{tabular}
\end{table}

\begin{figure}[htb]
\centering
\includegraphics[width=0.9\linewidth]{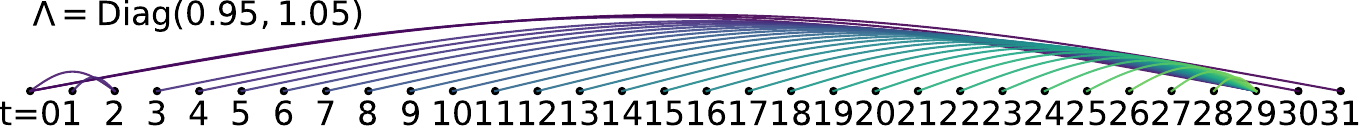}\\[1em]
\includegraphics[width=0.9\linewidth]{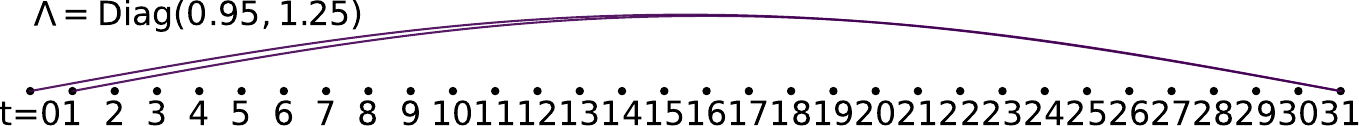}\\[1em]
\includegraphics[width=0.9\linewidth]{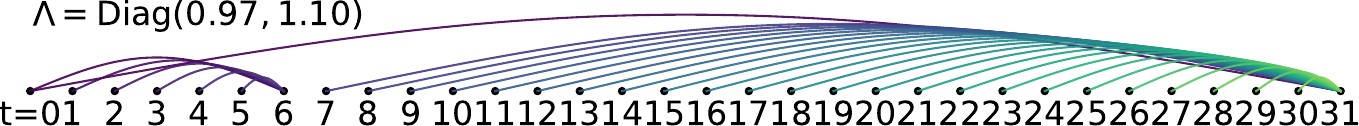}\\[1em]
\includegraphics[width=0.9\linewidth]{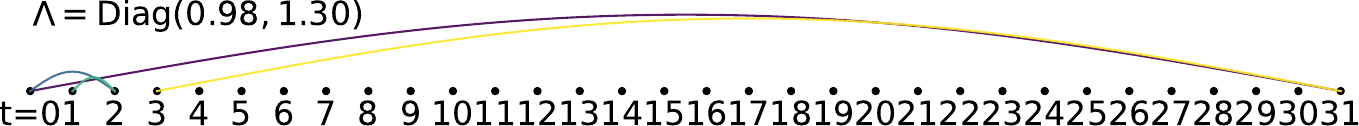}\\[1em]
\includegraphics[width=0.9\linewidth]{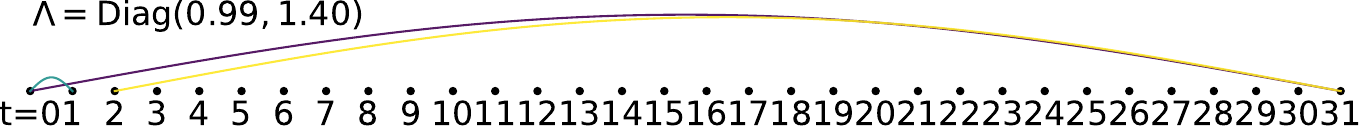}\\[1em]
\includegraphics[width=0.9\linewidth]{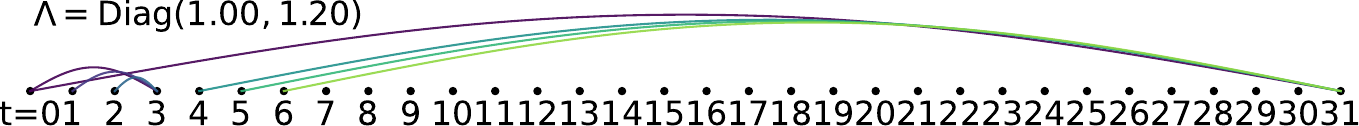}\\[1em]
\includegraphics[width=0.9\linewidth]{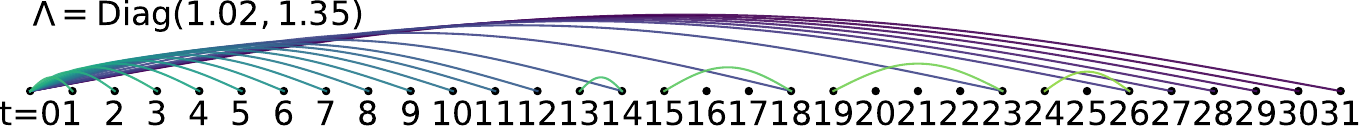}\\[1em]
\includegraphics[width=0.9\linewidth]{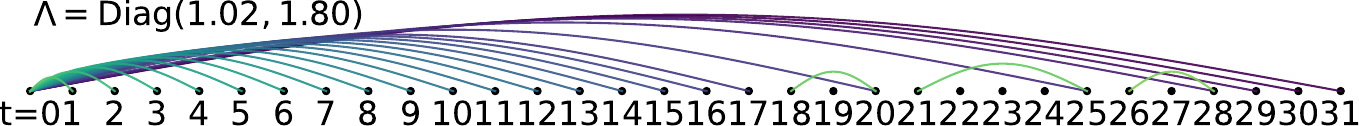}\\[1em]
\includegraphics[width=0.9\linewidth]{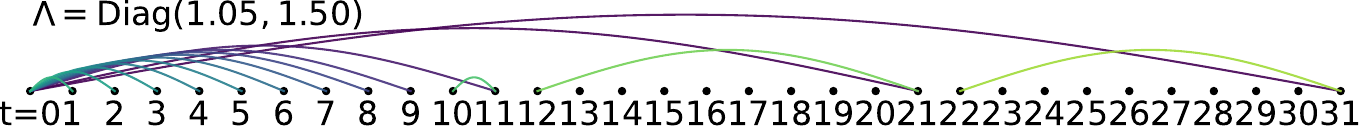}\\[1em]
\includegraphics[width=0.9\linewidth]{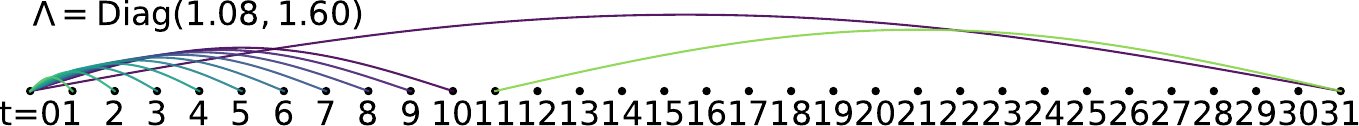}
\caption{Pareto dynamic programming merge plans for a collection of transitional-regime covariance settings. Each subplot corresponds to a distinct \(\bm{\Lambda} = \mathrm{Diag}(\lambda_1, \lambda_2)\) configuration, with the specific values labeled at the top-left corner of each plot. All choices of \((\lambda_1, \lambda_2)\) are summarized in~\Cref{tab:lambda_critical_choices}, and are selected to probe transitional regimes where optimal strategies deviate from canonical forms such as vanilla or sequential BOOT. The arcs indicate the merge operations selected by the dynamic programming algorithm, where color encodes the merge order. Lighter arcs correspond to earlier merges, and darker arcs to later ones. These results demonstrate that optimal merge structures in the transitional regime are highly nontrivial and exhibit intrinsic patterns that cannot be captured by the canonical strategies.}
\label{fig:dp_more_results_on_Lambda}
\end{figure}

%%%%%%%%%%%%%%%%%%%%%%%%%%%%%%%%%%%%%%%%%%%%%%%%%%%%%%%%%%%%%%%%%%%%%%%%%%%%%%%%
\section{Additional results on synthetic datasets}
\label{app:additional_results_on_synthetic_datasets}

In this section, we provide further experiments on synthetic datasets, complementing~\Cref{subsec:experiments_on_synthetic_datasets}.

%%%%%%%%%%%%%%%%%%%%%%%%%%%%%%%%%%%%%%%%%%%%%%%%%%%%%%%%%%%%
\subsection{Gaussian dataset}
\label{subsec:gaussian_dataset}

\begin{table}[tb]
\centering
\caption{Simulation results of four canonical trajectory distillation methods on synthetic datasets with varying total time steps \(T\) and covariance parameter \(\lambda\). Each entry reports the mean signed error (no absolute value taken) between the student operator and the surrogate composite operator \(\widetilde{\mathcal{T}}_T\), averaged over \(10\) independent trials. The best result for each setting is underlined. Results are generally consistent with our theoretical predictions in the linear regime.}
\label{tab:synthetic_all_T}
\smallskip
\tiny
\setlength{\tabcolsep}{2pt}
\begin{tabular}{lcccc}
\toprule
\multicolumn{5}{c}{\scriptsize \(T=64\)} \\
\scriptsize \(\lambda\) & \scriptsize Vanilla & \scriptsize Progressive & \scriptsize BOOT & \scriptsize Consistency \\
\midrule
0.20 & \(8.3 \times 10^{-4} \pm 1.1 \times 10^{-4}\) & \(4.5 \times 10^{-2} \pm 5.6 \times 10^{-4}\) & \(\underline{4.9 \times 10^{-4}} \pm 4.8 \times 10^{-5}\) & \(2.0 \times 10^{-1} \pm 1.1 \times 10^{-2}\) \\ 
0.50 & \(4.4 \times 10^{-4} \pm 9.7 \times 10^{-5}\) & \(5.9 \times 10^{-3} \pm 6.7 \times 10^{-5}\) & \(\underline{3.2 \times 10^{-4}} \pm 3.7 \times 10^{-5}\) & \(4.8 \times 10^{-2} \pm 4.5 \times 10^{-3}\) \\ 
1.00 & \(\underline{2.6 \times 10^{-5}} \pm 3.0 \times 10^{-6}\) & \(7.8 \times 10^{-5} \pm 4.1 \times 10^{-6}\) & \(2.8 \times 10^{-5} \pm 4.8 \times 10^{-6}\) & \(6.0 \times 10^{-4} \pm 2.0 \times 10^{-4}\) \\ 
1.02 & \(-9.3 \times 10^{-3} \pm 2.6 \times 10^{-6}\) & \(-9.3 \times 10^{-3} \pm 1.7 \times 10^{-6}\) & \(-9.3 \times 10^{-3} \pm 2.5 \times 10^{-6}\) & \(\underline{-9.2 \times 10^{-3}} \pm 1.0 \times 10^{-4}\) \\ 
2.00 & \(\underline{-8.9 \times 10^{-4}} \pm 6.8 \times 10^{-5}\) & \(-1.3 \times 10^{-3} \pm 6.8 \times 10^{-5}\) & \(-1.0 \times 10^{-3} \pm 1.2 \times 10^{-4}\) & \(-8.4 \times 10^{-3} \pm 1.8 \times 10^{-3}\) \\ 
5.00 & \(\underline{-2.3 \times 10^{-3}} \pm 2.0 \times 10^{-4}\) & \(-3.0 \times 10^{-3} \pm 2.8 \times 10^{-4}\) & \(-3.0 \times 10^{-3} \pm 3.7 \times 10^{-4}\) & \(-1.3 \times 10^{-2} \pm 2.5 \times 10^{-3}\) \\ 
\midrule
\multicolumn{5}{c}{\scriptsize \(T=128\)} \\
\scriptsize \(\lambda\) & \scriptsize Vanilla & \scriptsize Progressive & \scriptsize BOOT & \scriptsize Consistency \\
\midrule
0.20 & \(7.7 \times 10^{-4} \pm 1.8 \times 10^{-4}\) & \(5.6 \times 10^{-2} \pm 6.0 \times 10^{-4}\) & \(\underline{5.5 \times 10^{-4}} \pm 8.0 \times 10^{-5}\) & \(2.7 \times 10^{-1} \pm 1.3 \times 10^{-2}\) \\ 
0.50 & \(3.7 \times 10^{-4} \pm 5.6 \times 10^{-5}\) & \(7.3 \times 10^{-3} \pm 1.3 \times 10^{-4}\) & \(\underline{3.3 \times 10^{-4}} \pm 4.6 \times 10^{-5}\) & \(7.4 \times 10^{-2} \pm 7.7 \times 10^{-3}\) \\ 
1.00 & \(\underline{1.3 \times 10^{-5}} \pm 3.1 \times 10^{-6}\) & \(4.5 \times 10^{-5} \pm 1.8 \times 10^{-6}\) & \(1.4 \times 10^{-5} \pm 2.1 \times 10^{-6}\) & \(7.7 \times 10^{-4} \pm 9.2 \times 10^{-5}\) \\ 
1.02 & \(\underline{-2.0 \times 10^{-3}} \pm 8.5 \times 10^{-8}\) & \(-2.0 \times 10^{-3} \pm 4.5 \times 10^{-7}\) & \(-2.0 \times 10^{-3} \pm 6.2 \times 10^{-8}\) & \(-2.1 \times 10^{-3} \pm 3.3 \times 10^{-6}\) \\ 
2.00 & \(\underline{-6.2 \times 10^{-4}} \pm 9.2 \times 10^{-5}\) & \(-1.0 \times 10^{-3} \pm 7.9 \times 10^{-5}\) & \(-7.6 \times 10^{-4} \pm 9.0 \times 10^{-5}\) & \(-1.5 \times 10^{-2} \pm 2.7 \times 10^{-3}\) \\ 
5.00 & \(\underline{-1.8 \times 10^{-3}} \pm 3.5 \times 10^{-4}\) & \(-2.8 \times 10^{-3} \pm 2.3 \times 10^{-4}\) & \(-2.3 \times 10^{-3} \pm 4.3 \times 10^{-4}\) & \(-2.3 \times 10^{-2} \pm 4.4 \times 10^{-3}\) \\ 
\midrule
\multicolumn{5}{c}{\scriptsize \(T=256\)} \\
\scriptsize \(\lambda\) & \scriptsize Vanilla & \scriptsize Progressive & \scriptsize BOOT & \scriptsize Consistency \\
\midrule
0.20 & \(7.3 \times 10^{-4} \pm 1.1 \times 10^{-4}\) & \(6.7 \times 10^{-2} \pm 7.6 \times 10^{-4}\) & \(\underline{4.8 \times 10^{-4}} \pm 6.8 \times 10^{-5}\) & \(3.3 \times 10^{-1} \pm 1.1 \times 10^{-2}\) \\ 
0.50 & \(4.3 \times 10^{-4} \pm 7.5 \times 10^{-5}\) & \(8.6 \times 10^{-3} \pm 1.1 \times 10^{-4}\) & \(\underline{3.2 \times 10^{-4}} \pm 5.6 \times 10^{-5}\) & \(1.2 \times 10^{-1} \pm 9.2 \times 10^{-3}\) \\ 
1.00 & \(\underline{6.6 \times 10^{-6}} \pm 1.2 \times 10^{-6}\) & \(2.6 \times 10^{-5} \pm 1.2 \times 10^{-6}\) & \(7.0 \times 10^{-6} \pm 1.3 \times 10^{-6}\) & \(7.6 \times 10^{-4} \pm 1.0 \times 10^{-4}\) \\ 
1.02 & \(-4.9 \times 10^{-4} \pm 1.6 \times 10^{-6}\) & \(-5.1 \times 10^{-4} \pm 1.1 \times 10^{-6}\) & \(\underline{-4.9 \times 10^{-4}} \pm 9.5 \times 10^{-7}\) & \(-1.3 \times 10^{-3} \pm 1.3 \times 10^{-4}\) \\ 
2.00 & \(\underline{-5.4 \times 10^{-4}} \pm 1.1 \times 10^{-4}\) & \(-1.0 \times 10^{-3} \pm 7.1 \times 10^{-5}\) & \(-6.5 \times 10^{-4} \pm 1.0 \times 10^{-4}\) & \(-2.7 \times 10^{-2} \pm 3.9 \times 10^{-3}\) \\ 
5.00 & \(\underline{-1.6 \times 10^{-3}} \pm 3.4 \times 10^{-4}\) & \(-2.8 \times 10^{-3} \pm 1.6 \times 10^{-4}\) & \(-2.7 \times 10^{-3} \pm 5.2 \times 10^{-4}\) & \(-4.3 \times 10^{-2} \pm 7.6 \times 10^{-3}\) \\ 
\midrule
\multicolumn{5}{c}{\scriptsize \(T=512\)} \\
\scriptsize \(\lambda\) & \scriptsize Vanilla & \scriptsize Progressive & \scriptsize BOOT & \scriptsize Consistency \\
\midrule
0.20 & \(7.0 \times 10^{-4} \pm 1.6 \times 10^{-4}\) & \(7.9 \times 10^{-2} \pm 4.9 \times 10^{-4}\) & \(\underline{5.0 \times 10^{-4}} \pm 1.0 \times 10^{-4}\) & \(4.0 \times 10^{-1} \pm 1.3 \times 10^{-2}\) \\ 
0.50 & \(4.3 \times 10^{-4} \pm 8.7 \times 10^{-5}\) & \(1.0 \times 10^{-2} \pm 9.9 \times 10^{-5}\) & \(\underline{3.4 \times 10^{-4}} \pm 6.0 \times 10^{-5}\) & \(1.7 \times 10^{-1} \pm 1.4 \times 10^{-2}\) \\ 
1.00 & \(\underline{3.0 \times 10^{-6}} \pm 4.4 \times 10^{-7}\) & \(1.5 \times 10^{-5} \pm 3.2 \times 10^{-7}\) & \(3.4 \times 10^{-6} \pm 6.2 \times 10^{-7}\) & \(7.1 \times 10^{-4} \pm 1.1 \times 10^{-4}\) \\ 
1.02 & \(-1.3 \times 10^{-4} \pm 2.4 \times 10^{-6}\) & \(-1.6 \times 10^{-4} \pm 1.2 \times 10^{-6}\) & \(\underline{-1.3 \times 10^{-4}} \pm 1.7 \times 10^{-6}\) & \(-2.2 \times 10^{-3} \pm 2.9 \times 10^{-4}\) \\ 
2.00 & \(\underline{-5.9 \times 10^{-4}} \pm 8.5 \times 10^{-5}\) & \(-1.1 \times 10^{-3} \pm 7.1 \times 10^{-5}\) & \(-6.7 \times 10^{-4} \pm 1.1 \times 10^{-4}\) & \(-5.2 \times 10^{-2} \pm 6.2 \times 10^{-3}\) \\ 
5.00 & \(\underline{-1.8 \times 10^{-3}} \pm 2.8 \times 10^{-4}\) & \(-2.9 \times 10^{-3} \pm 3.7 \times 10^{-4}\) & \(-2.4 \times 10^{-3} \pm 6.4 \times 10^{-4}\) & \(-9.2 \times 10^{-2} \pm 1.3 \times 10^{-2}\) \\ 
\bottomrule
\end{tabular}
\end{table}

\paragraph{Simulation results under different \(T\)}
We evaluate the performance of four canonical trajectory distillation strategies across a range of covariance parameters \(\lambda\) and trajectory lengths \(T\). All student models are implemented as sinusoidal embedding-based diagonal affine operators. For each time step \(t\), the student computes a time-dependent diagonal scaling operator
\begin{equation}
\mathrm{Diag}\bigl(\bm{\theta}_t^\top \mathrm{emb}(t/T)\bigr),
\end{equation}
where \(\bm{\theta}_t \in \mathbb{R}^{d_{\text{emb}} \times d}\) is a learnable parameter tensor specific to time step \(t\), and \(\mathrm{emb}(t/T) \in \mathbb{R}^{d_{\text{emb}}}\) is the sinusoidal time embedding applied to the normalized timestep. The embedding is defined by
\begin{multline}
\mathrm{emb}(t/T) = \bigl(\sin(\omega_1 t/T), \cos(\omega_1 t/T), \dotsc, \\\sin(\omega_{d_{\text{emb}}/2} t/T), \cos(\omega_{d_{\text{emb}}/2} t/T)\bigr),
\end{multline}
with frequencies \(\omega_i = \frac{1}{10000^{2i/d_{\text{emb}}}}\) for \(i = 1, \dotsc, d_{\text{emb}}/2\). We set the embedding dimension to \(d_{\text{emb}} = 64\). This design enables each time step to have a distinct, smoothly varying weight. All student weights are initialized using a closed-form least-squares projection of the corresponding teacher operator. Training is performed for \(100\) epochs with a batch size of \(2560\), using SGD with learning rate \(10^{-3}\). Each experiment is repeated independently for each \(\lambda\), and the results are averaged over \(10\) trials to report the mean and standard deviation. The complete set of results across different values of \(T\) and \(\lambda\) is summarized in~\Cref{tab:synthetic_all_T}.

%%%%%%%%%%%%%%%%%%%%%%%%%%%%%%%%%%%%%%%%%%%%%%%%%%%%%%%%%%%%
\subsection{Gaussian mixtures dataset}
\label{subsec:gaussian_mixtures_dataset}

To validate our analysis of operator merging in multimodal regimes (\Cref{sec:approximation_error_and_error_propagation_in_the_gaussian_mixture_regime}), we conduct experiments on a 2D Gaussian mixture model. The ground truth distribution \(p_0\) is an isotropic mixture of \(K=8\) Gaussian modes with means \(\{\bm{\mu}_k\}_{k=1}^K\) uniformly spaced on a circle of radius \(R=5.0\) and covariance \(\bm{\Sigma}_k = 0.3^2 \bm{I}\). We use a cosine noise schedule~\citep{nichol2021improved} with \(T=32\) steps.

To isolate distillation errors, we use an analytical teacher based on the closed-form optimal denoising estimator derived in~\Cref{prop:optimal_mixture_denoising_estimator}. The student model comprises two sub-networks: a weighting network and an expert hypernetwork, both implemented as MLPs with a single hidden layer of \(64\) units and ReLU activations. The weighting network determines the mixing weights by processing the concatenation of the \(2\)-dimensional noisy state and the \(64\)-dimensional sinusoidal time embedding. It maps this \(66\)-dimensional combined input to the hidden layer, and subsequently to \(8\) output units which are normalized via a Softmax function to produce the expert probabilities. The expert hypernetwork predicts the parameters of the affine transformations conditioned solely on the \(64\)-dimensional time embedding, ensuring that the local linearity of each expert evolves only with time rather than the state. It maps the time embedding to the hidden layer and then to a \(48\)-dimensional output vector. This output is reshaped to form \(8\) separate \(2\times2\) weight matrices and \(8\) length-\(2\) bias vectors, which are applied as residual updates to fixed parameters initialized as identity matrices and zero vectors, respectively.

All student models are initialized by training to match the single-step teacher operator. Optimization uses Adam with a learning rate of \(10^{-3}\) and batch size \(2{,}560\). Each distillation stage is trained for \(10{,}000\) steps.

%%%%%%%%%%%%%%%%%%%%%%%%%%%%%%%%%%%%%%%%%%%%%%%%%%%%%%%%%%%%
\section{Additional results on real dataset}
\label{app:additional_results_on_real_dataset}

This section provides further experiments on the real dataset, complementing~\Cref{subsec:experiments_on_real_dataset}.

%%%%%%%%%%%%%%%%%%%%%%%%%%%%%%%%%%%%%%%%%%%%%%%%%%%%%%%%%%%%
\subsection{Additional results on CelebA}
\label{subsec:additioanl_results_on_celeba}

\paragraph{Experimental setup}
We first pretrain a \emph{Multiscale Structural Similarity Variational Autoencoder} (MSSIMVAE) model~\citep{snell2017learning} following standard settings from publicly available implementations. MSSIMVAE is a variant of the standard Variational Autoencoder (VAE) that replaces the typical pixel-wise reconstruction loss with a multiscale structural similarity (MS-SSIM) objective, encouraging perceptually meaningful reconstructions and improving latent structure.

We use the CelebA dataset~\citep{liu2015deep}, which contains \(202{,}599\) face images, as the training data for the MSSIMVAE. After training, we extract the latent codes of the entire CelebA dataset, resulting in \(202{,}599\) latent vectors, each of dimension \(128\). We then compute the empirical mean and empirical covariance matrix of these latent codes.

Interestingly, we observe that the sample covariance matrix is approximately diagonal, as shown in~\Cref{fig:latent_cov_matrix}, indicating that the learned latent space exhibits near-independence across dimensions. However, the diagonal entries are not uniformly equal to \(1\), implying that the latent distribution is not exactly standard Gaussian \(\mathcal{N}(\bm{0}, \bm{I})\). Instead, the variances differ across dimensions. We extract the diagonal entries of the empirical covariance matrix and use them as target variances for our subsequent trajectory distillation experiments. Specifically, we consider the target latent distribution to be \(\mathcal{N}\bigl(\bm{0}, \mathrm{Diag}(\lambda_1, \lambda_2, \dotsc, \lambda_{128})\bigr)\), where \(\lambda_i\) denotes the extracted variance for the \(i\)th latent dimension. This empirical distribution serves as a data-driven alternative to the conventional isotropic prior.

We apply different diffusion trajectory distillation strategies to generate new latent codes. Working in the latent space not only reduces the computational cost but also aligns with recent advances in generative modeling, where diffusion models are increasingly applied in learned latent spaces for efficiency and improved sample quality~\citep{rombach2022high}. Our setup mirrors this practice by distilling trajectories directly in the latent space of the pretrained MSSIMVAE.

\paragraph{Results}
For each trajectory distillation strategy, we compare the generated latent codes to those produced by applying the surrogate composite teacher operator under the same initial noise. All generated latent codes are decoded into images through the pretrained MSSIMVAE decoder. For reference, the decoded outputs of the surrogate teacher operator are shown in~\Cref{fig:surrogate_decoded}. These serve as the ground-truth targets for evaluating student performance. We then visualize the decoded outputs produced by the student models trained with four different strategies, i.e., vanilla, progressive, consistency, and BOOT, at three training stages: early (epoch \(10\)), intermediate (epochs \(100\) and \(1{,}000\)), and converged (epoch \(10{,}000\)), shown in~\Cref{fig:real_decoded_diff_epoch10,fig:real_decoded_diff_epoch100,fig:real_decoded_diff_epoch1000,fig:real_decoded_diff_epoch10000}, respectively. In each figure, the left column shows the decoded outputs of the student under different strategies, while the right column presents heatmaps of the absolute difference between each student output and the corresponding teacher output. Red indicates higher deviation, while blue indicates lower deviation. The overlaid numbers report the pixel-wise \(L_2\) distance for each case. Since all strategies operate in a shared latent space and aim to approximate the same surrogate teacher operator, it is expected that their decoded outputs appear visually similar. Nonetheless, subtle differences in operator fidelity are clearly reflected in the quantitative metrics. Across all epochs, the sequential BOOT strategy consistently achieves the lowest \(L_2\) error, validating our theoretical prediction that it is optimal when the target covariance structure satisfies \(\lambda_i \leq 1\) in most dimensions.

\begin{figure}[htb]
\centering
\includegraphics[width=0.9\linewidth]{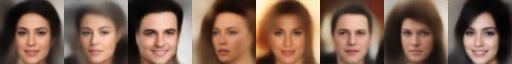}
\caption{Reference outputs generated by the surrogate teacher operator, decoded through the pretrained MSSIMVAE. These images serve as ground-truth targets for evaluating the outputs of different distillation strategies.}
\label{fig:surrogate_decoded}
\end{figure}

\begin{figure}[htb]
\centering
\begin{minipage}[t]{0.47\linewidth}
\vspace{5pt}
\includegraphics[width=\linewidth]{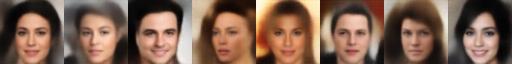}\\[0.2em]
\includegraphics[width=\linewidth]{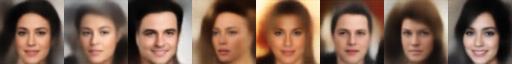}\\[0.2em]
\includegraphics[width=\linewidth]{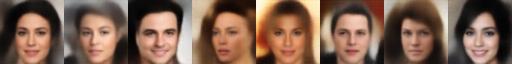}\\[0.2em]
\includegraphics[width=\linewidth]{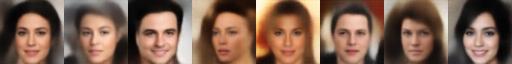}
\end{minipage}
\hfill
\begin{minipage}[t]{0.5\linewidth}
\vspace{0pt}
\includegraphics[width=\linewidth]{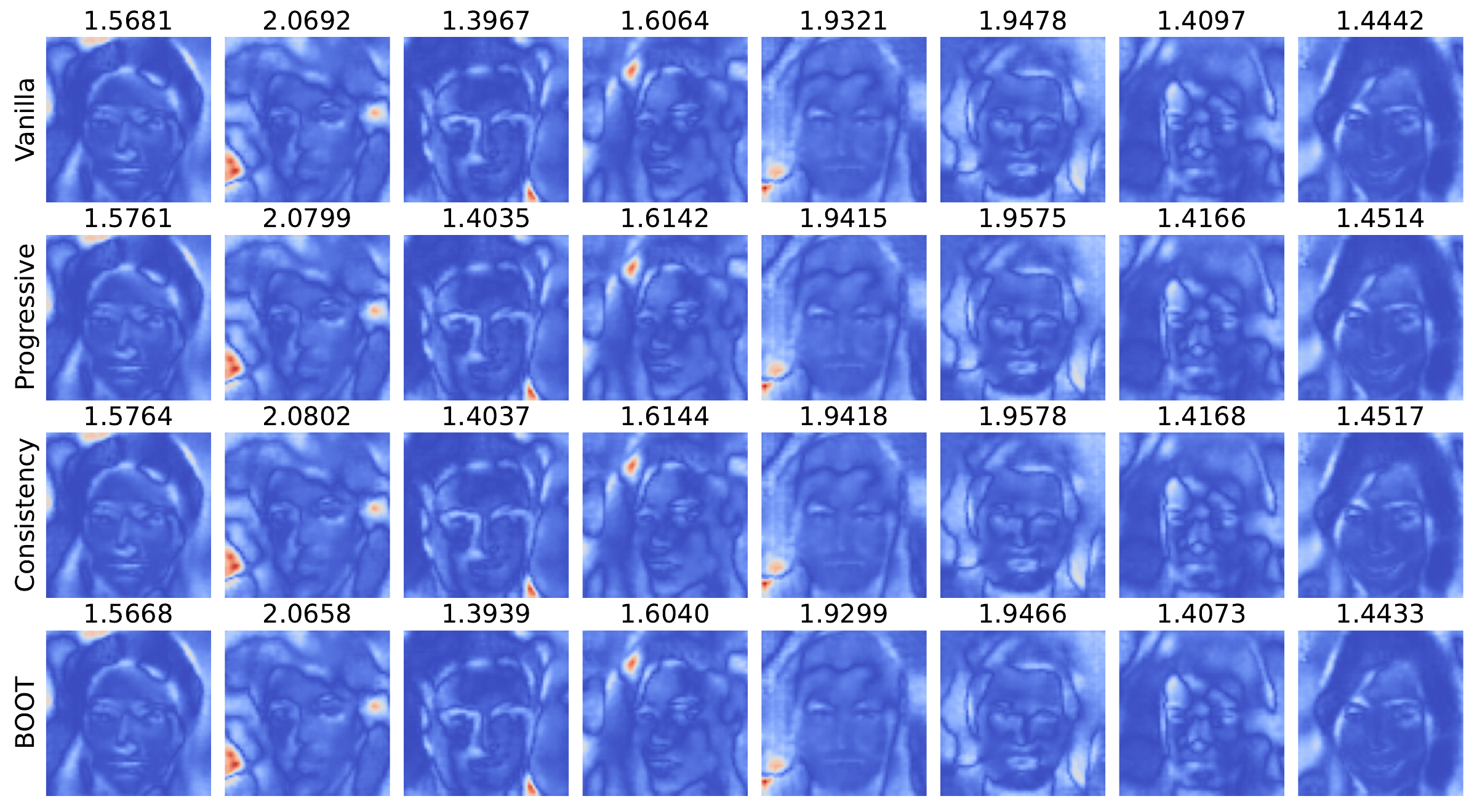}
\end{minipage}
\caption{Comparison of decoded outputs from different distillation strategies at training epoch \(10\). The visualization format is identical to~\Cref{fig:real_decoded_diff_epoch10000}.}
\label{fig:real_decoded_diff_epoch10}
\end{figure}

\begin{figure}[htb]
\centering
\begin{minipage}[t]{0.47\linewidth}
\vspace{5pt}
\includegraphics[width=\linewidth]{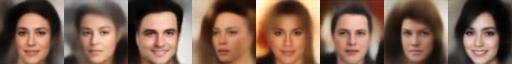}\\[0.2em]
\includegraphics[width=\linewidth]{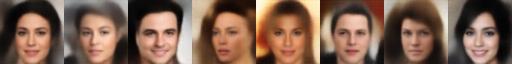}\\[0.2em]
\includegraphics[width=\linewidth]{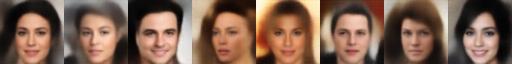}\\[0.2em]
\includegraphics[width=\linewidth]{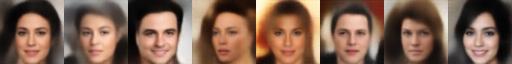}
\end{minipage}
\hfill
\begin{minipage}[t]{0.5\linewidth}
\vspace{0pt}
\includegraphics[width=\linewidth]{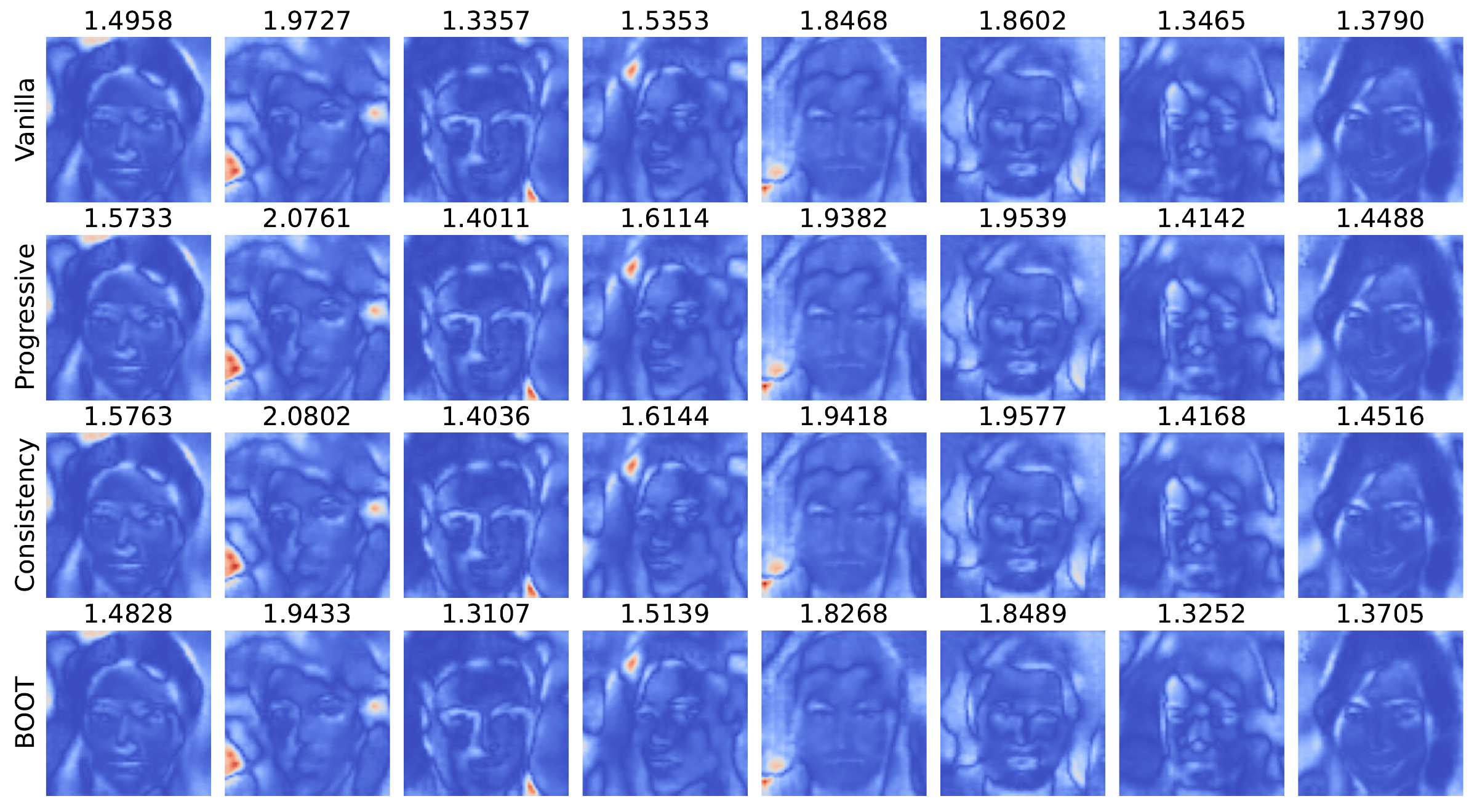}
\end{minipage}
\caption{Comparison of decoded outputs from different distillation strategies at training epoch \(100\). The visualization format is identical to~\Cref{fig:real_decoded_diff_epoch10000}.}
\label{fig:real_decoded_diff_epoch100}
\end{figure}

\begin{figure}[H]
\centering
\begin{minipage}[t]{0.47\linewidth}
\vspace{5pt}
\includegraphics[width=\linewidth]{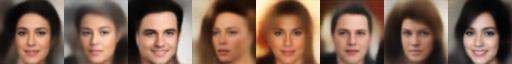}\\[0.2em]
\includegraphics[width=\linewidth]{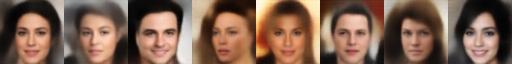}\\[0.2em]
\includegraphics[width=\linewidth]{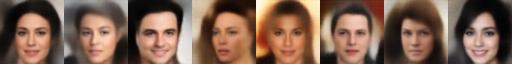}\\[0.2em]
\includegraphics[width=\linewidth]{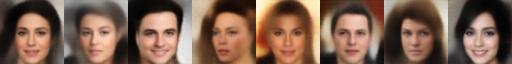}
\end{minipage}
\hfill
\begin{minipage}[t]{0.5\linewidth}
\vspace{0pt}
\includegraphics[width=\linewidth]{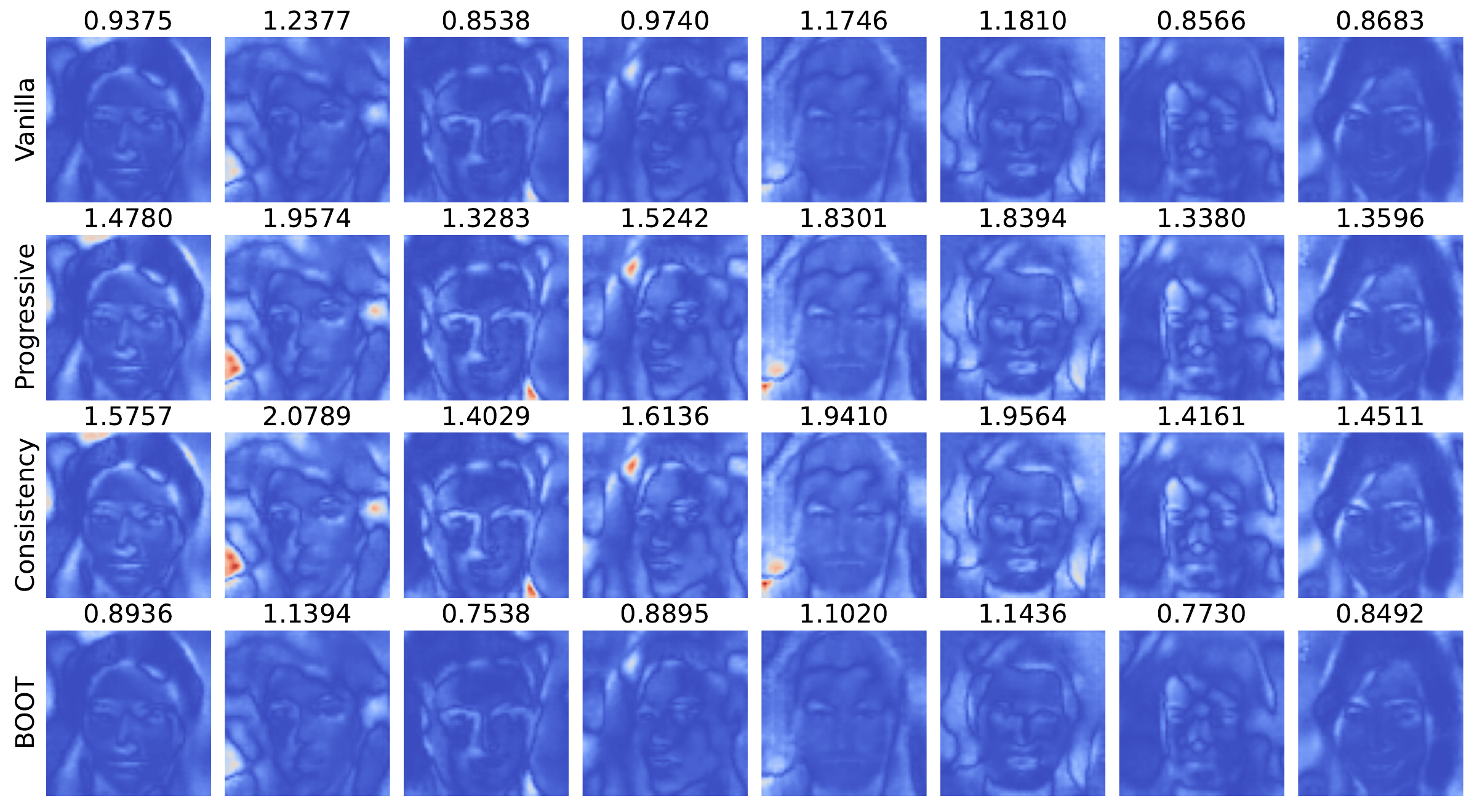}
\end{minipage}
\caption{Comparison of decoded outputs from different distillation strategies at training epoch \(1{,}000\). The visualization format is identical to~\Cref{fig:real_decoded_diff_epoch10000}.}
\label{fig:real_decoded_diff_epoch1000}
\end{figure}

%%%%%%%%%%%%%%%%%%%%%%%%%%%%%%%%%%%%%%%%%%%%%%%%%%%%%%%%%%%%
\subsection{\rev{Additional results on CIFAR-10}}
\label{subsec:additional_results_on_cifar10}

\paragraph{\rev{Experimental setup}}
\rev{We perform experiments directly in the pixel space of CIFAR-10 using a pretrained DDIM teacher. CIFAR-10 contains \(32\times 32\) color images from \(10\) semantic classes and exhibits a highly multimodal data structure in pixel space, making it a natural benchmark for the Gaussian-mixture perspective developed in~\Cref{sec:approximation_error_and_error_propagation_in_the_gaussian_mixture_regime}.}

\rev{The teacher model is an \(x_0\)-predicting U-Net~\citep{ronneberger2015u} trained on CIFAR-10 with the standard denoising objective. The network consists of residual blocks with sinusoidal time embeddings, together with downsampling and upsampling layers in a compact U-Net architecture. We use a cosine noise schedule with \(T=64\) diffusion steps, parameterized by \(\bm{x}_t = \alpha_t \bm{x}_0 + \sigma_t \bm{\varepsilon}\), and wrap the trained \(x_0\)-prediction network into a deterministic DDIM sampler, so that the one-step map is given by \( \bm{f}_{\bm{\eta}}(\bm{x}_t,t)=\bm{x}_{t-1} \). Images are normalized to \([-1,1]\), and the teacher is trained with AdamW~\citep{loshchilov2017decoupled} using learning rate \(2\times 10^{-4}\), batch size \(256\), gradient clipping \(1.0\), and exponential moving average with decay \(0.999\). The default training length is \(100{,}000\) parameter updates.}

\rev{For vanilla distillation, we initialize the student from the pretrained teacher checkpoint and keep the same DDIM-wrapped \(x_0\)-predicting U-Net architecture. At each update, we sample \( \bm{z}_T \sim \mathcal{N}(\bm{0},\bm{I}) \), compute the teacher target \( \mathcal{T}_T(\bm{z}_T) \) by a full deterministic DDIM rollout from time \(T\) to \(0\), and train the student output \( \bm{f}^{\mathrm{st}}(\bm{z}_T,T) \) to match this target with a mean squared error loss. Optimization uses AdamW for \(10{,}000\) updates with learning rate \(2\times 10^{-4}\), batch size \(64\), gradient clipping \(1.0\), and exponential moving average with decay \(0.999\).}

\rev{For progressive distillation, we recursively halve the teacher trajectory until only one student step is reached. At each level, every merged step is trained as an independent student initialized from the pretrained teacher checkpoint. Given a current sequence of step operators, we form the target for each adjacent pair by composing the corresponding two frozen operators, using the original teacher at the first level and the previously distilled step models at later levels. If a merged step starts from diffusion index \(t\), we sample clean images \(\bm{x}_0\) from CIFAR-10, generate \(\bm{x}_t=\alpha_t\bm{x}_0+\sigma_t\bm{\varepsilon}\), evaluate the frozen two-step composition, and train \( \bm{f}^{\mathrm{st}}(\bm{x}_t,t) \) against this target with a mean squared error loss. Each merged step is trained for \(10{,}000\) updates with AdamW under the same optimizer settings as above.}

\rev{For sequential BOOT, we fix the input at the terminal Gaussian noise \( \bm{x}_T \sim \mathcal{N}(\bm{0},\bm{I}) \) and train stagewise students that progressively predict cleaner states. Training proceeds backward in diffusion time. At stage \(t\), the student is trained to approximate the teacher state one level cleaner than \(t\) from the same terminal input \( \bm{x}_T \). For the first stage, the target is obtained directly from the teacher by rolling out from \( \bm{x}_T \). For later stages, the target is bootstrapped by first applying the previous BOOT stage model to \( \bm{x}_T \) and then applying one additional frozen teacher step. The student is then trained against this stage target with a mean squared error loss for \(10{,}000\) updates using the same optimizer settings.}

\rev{For sequential consistency, we train stagewise students that progressively extend the direct mapping from an intermediate noisy state to the clean sample. Training proceeds in increasing diffusion time. At stage \(s\), corresponding to model index \(t=s-1\), we sample clean images \(\bm{x}_0\) from CIFAR-10, generate \(\bm{x}_t=\alpha_t\bm{x}_0+\sigma_t\bm{\varepsilon}\), and train \( \bm{f}^{\mathrm{st}}(\bm{x}_t,t) \) against a recursively defined target. For the first nontrivial stage \(s=2\), the target is the full teacher rollout from \(\bm{x}_t\) to \(\bm{x}_0\). For later stages, the target is obtained by applying one frozen teacher step to \(\bm{x}_t\) and then passing the result to the previous consistency-stage model, thereby extending the horizon by one step at each stage. Again, optimization uses a mean squared error loss for \(10{,}000\) updates with the same AdamW configuration.}

\rev{For quantitative evaluation, we compute Fr\'echet Inception Distance (FID)~\citep{heusel2017gans} from \(50{,}000\) generated samples for each method. We extract \(2048\)-dimensional Inception-V3 features, estimate their empirical mean and covariance, and compute the corresponding Fr\'echet distance.}
\end{document}